\documentclass{article} 
\usepackage{iclr2027_conference,times}

\usepackage{graphicx}
\usepackage{subcaption}
\graphicspath{{figures/}}

\usepackage{tikz}
\usetikzlibrary{backgrounds,arrows,shapes,shapes.geometric,shapes.misc,positioning,fit}

\usepackage{amsmath,amssymb,amsthm,mathtools,bm}

\usepackage{booktabs}
\usepackage[ruled,vlined]{algorithm2e}

\usepackage[colorlinks=true,citecolor=blue,linkcolor=blue]{hyperref}

\usepackage{dma-macros}

\newtheorem{theorem}{Theorem}[section]
\newtheorem{lemma}[theorem]{Lemma}
\newtheorem{corollary}[theorem]{Corollary}
\newtheorem{proposition}[theorem]{Proposition}
\newtheorem{definition}[theorem]{Definition}
\newtheorem{remark}[theorem]{Remark}

\title{Direct Message Approximation (DMA): 
A Consistency-Based Framework for Tractable 
Approximate Inference on Factor Graphs}

\author{Ralf Herbrich\thanks{Equal contribution.} , Rainer Schlosser\footnotemark[1] , 
Jan Lemcke, Johann Ukrow, Anna Kazachkova, \\
\bf Nicolas Alder, Leonhard Hennicke, Theo Bardey, Nico Grimm, Luca Kleinschmidt, \\
\bf Philipp Kolbe, Cezary Kujath, Johanna Schlimme, Karl Matti Schütz\\
Hasso Plattner Institute \\
University of Potsdam \\
Potsdam, Germany \\
\texttt{\{ralf.herbrich,rainer.schlosser,jan.lemcke,johann.ukrow,}\\
\texttt{anna.kazachkova,nicolas.alder,leonhard.hennicke,theo.bardey,}\\
\texttt{nico.grimm,luca.kleinschmidt,philipp.kolbe,cezary.kujath,}\\
\texttt{johanna.schlimme,karlmatti.schuetz\}@hpi.de}
}

\iclrfinalcopy 
\begin{document}

\maketitle


\begin{abstract}
Approximate message passing on factor graphs underlies two dominant families
of probabilistic inference algorithms: expectation propagation~(EP) and
variational message passing~(VMP).
Both methods approximate the \emph{marginal} at each factor edge, forcing an
iterative round-robin schedule, risking negative-precision messages, and, for
VMP, collapsing to point estimates at Dirac-delta factors.
We introduce \emph{Direct Message Approximation}~(DMA), which approximates
factor-to-variable \emph{messages} directly rather than the marginal.
For normalisable factors, we define a consistency condition 
(requiring exactness when all other incoming messages are Dirac deltas) to guide message construction.
We prove a master theorem (proper messages, any graph) 
bounding marginal KL from message KL, with three structural corollaries: 
Dirac-input consistency, no EP-style inner-loop iteration, and no negative-precision messages.
Further, we prove a complementary $O(1/r^2)$ guarantee for the 
inherently improper backward message of the product factor, 
whose closed-form treatment has resisted prior work.
%
%
As a concrete instantiation, we derive explicit DMA messages for the product 
and leaky-ReLU factors and assemble a Bayesian neural network (BNN) inference algorithm with
one forward/backward sweep per training example and no gradient learning-rate hyperparameter,
validating that the structural guarantees translate to predictive uncertainty 
that widens in data-sparse regions, including under model mismatch.
\end{abstract}


\section{Introduction}
\label{sec:intro}

Factor graphs provide a unifying language for probabilistic inference: a
joint density factorises into local potentials, and the sum-product algorithm
computes exact marginals on trees by passing messages along edges
\citep{KschischangFreyLoeliger2001}.
When factors are non-conjugate to the message family (as in neural network
likelihoods, latent Dirichlet models, and models with nonlinear potentials),
the factor-to-variable integral is intractable, and approximate message
passing is required.

Two classical strategies dominate.
Expectation propagation \citep[EP;][]{Minka2001} minimises the forward KL
divergence from the true posterior marginal to a tractable exponential-family
approximation, then recovers the approximate factor-to-variable message by
dividing out the variable-to-factor message at the same edge.
Variational message passing \citep[VMP;][]{WinBis2005a} instead maximises
the evidence lower bound under a factorised approximation, minimising the
reverse KL\@.
Despite their theoretical differences, both methods share a structural
commitment: they approximate the \emph{marginal} $\hat{p}_{X_j}$ at each
factor edge, and recover the outgoing factor-to-variable message by
\begin{equation}
  \hat{m}_{f \to X_j}(\cdot)
    \;=\; \hat{p}_{X_j}(\cdot) \;/\; \msgto{X_j}{f}(\cdot).
  \label{eq:intro:ep-division}
\end{equation}

\paragraph{Three pathologies of marginal-based approximation.}
This marginal-first design induces three operational problems, each well-known
in the EP and VMP literature.

\begin{enumerate}

\item \textbf{Iterative schedule.}
Because $\hat{m}_{f \to X_j}$ depends on the current incoming message
$\msgto{X_j}{f}$ via the division in~\eqref{eq:intro:ep-division}, updating
one outgoing message changes the variable-to-factor message at the same edge,
invalidating the outgoing messages of all neighbouring factors.
The only remedy is a round-robin sweep that iterates until a fixed point is
reached, typically requiring many passes over the entire factor graph \citep{Minka2001}.

\item \textbf{Invalid messages in Gaussian EP.}
When the projected Gaussian marginal $\hat{p}_{X_j}$ is wider than the
incoming message $\msgto{X_j}{f}$, the division in~\eqref{eq:intro:ep-division}
produces a Gaussian with negative precision.
Such messages are invalid probability distributions; they arise routinely when
a factor disperses mass (e.g.\ a product of two uncertain variables) and can
destabilise subsequent updates by propagating indefinitely.

\item \textbf{Dirac-delta collapse in VMP.}
For a factor $f(\mathbf{x}) = \dirac{x_j - g(\mathbf{x}_{-j})}$ encoding a
deterministic relation, the VMP update~\eqref{eq:intro:ep-division} collapses
$\hat{p}_{X_j}$ to a Dirac delta, setting $\hat{m}_{f \to X_j}$ to a point mass
and eliminating all posterior variance.
Since every parameterised layer in a feedforward neural network is such a
Dirac-delta factor, VMP cannot maintain a non-trivial weight posterior in
Bayesian neural networks \citep{WinBis2005a}.

\end{enumerate}

\paragraph{Our proposal.}
We argue that the root cause of all three pathologies is the division in
\eqref{eq:intro:ep-division}: it creates an incoming-message dependence that
forces iteration and can produce invalid outputs.
DMA eliminates cavity division from the construction 
of approximate factor-to-variable messages by approximating the
\emph{message} $\hat{m}_{f \to X_j}$ directly rather than the marginal.

The natural design criterion for a message approximation is a
\emph{consistency condition}: it should agree with the exact sum-product
message whenever all other inputs to the factor are point masses.
In that limit the factor-to-variable integral reduces to a deterministic
function evaluation and is always tractable, so the exact message is known
and we can demand recovery.
We formalise this as Definition~\ref{def:dma} and call the resulting
approximation a \emph{direct message approximation} (DMA).
Because no per-factor message computation requires the ratio~\eqref{eq:intro:ep-division},
the iterative schedule is unnecessary, negative-precision messages cannot arise
from message construction itself, and Dirac-delta factors satisfy the consistency
condition exactly in the concentrated-input limit.
(Training does use a division when replacing old factor messages across epochs,
see Algorithm~\ref{alg:bnn}, but this outer-loop replacement is distinct from
the per-factor computation and does not create the feedback loop or
negative-precision pathologies of~\eqref{eq:intro:ep-division}.)

\medskip\noindent
The key distinction from assumed-density filtering
\citep[ADF;][]{Lauritzen1992,OpperWinther1999}, which also skips cavity
division, is that DMA adds a formal \emph{consistency axiom} that defines the admissible
direct-approximation principle, here instantiated through moment-matching
projection onto the Gaussian family (Section~\ref{sec:dma:construction}), an \emph{edge-local} master theorem
bounding marginal KL from message KL without any graph contraction condition
(Section~\ref{sec:dma}), and an O($1/r^2$) guarantee for the improper product-backward message
that falls outside the classical moment-matching scope
(Section~\ref{sec:factors}); see Section~\ref{sec:related} for a detailed
comparison with EP and ADF.

\paragraph{Contributions.}
\begin{enumerate}

\item \textbf{DMA framework and master theorem (Section~\ref{sec:dma}).}
We define a \emph{consistency condition} as the design criterion for direct
message approximation, give a construction recipe via moment matching, and
prove a \emph{master theorem} (Theorem~\ref{thm:dma-master}) bounding the
marginal KL at any edge by $O(\delta)$ in the message KL
$\delta$, edge-locally and without any graph contraction condition.
The closest prior result \citep{IhlerFisherWillsky2005} bounds perturbations
in loopy BP via graph-level contraction rates; our bound applies to
EP-style projected messages and holds on any graph.

\item \textbf{Product-factor DMA construction (Section~\ref{sec:factors}).}
We derive DMA messages for the product factor $\dirac{z - xy}$, the central 
and technically most demanding case: it is the main factor required for BNN 
weight updates, where the backward distribution of $X = Z/Y$ has no closed-form 
Gaussian representation and where the standard EP backward message is structurally improper.
We use a \emph{log-normal intermediate technique}
to characterise the ratio distribution.
Moreover, we prove an $O(1/r^2)$ accuracy bound in the concentrated-input regime
(Theorem~\ref{thm:product-bwd-bound}), an analysis outside the scope of the master theorem.
The complementary activation factor $\dirac{y - \relu{x}{\alpha}}$ is handled by
truncated-Gaussian moment matching; both factors verify the consistency condition analytically.

\item \textbf{Instantiation and validation (Section~\ref{sec:bnn}).}
The two factors are sufficient to assemble a complete BNN inference
algorithm.
Three corollaries of Theorem~\ref{thm:dma-master} give: no EP-style
inner-loop iteration (Corollary~\ref{cor:single-pass}), no negative-precision
messages by construction (Corollary~\ref{cor:no-neg-prec}), and Dirac-input
consistency on all linear factors (Corollary~\ref{cor:dirac-exact}).
Experiments confirm convergence, positive weight variances, and
structurally widening extrapolation uncertainty; a $23{\times}$ larger
network (Appendix~\ref{app:large}) confirms the procedure remains
computationally viable without algorithmic changes.
\end{enumerate}

Sections~\ref{sec:dma}--\ref{sec:related} cover the DMA framework, factor
derivations, the BNN algorithm, and related work; background, all proofs and experimental
details are in the Appendix \ref{sec:background} - \ref{app:large}.


\section{Direct Message Approximation}
\label{sec:dma}

\paragraph{Notation.}
We follow the factor-graph and message-passing conventions of
Appendix~\ref{sec:background}; readers unfamiliar with the sum-product
algorithm may consult it for background.
We write $\msgto{f}{X_j}$ for the factor-to-variable message on edge
$(f, X_j)$ and $\msgto{X_j}{f}$ for the variable-to-factor message.
The marginal at $X_j$ is $p_{X_j} = \prod_{i \in \neighbours{X_j}}
\msgto{f_i}{X_j}$, and $\mathcal{Q}$ denotes a fixed exponential family
(Gaussian throughout this paper).
$\KL{\cdot}{\cdot}$ is the forward KL divergence.

DMA approximates the outgoing message directly, without forming the marginal
at all; this avoids the division in~\eqref{eq:intro:ep-division} and all
three pathologies it creates.
This section formalises the approach, proves a master theorem bounding the
resulting marginal error, and derives three structural corollaries for BNN
inference.

\subsection{Definition and Consistency Condition}
\label{sec:dma:def}

The exact factor-to-variable message~\eqref{eq:bg:f-to-x} is, in general,
intractable.
There is, however, one regime in which it is always tractable: when all other
incoming messages are Dirac deltas $\dirac{\cdot - x_k}$.
Then the multi-dimensional integral collapses to a single factor evaluation,
\begin{equation}
  \msgto{f_i}{X_j}(x_j)
    = \int f_i(x_j, \mathbf{x}_{-j})
      \prod_{k \neq j} \dirac{x_k - \bar{x}_k}
      \intd{\mathbf{x}_{-j}}
    = f_i(x_j, \bar{\mathbf{x}}_{-j}),
  \label{eq:dma:dirac-limit}
\end{equation}
which is just the factor evaluated at the fixed point values $\bar{\mathbf{x}}_{-j}$.
The resulting message is an unnormalised one-dimensional density in $x_j$; for
Gaussian $\mathcal{Q}$ and the common case where $f_i(\cdot, \bar{\mathbf{x}}_{-j})$
is a Gaussian likelihood or a pushforward of a Gaussian through a smooth map,
the projection onto $\mathcal{Q}$ is tractable in closed form.

The consistency condition asks that a message approximation converge to the
exact message as inputs concentrate toward this tractable limit.

\begin{definition}[Direct Message Approximation]
\label{def:dma}
Let $f_i$ be a factor in a factor graph with neighbours $\neighbours{f_i}$,
let $X_j \in \neighbours{f_i}$, and suppose $f_i(\cdot, \bar{\mathbf{x}}_{-j})$
is normalisable for every $\bar{\mathbf{x}}_{-j}$.
A family of distributions $\hat{m}_{f_i \to X_j}$, parametrised by the
incoming messages $\{\msgto{X_k}{f_i}\}_{k \neq j}$, from an exponential
family $\mathcal{Q}$ is a \emph{direct message approximation} (DMA) if it
satisfies the \emph{concentration consistency condition}:
for every $\bar{\mathbf{x}}_{-j} \in \mathbb{R}^{|\neighbours{f_i}|-1}$,
if $\msgto{X_k}{f_i} = \mathcal{N}(\bar x_k, \sigma_k^2)$ for all $k \neq j$,
then as $\sigma_k \to 0$,
\begin{equation}
  \hat{m}_{f_i \to X_j}
    \;\xrightarrow{\;w\;}\;
    \frac{f_i\!\left(\cdot,\;\bar{\mathbf{x}}_{-j}\right)}
         {\displaystyle\int f_i\!\left(x_j,\;\bar{\mathbf{x}}_{-j}\right)\mathrm{d}x_j},
  \label{eq:dma:consistency}
\end{equation}
where $\xrightarrow{w}$ denotes weak convergence of probability measures
and the right-hand side is the normalised exact message~\eqref{eq:dma:dirac-limit}
evaluated at the concentrated inputs.
\end{definition}

\begin{remark}
The normalisability condition in Definition~\ref{def:dma} holds for all
factor-edge pairs in this paper except one: the product factor
\emph{backward} message is improper for general Gaussian inputs
(Remark~\ref{rem:product-improper}).
This case is handled by a separate log-normal intermediate construction
whose approximation quality is characterised by
Theorem~\ref{thm:product-bwd-bound}: the KL to a truncated proper reference
is $O(1/r^2)$ in the input signal-to-noise ratio $r$, recovering
concentration consistency in the limit $r \to \infty$.
\end{remark}

The consistency condition is a necessary but not sufficient design
criterion: the moment-matching construction
(Section~\ref{sec:dma:construction}) selects a specific DMA satisfying
it, and Theorem~\ref{thm:dma-master} bounds the marginal error at any
input width---consistency, construction, and master theorem together
constitute the local theoretical guarantees for individual DMA message computations.

\subsection{Construction Recipe}
\label{sec:dma:construction}

For the class of factors that arises in Bayesian neural networks --- Dirac-delta
factors of the form $f(\mathbf{x}) = \dirac{x_j - g(\mathbf{x}_{-j})}$
encoding a continuous deterministic function $g$ --- there is a systematic recipe for
constructing a Gaussian DMA.

\paragraph{Recipe.}
Let the incoming messages $\{\msgto{X_k}{f}\}_{k \neq j}$ be Gaussians
$\mathcal{N}(\bar{x}_k, \sigma_k^2)$.
\begin{enumerate}
  \item \textbf{Compute moments.}
    Treat the inputs $\{X_k\}_{k \neq j}$ as independent Gaussian random
    variables with the given means and variances, and compute the first two
    moments of $x_j = g(\mathbf{x}_{-j})$ under this joint:
    \[
      \mu_{x_j} := \mathbb{E}[g(\mathbf{X}_{-j})],
      \qquad
      \sigma_{x_j}^2 := \mathrm{Var}[g(\mathbf{X}_{-j})].
    \]
  \item \textbf{Project.}
    Set $\hat{m}_{f \to X_j} := \mathcal{N}(\mu_{x_j},\, \sigma_{x_j}^2)$
    via Theorem~\ref{thm:kl-min} (moment matching minimises the forward KL
    to the Gaussian family).
\end{enumerate}

\paragraph{Consistency verification.}
In the Dirac limit $\sigma_k \to 0$ for all $k \neq j$, the inputs
concentrate on their means $\bar{x}_k$, so
$\mu_{x_j} \to g(\bar{\mathbf{x}}_{-j})$ and $\sigma_{x_j}^2 \to 0$.
The approximate message converges to $\dirac{\cdot - g(\bar{\mathbf{x}}_{-j})}$,
which equals the exact message~\eqref{eq:dma:dirac-limit} (a point mass at
the function value).
The KL divergence between two identical distributions is zero, so the
consistency condition~\eqref{eq:dma:consistency} holds.

\subsection{Master Theorem and Consequences}
\label{sec:dma:master}

The master theorem bounds the error in the marginal $p_{X_j}$ that results
from using a DMA in place of the exact message.
The bound is edge-local: it depends only on the quality $\delta$ of the
single outgoing message, and on the sup-norm of the incoming message at the
same edge.

\begin{theorem}[DMA Master Theorem]
\label{thm:dma-master}
Let $f$ be a factor with target variable $X_j$.
Let $\msgto{f}{X_j}$ be the true normalised sum-product message and
$\hat{m}_{f \to X_j}$ a DMA; write
$\delta := \KL{\msgto{f}{X_j}}{\hat{m}_{f \to X_j}} < \infty$.
Let $\msgto{X_j}{f}$ be a normalised incoming message at the same edge,
define $p_{X_j} := \msgto{f}{X_j} \cdot \msgto{X_j}{f}$ and
$\hat{p}_{X_j} := \hat{m}_{f \to X_j} \cdot \msgto{X_j}{f}$,
with normalising constants $Z$ and $\hat{Z}$ respectively.
Then
\begin{equation}
  \KL{\frac{p_{X_j}}{Z}}{\frac{\hat{p}_{X_j}}{\hat{Z}}}
  \;\leq\;
  \frac{\|\msgto{X_j}{f}\|_\infty}{Z}\,\delta.
  \label{eq:dma:master}
\end{equation}
\end{theorem}

\paragraph{Proof.}
The full proof is in Appendix~\ref{app:proofs:master}.

\paragraph{Interpretation.}
The bound is linear in $\delta$; for $\delta < 1$ (the approximation regime)
this is strictly tighter than a square-root dependence.
The bound is \emph{edge-local} and requires no global contraction condition,
contrasting with \citet{IhlerFisherWillsky2005} whose analogous result for
loopy BP uses graph-level contraction rates.
Across 192 leaky-ReLU factor configurations spanning a range of slopes,
input widths, and SNR values (Appendix~\ref{app:validation:master}),
the bound holds in every case with a maximum normalised ratio of $0.97$,
directly certifying the theorem.

The corollaries below extract the structural consequences of the master
theorem for BNN inference.

\begin{corollary}[Asymptotic Dirac-Input Consistency]
\label{cor:dirac-exact}
Let $f(\mathbf{x}) = \dirac{x_j - g(\mathbf{x}_{-j})}$, suppose the DMA is
constructed by the recipe of Section~\ref{sec:dma:construction}, and assume
$g$ is twice continuously differentiable in a neighbourhood of
$\bar{\mathbf{x}}_{-j}$ with $\nabla g(\bar{\mathbf{x}}_{-j}) \neq \mathbf{0}$,
and regular in the tails.
Then, with each incoming message $\msgto{X_k}{f} = \mathcal{N}(\bar x_k, \sigma_k^2)$,
\[
  \lim_{\sigma_k \to 0} \delta
  \;=\;
  \lim_{\sigma_k \to 0} \KL{\msgto{f}{X_j}}{\hat{m}_{f \to X_j}}
  \;=\; 0.
\]
The result extends to piecewise-$C^2$ factors (e.g.\ leaky-ReLU) by treating
isolated non-smooth points separately; the upper bound in~\eqref{eq:dma:master} 
therefore vanishes in this limit.
\end{corollary}

\begin{corollary}[No EP Inner-Loop Iteration]
\label{cor:single-pass}
Because $\hat{m}_{f \to X_j}$ does not depend on $\msgto{X_j}{f}$
(cf.\ Definition~\ref{def:dma}), all factor-to-variable messages for a single
training example can be computed in one forward sweep followed by one backward
sweep, with no EP-style inner-loop fixed-point iteration between individual
message computations.
This is distinct from the outer training loop, which repeats sweeps over
examples and epochs until the weight beliefs converge.
\end{corollary}

\begin{corollary}[No Negative-Precision Messages]
\label{cor:no-neg-prec}
Because DMA produces a valid member of $\mathcal{Q}$ by construction
(via Theorem~\ref{thm:kl-min}), no approximate message can have negative
precision.  The pathology arises in EP when the projected marginal is wider
than the incoming message; DMA never forms this ratio.
\end{corollary}

Proofs of all three corollaries are in Appendix~\ref{app:proofs:corollaries}.
Section~\ref{sec:factors} derives explicit DMA messages for the product
and ReLU factors, the two non-conjugate building blocks needed for BNNs.
Section~\ref{sec:bnn} then assembles these into a complete inference
algorithm with one forward/backward factor sweep per training example
and no EP-style inner-loop iteration.


\section{DMA for Non-Gaussian Factors}
\label{sec:factors}

Two non-conjugate factors appear in every feedforward BNN and cannot be handled
by the exact sum-product algorithm: the \emph{product factor}
$\dirac{z - xy}$ (elementwise weight-activation products) and the
\emph{ReLU factor} $\dirac{y - \relu{x}{\alpha}}$ (activation nonlinearity).
The Gaussian prior $\mathcal{N}(w;\,0,\sigma_0^2)$ and likelihood
$\mathcal{N}(y;\,z,\beta^2)$ are conjugate to the message family, so their
messages are exact and $\delta = 0$ at both edges \citep{WinBis2005a}.
We apply the construction recipe of Section~\ref{sec:dma:construction} to
the two non-conjugate factors, yielding closed-form Gaussian DMA messages.
The product factor has exact closed-form forward moments (via independence of
$X$ and $Y$), so the only approximation is the Gaussian projection; its
backward message is inherently improper.
The ReLU factor has a proper backward message for all $\alpha > 0$, with an
explicit normalisation correction that depends on $\alpha$.

\subsection{The Product Factor}
\label{sec:factors:product}

\paragraph{Factor definition.}
The product factor encodes the deterministic relation $z = xy$:
\begin{equation}
  f(x, y, z) = \dirac{z - xy}.
  \label{eq:product-factor}
\end{equation}
For Gaussian inputs $X \sim \mathcal{N}(\mu_x, \sigma_x^2)$ and
$Y \sim \mathcal{N}(\mu_y, \sigma_y^2)$, the product $Z = XY$ is not Gaussian,
so the exact forward message to $Z$ is not in $\mathcal{Q}$.
We apply the recipe: compute the exact moments of $Z$ under the independent
Gaussian inputs, then project onto $\mathcal{N}$.

\paragraph{Forward message.}
The moments of $Z = XY$ under independent Gaussians follow directly from
the law of total variance: $\mathbb{E}[XY] = \mu_x \mu_y$ and
$\mathrm{Var}[XY] = \sigma_x^2 \sigma_y^2 + \mu_x^2 \sigma_y^2 +
\mu_y^2 \sigma_x^2$.

\begin{proposition}[Product Factor Forward Message]
\label{prop:product-fwd}
For the factor $\dirac{z - xy}$ with independent inputs
$X \sim \mathcal{N}(\mu_x, \sigma_x^2)$ and
$Y \sim \mathcal{N}(\mu_y, \sigma_y^2)$,
the DMA message to $Z$ is $\Normal{z}{m_z}{s_z^2}$ with
\begin{equation}
  m_z   = \mu_x \mu_y, \qquad 
  s_z^2 = \sigma_x^2 \sigma_y^2 + \mu_x^2 \sigma_y^2 + \mu_y^2 \sigma_x^2.
             \label{eq:product-fwd-mean-var}
\end{equation}
\end{proposition}

Because the moments~\eqref{eq:product-fwd-mean-var}
are the \emph{exact} moments of $Z = XY$ (not approximations), the only
approximation is in the projection of the marginal of $Z$ onto a Gaussian.
In the Dirac limit $\sigma_x, \sigma_y \to 0$ both moments converge to those
of a point mass at $\mu_x\mu_y$, so $\delta = 0$ and Corollary~\ref{cor:dirac-exact}
applies.

\paragraph{Backward message via log-normal intermediates.}
The backward message to $X$ requires the moments of $X = Z/Y$.
Direct integration is problematic: the ratio of two Gaussians has no finite
mean (the integrand carries a $1/|y|$ singularity).
We resolve this using a \emph{log-normal intermediate} technique: for a
Gaussian $W \sim \mathcal{N}(\mu_w, \sigma_w^2)$ with $|\mu_w|/\sigma_w \gg 1$
(a sufficient condition for the approximation quality guaranteed by Theorem~\ref{thm:product-bwd-bound}),
the log-absolute-value $\log|W|$ is approximately normal with mean
$\log|\mu_w|$ and variance $\sigma_w^2/\mu_w^2$ (a first-order delta-method
approximation).
Under this approximation, $\log|Z/Y| = \log|Z| - \log|Y|$ is the difference
of two independent normals, giving a normal distribution whose parameters
can be propagated in closed form.
Converting back to natural parameters $(\tau_w, \rho_w) = (\mu_w/\sigma_w^2,
1/\sigma_w^2)$ yields explicit moment formulas for $X$ under the joint $(Y, Z)$.

\begin{proposition}[Product Factor Backward Message]
\label{prop:product-bwd}
For the factor $\dirac{z - xy}$, the DMA message to $X$ in natural
parameters is $\mathcal{N}(\hat\tau_x/\hat\rho_x,\, 1/\hat\rho_x)$ with
\begin{equation}
  \hat\tau_x = \frac{\tau_y\, r\, \tau_z}{D},  \quad
  \hat\rho_x = \frac{r^2}{D}, \label{eq:product-bwd-tau-rho}
\end{equation}
where
$r := \tau_y^4 \rho_z / (\rho_y(\tau_y^2 + \rho_y))$, and
$D := \tau_z^2 \rho_y + \tau_y^2 \rho_z + \rho_y \rho_z$.
The message to $Y$ is obtained by symmetry (swap $x \leftrightarrow y$
labels). The derivation is given in Appendix~\ref{app:proofs:product}.
\end{proposition}

\paragraph{Consistency verification.}
Taking $\rho_y \to \infty$ (Dirac limit on $Y$) gives
$\hat\tau_x/\hat\rho_x \to \mu_z/\bar{y}$ and $1/\hat\rho_x \to \sigma_z^2/\bar{y}^2$,
recovering the distribution of $X = Z/\bar{y}$; the derivation is in
Appendix~\ref{app:proofs:product}.

\subsection{Approximation Guarantee for the Product Backward Message}
\label{sec:factors:product-bound}

The exact backward message is improper, placing it outside
the scope of Theorem~\ref{thm:dma-master}, which requires a proper true message.
Theorem~\ref{thm:product-bwd-bound} below is the dedicated complement: it covers
this case by establishing a $C/r^2$ upper bound in the concentrated-input regime
($r_y/r_z$ bounded) where the improper tail is exponentially negligible.
Concretely, let $r := \min(r_y,r_z) = \min(|\mu_y|/\sigma_y,\; |\mu_z|/\sigma_z)$; for $r \geq 2$
one constructs a proper \emph{truncated reference} $\tilde{m}^r_{f \to X}$
(Appendix~\ref{app:proofs:product-bound}) by restricting to $|Y|\geq|\mu_y|/2$;
the excluded input region has probability $\leq \Phi(-r_y/2) = O(e^{-cr^2})$.

\begin{theorem}[Concentrated-Input Bound for Product Backward Message]
\label{thm:product-bwd-bound}
Let $r = \min(|\mu_y|/\sigma_y,\; |\mu_z|/\sigma_z) \geq 2$, assume the
ratio $r_y/r_z$ is bounded above and below by a constant $\kappa \geq 1$,
and let $\hat{m}_{f \to X}$ be the DMA message of
Proposition~\ref{prop:product-bwd}.
Then
\begin{equation}
  \KL{\tilde{m}^r_{f \to X}}{\hat{m}_{f \to X}}
  \;\leq\; C_\kappa\, r^{-2},
  \label{eq:product-bwd-bound}
\end{equation}
\end{theorem}
for some constant $C_\kappa > 0$ depending only on $\kappa$. 
The proof is in Appendix~\ref{app:proofs:product-bound}.

\subsection{The ReLU Factor}
\label{sec:factors:relu}

\paragraph{Factor definition.}
The ReLU factor encodes the leaky-ReLU nonlinearity
$y = \relu{x}{\alpha} := \max(0,x) + \alpha\min(0,x)$ for $\alpha \geq 0$:
\begin{equation}
  f(x, y) = \dirac{y - \relu{x}{\alpha}}.
  \label{eq:relu-factor}
\end{equation}
Setting $\alpha = 0$ gives the standard ReLU; $\alpha = 1$ gives the
identity; $\alpha > 0$ gives a leaky variant.
We focus on the leaky case ($\alpha > 0$), which is used throughout the BNN
experiments; the standard ReLU ($\alpha = 0$) is discussed only to identify
where its backward message becomes improper (Remark~\ref{rem:relu-improper}).

\paragraph{Forward message.}
The exact forward message to $Y$ is the pushforward of
$X \sim \mathcal{N}(\mu_x, \sigma_x^2)$ through $\relu{\cdot}{\alpha}$.
This is a mixture of a truncated Gaussian on $\mathbb{R}_{>0}$ (for $x > 0$,
$y = x$) and a scaled truncated Gaussian on $\mathbb{R}_{\leq 0}$ (for
$x \leq 0$, $y = \alpha x$); neither piece is Gaussian.
We apply moment matching via the doubly-truncated Gaussian moments
(whose computation uses the Mills ratio $M(t) := \phi(t)/\Phi(t)$ for
the CDF and PDF of the standard normal).

\begin{proposition}[ReLU Forward Message]
\label{prop:relu-fwd}
For the factor~\eqref{eq:relu-factor} with $X \sim \mathcal{N}(\mu_x, \sigma_x^2)$, let
  $u := \frac{\mu_x}{\sigma_x}$, $P := \Phi(u)$, $\phi := \varphi(u)$, 
  $A := \alpha + (1-\alpha)P$, $B := \alpha^2 + (1-\alpha^2)P$.
The DMA forward message is $\Normal{y}{m_Y}{s_Y^2}$ with
\begin{equation}
  m_Y   = \mu_x A + (1-\alpha)\,\sigma_x\,\phi, \qquad 
  s_Y^2 = (\mu_x^2 + \sigma_x^2)B 
             + (1-\alpha^2)\,\mu_x\,\sigma_x\,\phi
             - m_Y^2. \label{eq:relu-fwd-var}
\end{equation}
\end{proposition}

\paragraph{Backward message.}
The backward message to $X$ integrates the factor against the incoming
message $m_{Y \to f}(y) = \mathcal{N}(\mu_y, \sigma_y^2)$.
For $x > 0$ the Jacobian of $y = x$ is $1$; for $x \leq 0$ the Jacobian of
$y = \alpha x$ is $\alpha$, contributing a factor $1/\alpha$ to the
unnormalised message on that piece.
The total mass of the (unnormalised) backward mixture is therefore
$\tilde{C}/\alpha$ with $\tilde{C} := \alpha P + Q$,
where $v := \mu_y/\sigma_y$, $P := \Phi(v)$, $Q := 1-P$.
Normalising by $\tilde{C}/\alpha$ before matching moments yields a proper
Gaussian approximation for all $\alpha > 0$.

\begin{proposition}[ReLU Backward Message]
\label{prop:relu-bwd}
For the factor $\dirac{y - \relu{x}{\alpha}}$ with $\alpha > 0$ and
$Y \sim \mathcal{N}(\mu_y, \sigma_y^2)$, let $v := \mu_y/\sigma_y$,
$P := \Phi(v)$, $Q := 1-P$, $\phi := \varphi(v)$, and
$\tilde{C} := \alpha P + Q$.
The DMA backward message is $\Normal{x}{m_X}{s_X^2}$ with
\begin{align}
  m_X   &= \frac{\mu_y(\alpha^2 P + Q)
                 + (\alpha^2-1)\,\sigma_y\,\phi}
                {\alpha\,\tilde{C}},
              \label{eq:relu-bwd-mean} \\
  s_X^2 &= \frac{(\mu_y^2+\sigma_y^2)(\alpha^3 P + Q)
                  + (\alpha^3-1)\,\mu_y\,\sigma_y\,\phi}
                 {\alpha^2\,\tilde{C}}
            - m_X^2.
              \label{eq:relu-bwd-var}
\end{align}
\end{proposition}

\paragraph{Consistency verification.}
In both directions the Dirac limit ($u \to \pm\infty$ forward, $v \to \pm\infty$
backward) recovers $\relu{\bar{x}}{\alpha}$ and $\relu{\bar{y}}{\alpha}^{-1}$
respectively; the derivation is in Appendix~\ref{app:proofs:relu}.
For $\alpha = 0$ the backward message is improper (Appendix~\ref{app:proofs:relu}).


\section{Application: Bayesian Neural Network Inference}
\label{sec:bnn}

The two DMA factors from Section~\ref{sec:factors} are sufficient to implement
a complete inference algorithm for feedforward BNNs composed of Gaussian,
product, sum, and leaky-ReLU factors.
We assemble them here, apply the three structural corollaries to the
proper-message components, and validate the resulting algorithm (one forward/backward factor sweep
per training example, no EP-style inner-loop fixed-point iteration) on a 1D regression task.

\subsection{Factor Graph of a Feedforward Network}
\label{sec:bnn:fg}

A feedforward BNN with $L$ layers, input $\mathbf{x} \in \mathbb{R}^{d_0}$,
and hidden widths $d_1, \ldots, d_L$ is modelled as 
\begin{equation}
  p(\mathbf{W}, \mathbf{x}^{(1)}, \ldots, \mathbf{x}^{(L)}, y)
    = p(y \mid \mathbf{x}^{(L)})
      \prod_{l=1}^{L}
        p(\mathbf{W}^{(l)})\,
        \dirac{\mathbf{z}^{(l)}\!-\!\mathbf{W}^{(l)}\mathbf{x}^{(l-1)}}\,
        \dirac{\mathbf{x}^{(l)}\!-\!\relu{\mathbf{z}^{(l)}}{\alpha}},
        \label{eq:bnn:joint}
\end{equation}
with $\mathbf{x}^{(0)} = \mathbf{x}$.
Each Dirac-delta factor decomposes into independent 1D factors:
the matrix-vector product $\mathbf{z}^{(l)} = \mathbf{W}^{(l)}\mathbf{x}^{(l-1)}$
decomposes into $d_l \times d_{l-1}$ 1D product factors
$\dirac{z_{ij} - W^{(l)}_{ij}\,x^{(l-1)}_j}$ connected by $d_l$ sum factors
$z^{(l)}_i = \sum_j z^{(l)}_{ij}$ that implement the inner product;
the activation $\mathbf{x}^{(l)} = \relu{\mathbf{z}^{(l)}}{\alpha}$ splits
into $d_l$ independent ReLU factors $\dirac{x^{(l)}_i - \relu{z^{(l)}_i}{\alpha}}$.

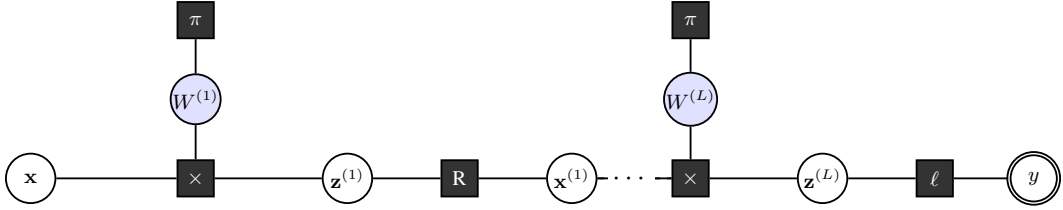
\begin{figure}[t]
  \centering
  \resizebox{\linewidth}{!}{%
  \begin{tikzpicture}[
    var/.style    = {circle, draw, thick, minimum size=7.5mm,
                     inner sep=0pt, font=\footnotesize},
    latent/.style = {circle, draw, thick, minimum size=7.5mm,
                     inner sep=0pt, font=\footnotesize, fill=blue!12},
    obs/.style    = {circle, draw, thick, double, minimum size=7.5mm,
                     inner sep=0pt, font=\footnotesize},
    fnode/.style  = {rectangle, draw, fill=black!80, text=white,
                     minimum size=5.5mm, inner sep=0pt, font=\footnotesize},
    >=stealth, thick
  ]
  \node[var]    (x0)  at (0,0)   {$\mathbf{x}$};

  \node[latent] (W1)  at (2.5, 1.2) {$W^{(1)}$};
  \node[fnode]  (pi1) at (2.5, 2.4) {$\pi$};
  \node[fnode]  (pr1) at (2.5, 0)   {$\times$};
  \node[var]    (z1)  at (4.8, 0)   {$\mathbf{z}^{(1)}$};

  \draw (pi1) -- (W1);
  \draw (W1)  -- (pr1);
  \draw (x0)  -- (pr1);
  \draw (pr1) -- (z1);

  \node[fnode]  (r1)  at (6.5, 0)   {\textrm{R}};
  \node[var]    (x1)  at (8.2, 0)   {$\mathbf{x}^{(1)}$};

  \draw (z1) -- (r1);
  \draw (r1) -- (x1);

  \node[latent] (W2)  at (10.0, 1.2) {$W^{(L)}$};
  \node[fnode]  (pi2) at (10.0, 2.4) {$\pi$};
  \node[fnode]  (pr2) at (10.0, 0)   {$\times$};
  \node[var]    (z2)  at (12.0, 0)   {$\mathbf{z}^{(L)}$};

  \draw (pi2) -- (W2);
  \draw (W2)  -- (pr2);
  \draw (x1)  -- (pr2);
  \draw (pr2) -- (z2);

  \node[fnode]  (lik) at (13.7, 0)   {$\ell$};
  \node[obs]    (y)   at (15.2, 0)   {$y$};

  \draw (z2)  -- (lik);
  \draw (lik) -- (y);

  \draw[dashed] (x1) -- (pr2);
  \node[fill=white, inner sep=2pt] at (9.1, 0) {\large$\cdots$};

  \end{tikzpicture}}%
  \caption{Factor graph of a BNN (schematic; two learnable layers shown).
    Circles are variable nodes; filled squares are factor nodes.
    Blue-tinted circles represent the $d_l\!\times\!d_{l-1}$ independent scalar
    weight variables $W^{(l)}_{ij}$ (shown collectively per layer for clarity),
    each with a Gaussian prior factor $\pi$.
    Each $\times$ node is a \emph{matrix-vector product factor}:
    shorthand for $d_l\!\times\!d_{l-1}$ 1D product factors
    $\dirac{z_{ij}-W^{(l)}_{ij}\,x^{(l-1)}_j}$ (Sec.~\ref{sec:factors:product})
    combined with $d_l$ sum factors $z^{(l)}_i=\sum_j z^{(l)}_{ij}$,
    together implementing the inner product (see Sec.~\ref{sec:bnn:fg}).
    \textrm{R} factors are element-wise leaky-ReLU factors (Sec.~\ref{sec:factors:relu});
    $\ell$ is the Gaussian likelihood factor (conjugate; $\delta=0$);
    the double circle is the observed output $y$.
    \vspace{-0.1cm}
    }
  \label{fig:bnn:fg}
\end{figure}

Figure~\ref{fig:bnn:fg} illustrates the factor graph schematically; each $\times$
node is a matrix-vector product factor composed of $d_l\!\times\!d_{l-1}$ 1D
product factors and $d_l$ sum factors.
The three factor types present (Gaussian prior factors, product
factors, and ReLU factors) carry conjugate, DMA (Prop.~\ref{prop:product-fwd}),
and DMA (Prop.~\ref{prop:relu-fwd}) messages respectively.
By Corollary~\ref{cor:dirac-exact}, all prior and linear-copy factors satisfy
the consistency condition; ReLU factors are the dominant source of
nonzero $\delta$ for uncertain inputs.

\subsection{Inference Algorithm: One Forward/Backward Sweep per Example}
\label{sec:bnn:alg}

Corollary~\ref{cor:single-pass} guarantees that no EP-style inner-loop
fixed-point iteration between individual message computations is required: each
factor-to-variable message is computed exactly once per sweep.
We fix the order as a single forward sweep followed by a single backward sweep
per training example, mirroring the computation graph of a standard neural
network.
The prior and likelihood factors produce exact Gaussian messages
\citep{WinBis2005a}; the non-conjugate factors use Propositions~\ref{prop:product-fwd}--\ref{prop:relu-bwd}.
Algorithm~\ref{alg:bnn} gives the complete training procedure
(split across two pages); each mini-batch maintains one stored factor-to-weight
message that is replaced (not accumulated) on each pass, so weight beliefs are
updated by EP-style message replacement rather than sequential accumulation.
Algorithm~\ref{alg:bnn} runs in $O(\sum_l d_l d_{l-1})$ time per example,
matching the asymptotic complexity of a standard neural network forward-backward pass
with a modest constant-factor overhead ($1.8\times$ per epoch),
and requires only two scalar parameters $(\tau, \rho)$ per weight belief plus
one stored factor message per mini-batch per weight.

\subsection{Illustrative Experiment: 1D Regression}
\label{sec:bnn:experiment}

\paragraph{Setup.}
We evaluate Algorithm~\ref{alg:bnn} on a scalar regression task under a
correctly specified model: the ground truth $f$ is itself a draw from the
BNN prior, which ensures that the model's parameters can be learned at all,
a standard simulation-study design that isolates inference quality from model
misspecification.
We generate $N = 200$ noise-corrupted observations at inputs drawn uniformly
from $[-2.5, 1.5]$ and run DMA to recover the posterior.
A fixed 7-component feature map $\varphi$ is prepended and standardised;
a two-hidden-layer leaky-ReLU network ($d = 6, 5$) with Gaussian likelihood
($\beta = 0.2$) is trained for up to 200 epochs with mini-batches of 20.
Full architecture and hyperparameter details 
(prior scale $\sigma_0$, likelihood $\beta$, activation slopes $\alpha$) 
are in Appendix~\ref{app:experiments}.

\vspace{-0.1cm}

\paragraph{Results.}
Figure~\ref{fig:bnn:regression} shows the posterior predictive distribution
and the posterior weight beliefs after training.
The predictive mean tracks the true function closely within the training region
$[-2.5, 1.5]$; the $\pm 2\sigma$ intervals widen in the extrapolation regions,
where the training data provides no information, a qualitatively
appropriate epistemic uncertainty that grows where the model is uninformed.

\begin{figure}[t]
  \centering
  \includegraphics[width=0.9\linewidth]{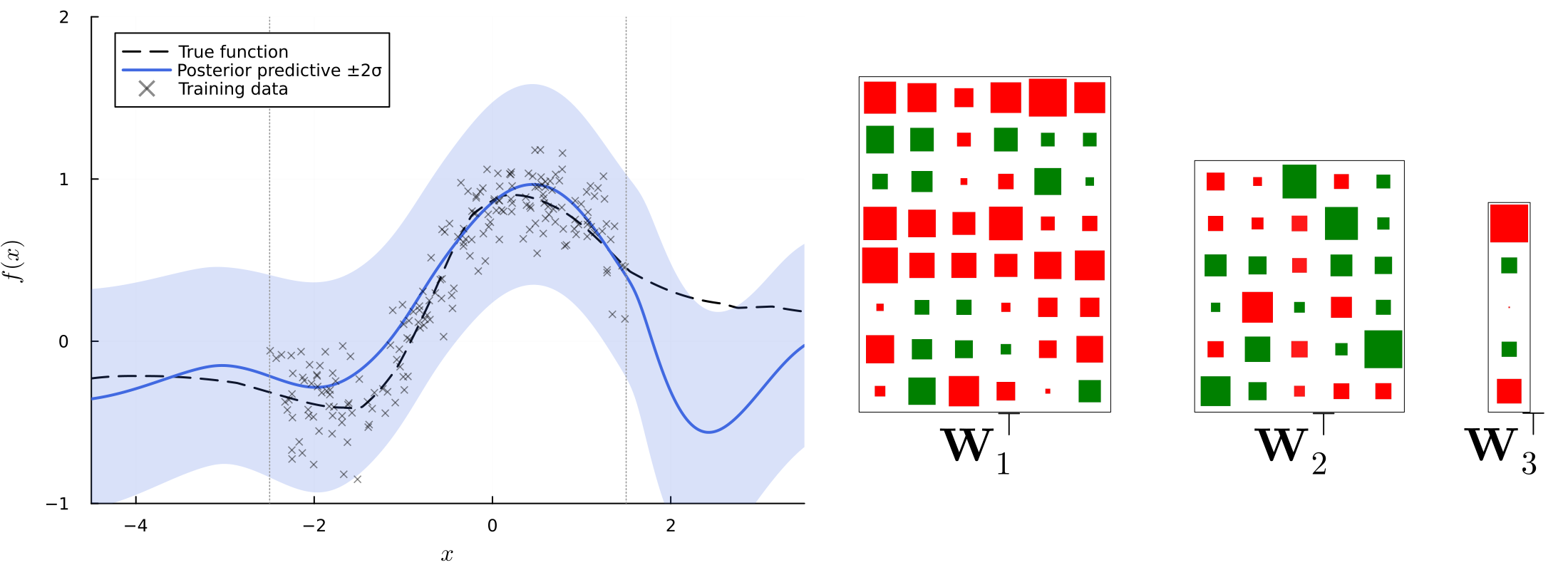}
    \vspace{-0.2cm}
  \caption{%
    \emph{Left}: Posterior predictive mean (solid) and $\pm 2\sigma$ intervals
    (shaded) versus the true data-generating function (dashed), where
    $\sigma{=}\sqrt{\mathrm{Var}[f(x)] + \beta^2}$ combines posterior
    variance with observation noise ($\beta{=}0.2$).
    Training points shown as crosses ($N=200$, $x\in[-2.5,\,1.5]$);
    dashed verticals mark the training boundaries.
    \emph{Right}: Hinton diagram of posterior weight beliefs after training.
    Square area $\propto$ posterior mean magnitude; transparency $\propto$
    posterior variance (opaque = certain).
    \vspace{-0.4cm}}
  \label{fig:bnn:regression}
\end{figure}


\vspace{-0.1cm}
\paragraph{Discussion.}
The experiment illustrates the behavior predicted by all three corollaries:
no EP-style inner-loop iteration is needed
(Corollary~\ref{cor:single-pass}), all weight variances remain positive throughout
(Corollary~\ref{cor:no-neg-prec}), and the posteriors are non-trivial distributions
because linear factors satisfy the consistency condition and contribute
small $\delta$ (Corollary~\ref{cor:dirac-exact}).
A comparison with Adam \citep{KingmaBa2015}, AdamW \citep{LoshchilovHutter2019}
at four weight-decay values, and the diagonal
Laplace approximation \citep{MacKay1992} (Appendix~\ref{app:experiments})
shows, on this correctly specified 1D task, DMA requires few epochs
(median epoch~16) with no gradient learning-rate hyperparameter; in the example, over 20
seeds DMA achieves a median extrapolation NLL of $0.80$ versus $4.05$ for AdamW
(Appendix~\ref{app:experiments:adamw}).
Both methods show similar calibration error over 20 seeds (DMA median $\Delta = -0.09$,
diagonal Laplace $\Delta = -0.10$); neither is systematically overconfident
(Appendix~\ref{app:experiments:calibration}).
Under model mismatch (data generated by a wider network), DMA's posterior
predictive continues to widen outside the training range while Adam
provides no epistemic uncertainty; see Appendix~\ref{app:mismatch}.

\vspace{-0.1cm}

\paragraph{Larger Networks.}
While the first experiment uses a small network for demonstration,
Appendix~\ref{app:large} evaluates DMA on a larger 1\,932-weight, four-output
network ($6{\to}6{\to}12{\to}48{\to}24{\to}4$, $N=1500$).
DMA obtains competitive results in 3 epochs ($0.72\,\text{s}$ total),
compared with $1.7\,\text{s}$ for Adam ($\eta=0.01$, requiring $\sim$100 epochs).
We used 100 mini-batches with the default settings in Appendix~\ref{app:experiments:setup}.


\section{Related Work}
\label{sec:related}

\paragraph{Marginal-error bounds from message perturbations.}
\citet{IhlerFisherWillsky2005} bound marginal changes under total-variation
perturbations via a graph-global contraction constant.
Theorem~\ref{thm:dma-master} is complementary: it handles projected messages
where $m \notin \mathcal{Q}$, is edge-local, and requires no contraction condition.

\vspace{-0.2cm}
\paragraph{Expectation propagation and variants.}
EP \citep{Minka2001} recovers outgoing messages by projecting the marginal and
dividing out the cavity (eq.~\ref{eq:bg:ep-division}).
Power EP \citep{Minka2004power}, $\alpha$-EP \citep{Minka2005}, and Stochastic
EP \citep{LiHernandezLobatoTurner2015} generalise the projection or reduce
memory, but all retain cavity division and its pathologies.
\citet{HernandezLobatoAdams2015} apply EP to a BNN, requiring per-example
cavity computation and iterative sweeps; DMA eliminates both
(Definition~\ref{def:dma}).

\vspace{-0.2cm}
\paragraph{Assumed density filtering.}
ADF \citep{Lauritzen1992,OpperWinther1999} also avoids cavity division, but
for a structural reason: it is a forward-only sequential algorithm that absorbs
each observation into a running prior and never performs a backward sweep.
Weight beliefs are therefore never updated via backward messages --- the step
DMA's product backward message (Section~\ref{sec:factors:product}) is designed
to handle.
Table~\ref{tab:method-comparison} (Appendix~\ref{sec:background}) compares
ADF, EP, VMP, and DMA.

\vspace{-0.2cm}
\paragraph{Variational inference and approximate BNNs.}
VMP \citep{WinBis2005a} and ELBO-based methods use the same marginal-first
division as EP; for deterministic factors this collapses the posterior to a
Dirac delta (Appendix~\ref{sec:background}).
Bayes by Backprop \citep{BlundellEtAl2015}, MC Dropout \citep{GalGhahramani2016},
and SWAG \citep{MaddoxEtAl2019} approximate the posterior via ELBO or SGD
trajectory statistics; none maintains explicit factor-graph message structure.

\vspace{-0.2cm}
\paragraph{Laplace approximation and gradient-based uncertainty.}
Diagonal Laplace \citep{MacKay1992} fits a Gaussian at the MAP post-hoc.
IVON \citep{shen2024variational} tracks a diagonal
natural-gradient estimate online; it propagates uncertainty through gradient
statistics rather than factor-graph messages and provides no backward product
message to weight beliefs.

\vspace{-0.2cm}
\paragraph{Message-passing systems.}
TrueSkill \citep{HerbrichMinkaGraepel2006} demonstrates EP at scale on
conjugate factors only, so the Dirac-delta collapse does not arise.
Infer.NET \citep{MinkaWinnGuiverKnowlesZaykov2018} provides a general EP/VMP
engine with pluggable message operators, illustrating the demand for
modular approximate-message infrastructure.


\section{Discussion and Conclusions}
\label{sec:discussion}

\paragraph{Summary.}
We have introduced direct message approximation (DMA), a framework for
approximate inference on factor graphs that replaces the marginal-projection
step of EP with a direct approximation of the outgoing message.
The guarantees form a two-tier structure.
\textbf{Theorem~\ref{thm:dma-master}} (\emph{proper} messages, any graph): bounds
the marginal KL at any edge by $O(\delta)$ with an edge-local,
topology-independent prefactor; three corollaries establish: no EP-style inner-loop iteration,
no negative-precision messages, and Dirac-input consistency at all linear factors.
\textbf{Theorem~\ref{thm:product-bwd-bound}} (\emph{improper} product-backward
message, concentrated-input regime): establishes a $C/r^2$ upper bound where
Theorem~\ref{thm:dma-master} does not apply; Appendix~\ref{app:validation:snr}
confirms this rate empirically against importance-sampling reference marginals
(Figure~\ref{fig:validation:snr-sweep}).

The BNN instantiation shows that DMA's structural guarantees carry through to 
larger models without sacrificing epistemic quality for scalability.
DMA fits the training data well and requires no gradient learning-rate
hyperparameter, avoiding the manual learning-rate tuning that gradient-based methods require.
Extrapolation uncertainty is as well calibrated as diagonal Laplace (limited to smaller models);
the widening intervals reflect genuine epistemic
uncertainty propagated structurally through the factor graph, not a seed artefact.
Under model mismatch, where no weight setting can fit the true function exactly,
DMA's posterior captures residual uncertainty and produces useful predictive
intervals beyond the training range; Adam, lacking uncertainty quantification,
extrapolates without any such signal.
The procedure scales to substantially larger networks without algorithmic change
and at wall-clock time comparable to Adam (Appendices~\ref{app:experiments}--\ref{app:large}).

\vspace{-0.2cm}
\paragraph{Limitations.}
The master theorem requires proper outgoing messages; factors whose exact messages are improper 
(such as the product-backward) need separate treatment beyond the general recipe.
DMA maintains a fully factorised Gaussian belief; block-diagonal extensions are
compatible in principle but require moment formulas under non-diagonal inputs.

\vspace{-0.2cm}
\paragraph{Future directions.}
Bounding accumulated error after $K > 1$ sweeps on cyclic graphs is the most
pressing open problem; closing this gap would give end-to-end convergence guarantees.
The construction recipe extends to any factor with closed-form moment formulas,
making the factor library a matter of mathematical derivation rather than framework change.
Student-$t$ and mixture-of-Gaussians messages are natural next targets for
heavier-tailed or multimodal posteriors.


\newpage
\subsection*{AI use statement}




In this work, we used generative AI tools for language editing and minor
implementation support, including assistance with code-related tasks during
development. We have not used generative AI tools to generate the scientific
contributions, theoretical results, proofs, experimental findings, or
interpretations presented in this work. We have reviewed and verified all
AI-assisted text and implementation changes, including checking the
correctness of any AI-assisted code. We take responsibility for the final
content of this work, including text, claims, and artifacts produced with the
aid of generative AI.

\subsection*{Ethics statement}



We do not identify any ethical concerns specific to this work.

\subsection*{Reproducibility statement}



We provide an anonymous Git repository containing the implementation of DMA,
the experimental configurations, and scripts used to reproduce the reported
results. The main paper and Appendix~B provide the assumptions, derivations,
and complete proofs of the theoretical results, while Appendix~D describes
the experimental setups and hyperparameters. Appendix~E provides the
experimental procedures and reference computations used to validate the
theoretical error bounds. Together, these materials provide the information
needed to reproduce both the theoretical and empirical results reported in
this work.

All results were produced on a consumer laptop (Intel Core Ultra~5 225U, 32\,GB RAM, 
single-threaded using Julia).

An anonymous code repository can be found here:

\url{https://anonymous.4open.science/r/DMA-Anom-0484/README.md}



\newpage


\bibliography{main_bibliography}
\bibliographystyle{iclr2027_conference}



\appendix

\bigskip
\bigskip

\section*{Appendix Overview}

\noindent
\begin{tabular}{@{}lp{0.88\linewidth}@{}}
  \textbf{A} & \textbf{Background} --- factor graphs, sum-product algorithm,
    $\alpha$-divergence family, EP, ADF, and VMP as marginal-based methods,
    and a comparison table. \\[2pt]
  \textbf{B} & \textbf{Proofs} --- all theorem and proposition proofs,
    including the master theorem (Thm.~\ref{thm:dma-master}),
    product-factor constructions (Props.~\ref{prop:product-fwd}--\ref{prop:product-bwd}
    and Thm.~\ref{thm:product-bwd-bound}), ReLU-factor constructions,
    and a theorem-coverage table by factor and direction. \\[2pt]
  \textbf{C} & \textbf{Empirical validation} --- numerical verification of
    the $O(\delta+\sqrt{\delta})$ master-theorem rate and the $O(1/r^2)$
    concentrated-input rate against importance-sampling reference marginals. \\[2pt]
  \textbf{D} & \textbf{Inference algorithm} --- complete pseudocode for
    Algorithm~\ref{alg:bnn}. \\[2pt]
  \textbf{E} & \textbf{Experimental details} --- feature map, architecture,
    hyperparameters, extended Adam, AdamW, and diagonal-Laplace comparisons,
    IVON sweep, EP exclusion rationale, and calibration curves. \\[2pt]
  \textbf{F} & \textbf{Model mismatch} --- 20-seed experiment comparing DMA
    and Adam when the inference model is narrower than the data-generating network. \\[2pt]
  \textbf{G} & \textbf{Larger-scale BNN} --- DMA on a 1\,932-weight,
    four-output network; baseline-choice rationale; calibration at scale. \\
\end{tabular}
\bigskip

\newpage


\section{Background (Notation and Review)}
\label{sec:background}

\subsection{Factor Graphs and the Sum-Product Algorithm}
\label{sec:background:fg}

A \emph{factor graph} is a bipartite graph that encodes the factorisation of a
joint density over $n$ variables $X_1, \ldots, X_n$ into $m$ local potentials:
\begin{equation}
  p(x_1, \ldots, x_n)
    = \prod_{i=1}^{m} f_i\!\bigl(\mathbf{x}_{\neighbours{f_i}}\bigr),
  \label{eq:bg:joint}
\end{equation}
where $\neighbours{f_i} \subseteq \{1,\ldots,n\}$ is the neighbourhood of factor
$f_i$ and $\neighbours{X_j} \subseteq \{1,\ldots,m\}$ is the neighbourhood of
variable $X_j$.

On a factor \emph{tree}, the sum-product algorithm computes the exact marginal
$p_{X_j}$ of every variable via four message equations
\citep{KschischangFreyLoeliger2001}:
\begin{align}
  p_{X_j}(x_j)
    &= \prod_{i \in \neighbours{X_j}} \msgto{f_i}{X_j}(x_j),
    \label{eq:bg:marginal} \\
  \msgto{f_i}{X_j}(x_j)
    &= \int f_i\!\bigl(\mathbf{x}_{\neighbours{f_i}}\bigr)
       \prod_{k \in \neighbours{f_i}\setminus\{j\}}
         \msgto{X_k}{f_i}(x_k)
       \intd{\mathbf{x}_{\neighbours{f_i}\setminus\{j\}}},
    \label{eq:bg:f-to-x} \\
  \msgto{X_j}{f_i}(x_j)
    &= \prod_{k \in \neighbours{X_j}\setminus\{i\}} \msgto{f_k}{X_j}(x_j).
    \label{eq:bg:x-to-f}
\end{align}
The factor-to-variable message~\eqref{eq:bg:f-to-x} integrates the factor
against all incoming variable messages \emph{except} the one at the target
edge; similarly, the variable-to-factor message~\eqref{eq:bg:x-to-f} is the
product of all incoming factor messages except at the target edge.
Multiplying the two messages on any edge recovers the \emph{marginal-edge
identity}:
\begin{equation}
  p_{X_j}(x_j)
    = \msgto{f_i}{X_j}(x_j) \cdot \msgto{X_j}{f_i}(x_j).
  \label{eq:bg:marginal-edge}
\end{equation}

\paragraph{Conjugacy constraint.}
The integral in~\eqref{eq:bg:f-to-x} is in closed form only when $f_i$ is
conjugate to the incoming message family~$\mathcal{Q}$.
For Gaussian messages, linear factors $\dirac{x_j - \mathbf{a}^\top
\mathbf{x}_{-j}}$ and Gaussian-likelihood factors are conjugate; the scalar
product factor $\dirac{z - xy}$ and activation factors such as ReLU and softmax
are not, and~\eqref{eq:bg:f-to-x} does not yield a Gaussian in closed form.

\subsection{The Projection Step and the $\alpha$-Divergence Family}
\label{sec:background:proj}

When the exact message~\eqref{eq:bg:f-to-x} is intractable, it is replaced by
a projection onto a tractable family~$\mathcal{Q}$.
The $\alpha$-divergence \citep{Amari1985,Minka2005}
\[
  D_\alpha(p \,\|\, q)
    = \frac{1}{\alpha(1-\alpha)}
      \Bigl(1 - \int p(x)^\alpha\, q(x)^{1-\alpha} \intd{x}\Bigr),
    \quad \alpha \in \mathbb{R},
\]
unifies many classical algorithms:
$\lim_{\alpha\to 1} D_\alpha = \KL{p}{q}$ (forward KL, used by EP)
and $\lim_{\alpha\to 0} D_\alpha = \KL{q}{p}$ (reverse KL, used by VMP).
The projection step in both algorithms relies on the following standard result.

\begin{theorem}[Moment-Matching Minimiser of KL]
\label{thm:kl-min}
Let $p$ be an arbitrary density and $\mathcal{Q} = \{q(\cdot\mid\boldsymbol
\theta) = \exp(\boldsymbol\theta^\top \mathbf{T}(\cdot) - A(\boldsymbol\theta))\}$
an exponential family with sufficient statistic~$\mathbf{T}$.
The minimiser $q^* = \arg\min_{q \in \mathcal{Q}} \KL{p}{q}$ satisfies
\begin{equation}
  \mathbb{E}_{q^*}[\mathbf{T}(X)] = \mathbb{E}_{p}[\mathbf{T}(X)].
  \label{eq:bg:moment-match}
\end{equation}
\end{theorem}

For Gaussian $\mathcal{Q}$ the condition~\eqref{eq:bg:moment-match} reduces to
matching the mean and variance of~$p$.
Theorem~\ref{thm:kl-min} will serve as the projection step in both EP and
in the DMA construction of Section~\ref{sec:dma:construction}.

\subsection{Prior Approximate Inference Methods}
\label{sec:background:ep-vmp}

\paragraph{Expectation propagation.}
EP \citep{Minka2001} approximates each factor-to-variable message in three
steps.
\begin{enumerate}
  \item Form the \emph{true unnormalised marginal}
    $p_{X_j} = \msgto{f_i}{X_j} \cdot \msgto{X_j}{f_i}$
    using the exact integral~\eqref{eq:bg:f-to-x} (combined with
    \eqref{eq:bg:marginal-edge}).
  \item Project $p_{X_j}$ onto $\mathcal{Q}$ via Theorem~\ref{thm:kl-min},
    obtaining $\hat{p}_{X_j} \in \mathcal{Q}$.
  \item Recover the approximate message by dividing out the incoming message:
\end{enumerate}
\begin{equation}
  \hat{m}_{f_i \to X_j}(\cdot)
    = \hat{p}_{X_j}(\cdot) \;/\; \msgto{X_j}{f_i}(\cdot).
  \label{eq:bg:ep-division}
\end{equation}
The approximation is exact when $p_{X_j} \in \mathcal{Q}$; otherwise EP
introduces a local error and must iterate to a fixed point.

\paragraph{Variational message passing.}
VMP \citep{WinBis2005a} instead maximises the evidence lower bound
$\mathcal{L}(\hat{p}) = \mathbb{E}_{\hat{p}}[\log p - \log \hat{p}]$
under a fully factorised approximation $\hat{p}(\mathbf{x}) = \prod_k
\hat{p}_{X_k}(x_k)$, which is equivalent to minimising $\KL{\hat{p}}{p}$.
Coordinate ascent on $\hat{p}_{X_k}$ yields
\begin{equation}
  \hat{p}_{X_k}(x_k)
    \propto \exp\!\bigl(
      \mathbb{E}_{\hat{p}_{-k}}[\log f_i(\mathbf{x}_{\neighbours{f_i}})]
    \bigr),
  \label{eq:bg:vmp-update}
\end{equation}
where the expectation is over all neighbours of $f_i$ except $X_k$.
When $\log f_i$ decomposes linearly in each argument separately, this
produces a message $\hat{m}_{f_i \to X_k} \propto \hat{p}_{X_k}$;
again the message is recovered by dividing out the incoming message at the
same edge, so VMP also obeys the form~\eqref{eq:bg:ep-division}.

\paragraph{Assumed-density filtering.}
ADF \citep{Lauritzen1992,OpperWinther1999} processes observations sequentially,
projecting the posterior onto $\mathcal{Q}$ after each observation and using it
as the prior for the next; no stored factor message needs to be divided out.
ADF therefore also avoids cavity division, but for a structural reason: it is
an online, forward-only algorithm that approximates the \emph{marginal} at each
step (Table~\ref{tab:method-comparison}).
Crucially, standard ADF's forward-only architecture provides no mechanism for
propagating output likelihood information backward to update upstream weight
beliefs --- the technically difficult step that DMA's backward product message
is designed to handle.

\paragraph{Pathologies of the marginal-first design.}
EP and VMP approximate $\hat{p}_{X_j}$ and recover the message by division
\eqref{eq:intro:ep-division}, inducing three pathologies.
\emph{(i) Iteration (EP and VMP).}
The division makes $\hat{m}_{f_i \to X_j}$ depend on $\msgto{X_j}{f_i}$,
so updating one edge invalidates neighbouring messages; convergence requires
repeated sweeps.
\emph{(ii) Negative precision (Gaussian EP).}
When $\hat{p}_{X_j}$ is wider than $\msgto{X_j}{f_i}$, the ratio
\eqref{eq:bg:ep-division} yields a Gaussian with negative precision, an
invalid distribution that propagates downstream.
\emph{(iii) Dirac-delta collapse (VMP).}
For a deterministic factor $f_i(\mathbf{x}) = \dirac{x_j - g(\mathbf{x}_{-j})}$,
the ELBO update~\eqref{eq:bg:vmp-update} forces $\hat{p}_{X_j}$ to a Dirac
delta; since every parameterised layer in the factorisation used here is such a
factor, VMP cannot propagate non-trivial posterior variance backward to the weights.
Section~\ref{sec:dma} introduces DMA, which avoids this design by approximating
the message directly.
\medskip

\begin{table}[h]
  \centering
  \caption{%
    Comparison of approximate inference frameworks on properties relevant to
    BNN inference. Below the first row, $\checkmark$ indicates a desirable property.
    \emph{Approximates}: whether the method targets the factor-to-variable
    message directly or approximates the marginal at each edge.
    \emph{No cavity division}: whether updates avoid dividing out a stored
    incoming message (eq.~\ref{eq:bg:ep-division}).
    \emph{No inner-loop iteration}: whether a single forward/backward sweep
    per example suffices, with no per-example fixed-point iteration.
    \emph{Backward weight updates}: whether weight beliefs are updated via
    a backward message sweep.
    \emph{Consistency axiom / edge-local KL bound}: see Definition~\ref{def:dma}
    and Theorem~\ref{thm:dma-master}.}
  \label{tab:method-comparison}
  \begin{tabular}{lcccc}
    \toprule
    & ADF & EP & VMP & \textbf{DMA (ours)} \\
    \midrule
    Approximates            & marginal & marginal & marginal & \textbf{message} \\
    \midrule
    Backward weight updates & $\times$ & $\checkmark$ & $\checkmark$ & $\checkmark$ \\
    No cavity division      & $\checkmark$ & $\times$ & $\times$ & $\checkmark$ \\
    No inner-loop iteration & $\checkmark$ & $\times$ & $\times$ & $\checkmark$ \\
    Consistency axiom       & $\times$ & $\times$ & $\times$ & $\checkmark$ \\
    Edge-local KL bound     & $\times$ & $\times$ & $\times$ & $\checkmark$ \\
    \bottomrule
  \end{tabular}
\end{table}


\newpage

\section{Proofs}
\label{app:proofs}

\subsection{Proof of Theorem~\ref{thm:kl-min}
  (Moment-Matching Minimiser of KL)}
\label{app:proofs:kl-min}

\begin{proof}
For $q(\cdot\mid\boldsymbol\theta) = \exp(\boldsymbol\theta^\top
\mathbf{T}(\cdot) - A(\boldsymbol\theta))$ in the exponential family:
\[
  \KL{p}{q}
    = \int p(x) \log p(x)\intd{x} - \int p(x) \bigl(\boldsymbol\theta^\top \mathbf{T}(x)
       - A(\boldsymbol\theta)\bigr)\intd{x}
    = -H(p) + A(\boldsymbol\theta) - \boldsymbol\theta^\top \mathbb{E}_p[\mathbf{T}].
\]
Differentiating with respect to $\boldsymbol\theta$ and setting to zero:
\[
  \nabla_{\boldsymbol\theta} \KL{p}{q}
    = \nabla A(\boldsymbol\theta) - \mathbb{E}_p[\mathbf{T}] = \mathbf{0}.
\]
For exponential families, $\nabla A(\boldsymbol\theta) =
\mathbb{E}_{q(\cdot\mid\boldsymbol\theta)}[\mathbf{T}]$ (the moment-generating
identity), so the stationary condition is exactly
$\mathbb{E}_{q^*}[\mathbf{T}(X)] = \mathbb{E}_p[\mathbf{T}(X)]$.
The KL is convex in $\boldsymbol\theta$ (since $A$ is log-partition and hence
convex), so the stationary point is the unique minimum.
For Gaussian $\mathcal{Q}$, $\mathbf{T}(x) = (x, x^2)$ and the condition
reduces to matching the first two moments.
\end{proof}

\subsection{Proof of Theorem~\ref{thm:dma-master}
  (DMA Master Theorem)}
\label{app:proofs:master}

\begin{proof}
Write $m := \msgto{f}{X_j}$, $\hat m := \hat{m}_{f \to X_j}$, $Z$, $\hat Z$,
$\delta$ as in the theorem statement; assume $\|\msgto{X_j}{f}\|_\infty < \infty$.
For $a, b > 0$ define the scalar Bregman divergence
$D(a,b) := a\log(a/b) - a + b \geq 0$,
with the same generator $\phi(t)=t\log t$ as the KL\@.
The key identity is the \emph{exact decomposition}
\begin{equation}
  Z\;\KL{\tfrac{p_{X_j}}{Z}}{\tfrac{\hat{p}_{X_j}}{\hat{Z}}}
  \;=\;
  \int \msgto{X_j}{f}(x)\,D\!\bigl(m(x),\hat m(x)\bigr)\intd{x}
  \;-\;
  D(Z,\hat Z).
  \label{eq:proof:bregman}
\end{equation}
\emph{Proof of \eqref{eq:proof:bregman}.}
Expand the left side:
$Z\,\KL{p/Z}{\hat p/\hat Z}
  = \int\msgto{X_j}{f}\, m\log(m/\hat m) + Z\log(\hat Z/Z)$.
Adding and subtracting $\int\msgto{X_j}{f}(m-\hat m) = Z - \hat Z$ gives
\[
  \int\msgto{X_j}{f}(x)\underbrace{\bigl[m(x)\log(m(x)/\hat m(x))-m(x)+\hat m(x)\bigr]}_{D(m(x),\hat m(x))}\intd{x}
  + \underbrace{(Z-\hat Z)+Z\log(\hat Z/Z)}_{-D(Z,\hat Z)},
\]
where the second group equals $-D(Z,\hat Z) = -(Z\log(Z/\hat Z)-Z+\hat Z)$. \hfill$\square$

Since $D(Z,\hat Z)\geq 0$, discarding it and dividing by $Z$ gives
\[
  \KL{\tfrac{p_{X_j}}{Z}}{\tfrac{\hat{p}_{X_j}}{\hat{Z}}}
  \;\leq\; \frac{1}{Z}\int\msgto{X_j}{f}(x)\,D\!\bigl(m(x),\hat m(x)\bigr)\intd{x}.
\]
Since $D(m(x),\hat m(x))\geq 0$ and $\msgto{X_j}{f}\geq 0$,
bounding $\msgto{X_j}{f}(x)\leq\|\msgto{X_j}{f}\|_\infty$ is valid.
Using $\int D(m,\hat m)\intd{x} = \int m\log(m/\hat m)\intd{x} = \delta$
(the $-m+\hat m$ terms integrate to zero since $\int m \intd{x} =\int\hat m \intd{x} =1$)
gives the stated bound.
\end{proof}

\subsection{Proofs of Corollaries~\ref{cor:dirac-exact},
  \ref{cor:single-pass}, and~\ref{cor:no-neg-prec}}
\label{app:proofs:corollaries}

\begin{proof}[Proof of Corollary~\ref{cor:dirac-exact}]
Take each incoming message as $\mathcal{N}(\bar x_k, \sigma_k^2)$ with
$\sigma_k > 0$.
For every $\sigma_k > 0$ both $\msgto{f}{X_j}$ and $\hat{m}_{f \to X_j}$
are normalizable densities, so the KL is well-defined and finite.

By the construction recipe (Section~\ref{sec:dma:construction}),
$\hat{m}_{f \to X_j} = \mathcal{N}(\mu_\sigma, \sigma_\sigma^2)$ where
$\mu_\sigma = \mathbb{E}[g(\mathbf{X}_{-j})]$ and
$\sigma_\sigma^2 = \mathrm{Var}[g(\mathbf{X}_{-j})]$ are the first two moments
of the exact outgoing message.
Since the DMA matches these moments exactly, for any distribution $m$
and its moment-matched Gaussian projection $\hat{m}$:
\[
  E_m[\log \hat{m}(X)] = -\tfrac{1}{2}\log(2\pi\sigma_\sigma^2) - \tfrac{1}{2}
  = -H(\hat{m}),
\]
which gives the identity
\[
  \KL{\msgto{f}{X_j}}{\hat{m}_{f \to X_j}}
    = H(\hat{m}_{f \to X_j}) - H(\msgto{f}{X_j}).
\]

The $C^2$ assumption in a neighbourhood of $\bar{\mathbf{x}}_{-j}$ provides the
second-order Taylor expansion
\[
  g(\bar{\mathbf{x}}_{-j} + \mathbf{h})
    = g(\bar{\mathbf{x}}_{-j})
      + \nabla g(\bar{\mathbf{x}}_{-j})^\top \mathbf{h}
      + \tfrac{1}{2}\mathbf{h}^\top \nabla^2 g(\bar{\mathbf{x}}_{-j})\,\mathbf{h}
      + o(\|\mathbf{h}\|^2).
\]
The nonzero linear term is $O(\|\sigma\|)$; the nonlinear remainder is
$O(\|\sigma\|^2)$, hence asymptotically negligible relative to the linear term.
Combined with tail regularity, this makes $H(\msgto{f}{X_j})$ asymptotically
determined by the linearisation, giving
$H(\msgto{f}{X_j}) = \tfrac{1}{2}\log(2\pi e\,\sigma_\sigma^2) + o(1)$
and hence $\KL{\msgto{f}{X_j}}{\hat{m}_{f \to X_j}} \to 0$.
We verify this explicitly factor-class by factor-class.

\textit{Conjugate factors} (Gaussian prior, Gaussian likelihood): the exact
message is already Gaussian, so $\msgto{f}{X_j} = \hat{m}_{f \to X_j}$ and
$\delta = 0$ for all $\sigma_k$.

\textit{Product forward} (Proposition~\ref{prop:product-fwd}): $g(x,y)=xy$ is
$C^\infty$ everywhere with $\nabla g = (y,x)^\top \neq \mathbf{0}$ when
$(\mu_x,\mu_y)\neq(0,0)$.
The DMA computes the exact moments of $Z = XY$; $m$ is a variance-gamma
distribution whose entropy satisfies
$H(m) = \tfrac{1}{2}\log \sigma_z^2 + c + o(1)$ as $\sigma_x,\sigma_y\to 0$,
matching the leading term of $H(\hat{m})$ with the same constant $c$, so
$\delta = H(\hat{m}) - H(m) \to 0$.

\textit{ReLU forward and backward} (Propositions~\ref{prop:relu-fwd}--\ref{prop:relu-bwd}):
leaky-ReLU is piecewise-$C^\infty$ with a single non-smooth point at the origin.
When the concentration point $\bar x$ is away from the origin, the $C^2$ argument
applies directly.
At the origin, the DMA uses exact truncated-Gaussian moments; Mills ratio bounds
give $\sigma_\sigma^2 \to 0$ and $H(\hat{m}) - H(m) \to 0$ via the explicit
moment formulas.

Hence $\delta \to 0$ in all cases and the bound~\eqref{eq:dma:master} vanishes.
\end{proof}

\begin{proof}[Proof of Corollary~\ref{cor:single-pass}]
The recipe in Section~\ref{sec:dma:construction} integrates only the incoming
messages at neighbours $\{X_k\}_{k \neq j}$; $\msgto{X_j}{f}$ does not
appear.  Therefore updating $\hat{m}_{f \to X_j}$ does not change
$\msgto{X_j}{f}$, and no re-computation of neighbouring messages is
triggered.
\end{proof}

\begin{proof}[Proof of Corollary~\ref{cor:no-neg-prec}]
The moment-matching step minimises $\KL{p}{\hat{m}}$ over $\mathcal{Q}$.
For the Gaussian family the minimiser has variance $\mathrm{Var}_p[X] \ge 0$
whenever $p$ has finite second moment.
If $\mathrm{Var}_p[X] > 0$ the message is a proper Gaussian with strictly
positive precision $\rho = 1/\mathrm{Var}_p[X] > 0$.
In the Dirac limit $(\sigma_k \to 0)$, $\mathrm{Var}_p[X] \to 0$ and the
precision $\rho \to \infty$; the message degenerates to a Dirac delta, which
is the correct limiting message (Corollary~\ref{cor:dirac-exact}).
At no point is $\rho$ negative.
\end{proof}

\subsection{Proofs of Propositions~\ref{prop:product-fwd}
  and~\ref{prop:product-bwd} (Product Factor DMA)}
\label{app:proofs:product}

\begin{proof}[Proof of Proposition~\ref{prop:product-fwd}]
For independent $X \sim \mathcal{N}(\mu_x, \sigma_x^2)$ and
$Y \sim \mathcal{N}(\mu_y, \sigma_y^2)$, the first two moments of
$Z = XY$ follow directly from independence and the law of total variance:
\begin{align*}
  \mathbb{E}[Z] &= \mathbb{E}[X]\mathbb{E}[Y] = \mu_x\mu_y, \\
  \mathbb{E}[Z^2]
    &= \mathbb{E}[X^2]\mathbb{E}[Y^2]
    = (\mu_x^2+\sigma_x^2)(\mu_y^2+\sigma_y^2), \\
  \mathrm{Var}[Z]
    &= \mathbb{E}[Z^2]-\mathbb{E}[Z]^2
    = \sigma_x^2\sigma_y^2 + \mu_x^2\sigma_y^2 + \mu_y^2\sigma_x^2.
\end{align*}
The DMA is the moment-matching Gaussian projection via Theorem~\ref{thm:kl-min}.
\end{proof}

\begin{proof}[Proof of Proposition~\ref{prop:product-bwd} (sketch)]
We compute the first two moments of $X = Z/Y$ under independent
$Y \sim \mathcal{N}(\mu_y, \sigma_y^2)$ and $Z \sim \mathcal{N}(\mu_z, \sigma_z^2)$.
The ratio distribution has no finite moments under the Gaussian directly
(the $1/|y|$ singularity at the origin diverges), so we use a log-normal
intermediate.

\begin{lemma}[Log-Normal Approximation]
\label{lem:lognormal}
For $W \sim \mathcal{N}(\mu_w, \sigma_w^2)$ with $|\mu_w|/\sigma_w \gg 1$,
the random variable $\log|W|$ is approximately $\mathcal{N}(\log|\mu_w|,\,
\sigma_w^2/\mu_w^2)$ by the first-order delta method
(Taylor expansion of $\log|\cdot|$ around $\mu_w$).
\end{lemma}

Apply Lemma~\ref{lem:lognormal} to $Y$ and $Z$, treating both as log-normal.
Since $\log|Z/Y| = \log|Z| - \log|Y|$, the difference of two independent
normals is normal with mean $m = \log|\mu_z| - \log|\mu_y|$ and variance
$v = \sigma_z^2/\mu_z^2 + \sigma_y^2/\mu_y^2$.

For a log-normal random variable $W$ with $\log|W| \sim \mathcal{N}(m, v)$,
the exact log-normal moments are $\mathbb{E}[W] = \mathrm{sgn}(\mu_z/\mu_y)
\cdot e^{m+v/2}$ and $\mathbb{E}[W^2] = e^{2m+2v}$ (using independence of
$Z$ and $Y$).
The variance is $\mathrm{Var}[X] = e^{2m+v}(e^v-1)$.

Converting back to natural parameters $\tau_w = \mu_w/\sigma_w^2$,
$\rho_w = 1/\sigma_w^2$, so $m = \log|\tau_z\rho_y/(\tau_y\rho_z)|$ and
$v = \rho_z/\tau_z^2 + \rho_y/\tau_y^2$, and algebraic simplification yields
the formulas in Proposition~\ref{prop:product-bwd}.

\emph{Consistency verification.}
In the limit $\sigma_y \to 0$ ($\rho_y \to \infty$, $\tau_y/\rho_y \to \bar y$):
$m \to \log|\mu_z/\bar y|$ and $v \to \sigma_z^2/\mu_z^2$.
Then $\mu_X = e^{m+v/2} \to |\mu_z/\bar y| \cdot e^{\sigma_z^2/(2\mu_z^2)} \to
\mu_z/\bar y$ as $\sigma_z \to 0$; this matches $g(\bar y) = \bar y$ applied to
the constraint $x = z/y$.
The direct arithmetic verification of the formula
$\hat\tau_x/\hat\rho_x \to \mu_z/\bar y$ and $1/\hat\rho_x \to \sigma_z^2/\bar y^2$
was given in Section~\ref{sec:factors:product}.
\end{proof}

\subsection{Proof of Theorem~\ref{thm:product-bwd-bound}
  (Concentrated-Input Bound for Product Backward Message)}
\label{app:proofs:product-bound}


\begin{remark}[Improper Backward Message]
\label{rem:product-improper}
The exact sum-product backward message to $X$ is
\begin{equation}
  m_{f \to X}(x)
    = \frac{1}{|x|}\,
      \mathcal{N}\!\left(\frac{\mu_z}{x};\, \mu_y,\, \sigma_y^2 + \frac{\sigma_z^2}{x^2}\right),
  \label{eq:product-exact-bwd}
\end{equation}
which is not normalisable: as $|x| \to \infty$ the Gaussian factor
approaches a positive constant and the $1/|x|$ prefactor yields a divergent
integral.
The DMA (Proposition~\ref{prop:product-bwd}) is therefore an approximation
to an improper distribution; its moments are the moments of the ratio
variable $X = Z/Y$ under the joint on $(Y,Z)$, which is always proper.
In the Dirac limit the $1/|x|$ prefactor and the implicit $|x|$ factor from
the Gaussian normalisation cancel, recovering the correct point-mass message.
\end{remark}

The \emph{truncated reference} excludes the non-normalisable tail of the
exact backward message~\eqref{eq:product-exact-bwd}:
\begin{equation}
  \tilde{m}^r_{f \to X}(x)
  \;\propto\;
  \int_{\,|y|\,\geq\,|\mu_y|/2}
    \mathcal{N}(xy;\,\mu_z,\sigma_z^2)\,
    \mathcal{N}(y;\,\mu_y,\sigma_y^2)
  \,\mathrm{d}y,
  \label{eq:product-bwd-trunc}
\end{equation}
which is proper for $r \geq 2$ since the $1/|x|$ singularity is integrable
when $Y$ is bounded away from zero.
For $\mu_y > 0$, the excluded probability is
$P(|Y| < \mu_y/2) = \Phi(-r_y/2) - \Phi(-3r_y/2) \leq \Phi(-r_y/2)$,
which decays faster than any polynomial in $r$.

\paragraph{Outline of the proof} The proof of Theorem~\ref{thm:product-bwd-bound} applies the Pythagorean
identity (Theorem~\ref{thm:kl-min}) to decompose
$\mathrm{KL}(\tilde{m}^r_{f\to X} \| \hat{m}_{f\to X})$ into two terms,
each $\leq C_\kappa/r^2$:
\begin{itemize}
  \item \emph{Non-Gaussianity} of the truncated ratio
    $\tilde{m}^r_{f\to X}$ relative to its moment-matched Gaussian $N^*_X$
    (Lemma~\ref{lem:nongauss-ratio}): bounded via the KL chain rule and
    data-processing inequality, exploiting the fact that $X = Z/Y$ is
    exactly Gaussian conditionally on $Y = y$.
  \item \emph{Moment error} of the log-normal DMA approximation relative to
    $N^*_X$ (Lemmas~\ref{lem:gauss-ln-kl}--\ref{lem:product-moment-error}):
    controlled by a second-order Taylor expansion of $Z/Y$ around
    $(\mu_z, \mu_y)$, with the Gaussian-to-log-normal KL as an intermediate.
\end{itemize}
The three lemmas below establish the ingredients in order.

\begin{lemma}[Gaussian–log-normal KL bound]
\label{lem:gauss-ln-kl}
Let $W \sim \mathcal{N}(\mu_w, \sigma_w^2)$ with $r_w := |\mu_w|/\sigma_w \geq 2$,
and let $\mathrm{LN}(\mu_\ell, \sigma_\ell^2)$ be the log-normal whose mean and
variance match those of $W$ (i.e.\ $\sigma_\ell^2 = \log(1 + 1/r_w^2)$,
$\mu_\ell = \log|\mu_w| - \sigma_\ell^2/2$).
Then
\begin{equation}
  \mathrm{KL}\!\left(\mathrm{LN}(\mu_\ell, \sigma_\ell^2) \,\Big\|\, \mathcal{N}(\mu_w, \sigma_w^2)\right)
  = \frac{3}{4r_w^2} + O\!\left(\frac{1}{r_w^4}\right).
  \label{eq:gauss-ln-kl}
\end{equation}
\end{lemma}

\begin{proof}
Let $L \sim \mathrm{LN}(\mu_\ell, \sigma_\ell^2)$ with matched moments
$\mathbb{E}[L] = |\mu_w|$ and $\mathrm{Var}[L] = \sigma_w^2$.
The cross-entropy of $L$ under $\mathcal{N}(\mu_w, \sigma_w^2)$ equals the
entropy of $\mathcal{N}$ by moment matching:
\[
  H(L,\, \mathcal{N}(\mu_w, \sigma_w^2))
  = -\mathbb{E}_L[\log \mathcal{N}(X;\mu_w, \sigma_w^2)]
  = \tfrac{1}{2}\log(2\pi\sigma_w^2) + \tfrac{\mathrm{Var}[L]+(\mathbb{E}[L]-\mu_w)^2}{2\sigma_w^2}
  = \tfrac{1}{2}\log(2\pi e\sigma_w^2).
\]
The entropy of the log-normal is
$H(L) = \mu_\ell + \tfrac{1}{2}\log(2\pi e\sigma_\ell^2)$.
Therefore
\begin{equation}
  \mathrm{KL}(L \| \mathcal{N}(\mu_w, \sigma_w^2))
  = \tfrac{1}{2}\log(2\pi e\sigma_w^2) - \mu_\ell - \tfrac{1}{2}\log(2\pi e\sigma_\ell^2)
  = \tfrac{1}{2}\log\!\tfrac{\sigma_w^2}{\sigma_\ell^2} - \mu_\ell.
  \label{eq:gauss-ln-kl-raw}
\end{equation}
Substitute $\sigma_\ell^2 = \log(1+1/r_w^2) = 1/r_w^2 - 1/(2r_w^4) + O(1/r_w^6)$,
$\sigma_w^2 = \mu_w^2/r_w^2$, and $\mu_\ell = \log|\mu_w| - \sigma_\ell^2/2$ into
\eqref{eq:gauss-ln-kl-raw}.
Expand $\log(\sigma_\ell^2) = \log(1/r_w^2(1-1/(2r_w^2)+O(1/r_w^4)))
= -2\log r_w - 1/(2r_w^2) + O(1/r_w^4)$, so
\begin{align*}
  \tfrac{1}{2}\log\!\frac{\sigma_w^2}{\sigma_\ell^2}
  &= \tfrac{1}{2}\!\left[\log\frac{\mu_w^2}{r_w^2} - \log\sigma_\ell^2\right] \\
  &= \tfrac{1}{2}\!\left[2\log|\mu_w| - 2\log r_w + 2\log r_w
      + \tfrac{1}{2r_w^2} + O(1/r_w^4)\right] \\
  &= \log|\mu_w| + \frac{1}{4r_w^2} + O(1/r_w^4).
\end{align*}
Using $-\mu_\ell = -\log|\mu_w| + \sigma_\ell^2/2 = -\log|\mu_w| + 1/(2r_w^2) + O(1/r_w^4)$:
\[
  \mathrm{KL}
  = \underbrace{\log|\mu_w| + \frac{1}{4r_w^2}}_{\tfrac{1}{2}\log(\sigma_w^2/\sigma_\ell^2)}
    \underbrace{- \log|\mu_w| + \frac{1}{2r_w^2}}_{-\mu_\ell}
    + O(1/r_w^4)
  = \frac{3}{4r_w^2} + O(1/r_w^4).
\]
The factor $\tfrac{3}{4}$ is the sum of two $O(r_w^{-2})$ contributions: $\tfrac{1}{4}$
from expanding $\log(\sigma_w^2/\sigma_\ell^2)$ and $\tfrac{1}{2}$ from $\sigma_\ell^2/2$;
this is~\eqref{eq:gauss-ln-kl}.
\end{proof}

\begin{lemma}[Moment error of the log-normal backward approximation]
\label{lem:product-moment-error}
Let $r = \min(|\mu_y|/\sigma_y,\, |\mu_z|/\sigma_z) \geq 2$.
Denote by $\mu^*_X$ and $(\sigma^*_X)^2$ the mean and variance of
$\tilde{m}^r_{f \to X}$ (equation~\eqref{eq:product-bwd-trunc}), and by
$\hat\mu_X := \hat\tau_x/\hat\rho_x$ and $\hat\sigma^2_X := 1/\hat\rho_x$
the DMA moments from Proposition~\ref{prop:product-bwd}.
Then
\begin{equation}
  |\hat\mu_X - \mu^*_X|
  = O\!\left(\frac{|\mu_z/\mu_y|}{r^2}\right),
  \qquad
  |\hat\sigma^2_X - (\sigma^*_X)^2|
  = O\!\left(\frac{(\sigma^*_X)^2}{r^2}\right),
  \label{eq:moment-error}
\end{equation}
and consequently
\begin{equation}
  \mathrm{KL}\!\left(N^*_X \,\Big\|\, \hat{m}_{f \to X}\right)
  = O\!\left(\frac{1}{r^2}\right),
  \label{eq:gaussian-kl-moment}
\end{equation}
where $N^*_X = \mathcal{N}(\mu^*_X, (\sigma^*_X)^2)$.
\end{lemma}

\begin{proof}[Proof sketch]
All moments are conditional on $A_r = \{|Y| \geq |\mu_y|/2\}$; the
unconditional ratio $Z/Y$ has no finite moments.
Without loss of generality, $\mu_y, \mu_z > 0$; write $Y = \mu_y(1+\eta U)$, $Z = \mu_z(1+\xi V)$,
$U, V \stackrel{\mathrm{i.i.d.}}{\sim} \mathcal{N}(0,1)$, $\eta = 1/r_y$,
$\xi = 1/r_z$, $a = \mu_z/\mu_y$.

\emph{Central-region Taylor expansion.}
Let $B_r = \{|U| \leq K_r\}$, $K_r = \sqrt{16\log r}$ (as in
Lemma~\ref{lem:nongauss-ratio}), and restrict to $A_r^+ \cap B_r$ (the
exponentially small branches $A_r^-$ and $A_r \cap B_r^c$ are treated below).
On $A_r^+ \cap B_r$, $|\eta U| \leq K_r/r_y \to 0$, so the Taylor expansion
\[
  \frac{1}{1+\eta U} = 1 - \eta U + (\eta U)^2 + R_3, \qquad
  |R_3| \leq 2|\eta U|^3 \leq 2(K_r/r_y)^3,
\]
holds uniformly (the denominator satisfies $1+\eta U \geq 1/2$ on $A_r^+$).
Hence
\[
  X = a(1+\xi V)(1 - \eta U + (\eta U)^2 + R_3).
\]
Taking $\mathbb{E}[\cdot \mid A_r^+ \cap B_r]$ and using
$\mathbb{E}[U^k \mid A_r^+] = \mathbb{E}[U^k] + O(e^{-cr^2})$ for $k \leq 4$
(since $A_r^+$ has probability $1 - O(e^{-cr^2})$):
\[
  \mu^*_X \;=\; a(1 + \eta^2) + O_\kappa\!\left(a\frac{(\log r)^{3/2}}{r^3}\right)
             \;=\; a(1 + \eta^2) + o_\kappa(ar^{-2}),
\]
where the leading remainder comes from $\mathbb{E}[|\eta U|^3 \mid B_r]
= O(\eta^3 K_r^3) = O((\log r)^{3/2}/r^3)$.

\emph{Tail-region bound.}
On $A_r \cap B_r^c$, $|1+\eta U| \geq 1/2$, so $|X| \leq 2|a|(1+|\xi V|)$.
Gaussian tail bounds give
$P(|U| > K_r) = O(r^{-4})$,
so $\mathbb{E}[|X|\,\mathbf{1}_{B_r^c} \mid A_r] = O(|a|\,r^{-4}\cdot\mathrm{poly}(r)) = o(r^{-2})$.
The branch $A_r^-$ has probability $\Phi(-3r_y/2) = O(e^{-9r_y^2/8})$,
contributing negligibly.

\emph{Comparison with the DMA moments.}
Let $D = \eta^2 + \xi^2$.
The DMA mean $\hat\mu_X = a\,e^{D/2}$ (Lemma~\ref{lem:lognormal})
expands to $a(1 + D/2 + O(r^{-4})) = a(1 + \eta^2/2 + \xi^2/2 + O(r^{-4}))$.
Combined with $\mu^*_X = a(1 + \eta^2) + o_\kappa(ar^{-2})$:
\[
  \hat\mu_X - \mu^*_X = a\!\left(\tfrac{\xi^2 - \eta^2}{2}\right) + o_\kappa(ar^{-2})
  = O_\kappa\!\left(\frac{|\mu_z/\mu_y|}{r^2}\right).
\]

\emph{Variance error.}
We bound $(\sigma^*_X)^2 = \mathbb{E}[X^2 \mid A_r] - (\mu^*_X)^2$ directly.
For the second moment, expand one order further:
$(1+t)^{-2} = 1 - 2t + 3t^2 - 4t^3 + R_4(t)$ with $|R_4(t)| \leq C|t|^4$
and $t = \eta U$.
Since $\sup_{B_r}|\eta U| = K_r/r_y = o(1)$, this expansion is uniform
on $A_r^+ \cap B_r$.
Using $\mathbb{E}[U] = 0$, $\mathbb{E}[U^2] = 1$, $\mathbb{E}[U^3] = 0$
(odd moments of $\mathcal{N}(0,1)$ vanish; conditioning on $A_r^+$ changes
these by $O(e^{-cr^2})$), and $\mathbb{E}[U^4] = 3$:
\[
  \mathbb{E}\!\left[\frac{1}{(1+\eta U)^2} \,\Big|\, A_r\right]
  = 1 + 3\eta^2 + O_\kappa(\eta^4)
  = 1 + 3\eta^2 + O_\kappa(r^{-4}),
\]
where the $-4\eta^3\mathbb{E}[U^3]$ term vanishes, the $R_4$ remainder
contributes $O(\eta^4\mathbb{E}[U^4\mathbf{1}_{B_r}]) = O(\eta^4) = O(r^{-4})$,
and the contributions of $B_r^c$ and $A_r^-$ are $O_\kappa(r^{-4})$ after
multiplying by their exponentially small probabilities.
Since $\mathbb{E}[X^2 \mid A_r] = a^2(1+\xi^2)\,\mathbb{E}[(1+\eta U)^{-2} \mid A_r]$:
\[
  \mathbb{E}[X^2 \mid A_r]
  = a^2(1+\xi^2)(1 + 3\eta^2 + O_\kappa(D/r^2)).
\]
Together with $(\mu^*_X)^2 = a^2(1+\eta^2)^2 + o_\kappa(a^2 r^{-2})
= a^2(1 + 2\eta^2 + O_\kappa(D/r^2))$:
\[
  (\sigma^*_X)^2
  = a^2(\eta^2 + \xi^2) + O_\kappa(a^2 D / r^2)
  = a^2 D\bigl(1 + O_\kappa(r^{-2})\bigr).
\]
The DMA variance $\hat\sigma^2_X = (e^D - 1)\hat\mu_X^2 = a^2 D e^{2D}(1+O(D))
= a^2 D(1 + O_\kappa(r^{-2}))$, so
\[
  \frac{|\hat\sigma^2_X - (\sigma^*_X)^2|}{(\sigma^*_X)^2}
  = O_\kappa\!\left(\frac{1}{r^2}\right).
\]
For~\eqref{eq:gaussian-kl-moment}, substitute into the exact Gaussian KL:
\[
  \mathrm{KL}(N^*_X \| \hat{m}_{f \to X})
  = \frac{(\hat\mu_X - \mu^*_X)^2}{2\hat\sigma^2_X}
  + \frac{(\sigma^*_X)^2}{2\hat\sigma^2_X}
  - \frac{1}{2}
  - \frac{1}{2}\log\frac{(\sigma^*_X)^2}{\hat\sigma^2_X}
  = O\!\left(\frac{1}{r^2}\right). 
\]
\end{proof}

\begin{lemma}[Non-Gaussianity of the truncated ratio]
\label{lem:nongauss-ratio}
Assume additionally that the ratio of the two signal-to-noise ratios is
bounded: $\kappa^{-1} \leq r_y/r_z \leq \kappa$ for some fixed $\kappa \geq 1$.
Then there exists a constant $C = C(\kappa)$ such that, for all sufficiently
large $r$ (depending only on $\kappa$),
\begin{equation}
  \mathrm{KL}\!\left(\tilde{m}^r_{f \to X} \,\Big\|\, N^*_X\right)
  \;\leq\; \frac{C}{r^2}.
  \label{eq:nongauss-ratio}
\end{equation}
In particular, $\mathrm{KL}(\tilde{m}^r_{f\to X}\|N^*_X) = O_\kappa(r^{-2})$
as $r \to \infty$.
\end{lemma}

\begin{proof}
By changing signs if necessary, assume $\mu_y > 0$ and $\mu_z > 0$.
Write $Y = \mu_y(1+\eta U)$ and $Z = \mu_z(1+\xi V)$ with
$U, V \stackrel{\mathrm{i.i.d.}}{\sim} \mathcal{N}(0,1)$, $\eta = 1/r_y$,
$\xi = 1/r_z$.
Set $a = \mu_z/\mu_y$, $D = \eta^2 + \xi^2$,
$\rho = -\eta/\sqrt{D}$, $\tau = \xi/\sqrt{D}$, so $\rho^2+\tau^2=1$
and $|\rho|, |\tau| \asymp_\kappa 1$.
Let $W = (X - \mu^*_X)/s_X$ be the standardised version of $X$ under $P_r$,
so $N^*_X$ is exactly $\mathcal{N}(0,1)$ in $W$-coordinates.

\emph{Step 1: Conditional Gaussianity.}
With $\mu_y > 0$, the truncation event is
$A_r = \{|Y| \geq \mu_y/2\} = A_r^+ \cup A_r^-$ where
$A_r^+ = \{1+\eta U \geq 1/2\}$ and $A_r^- = \{1+\eta U \leq -1/2\}$.
The negative branch satisfies
$\mathbb{P}(A_r^-) = \Phi(-3r_y/2) \leq e^{-9r_y^2/8}$, which is
exponentially small; its contribution to all subsequent expectations is
absorbed into the $O(e^{-cr^2})$ remainder and we work hereafter on $A_r^+$.
On $A_r^+$ we have $1+\eta u \geq 1/2 > 0$, so the denominator is
bounded away from zero, and for fixed $U = u$,
$X = a(1+\xi V)/(1+\eta u)$ is a linear function of the Gaussian $V$.
Hence $P_{W|U=u,A_r^+}$ is exactly $\mathcal{N}(m_u, v_u)$ with
\[
  m_u = \frac{a/(1+\eta u) - \mu^*_X}{s_X}, \qquad
  v_u = \frac{a^2\xi^2}{s_X^2(1+\eta u)^2}.
\]

\emph{Step 2: Per-slice KL bound.}
Let $B_r = \{|U| \leq K_r\}$ with $K_r = \sqrt{16\log r}$.
For $u \in B_r$, a Taylor expansion of $(1+\eta u)^{-1}$ together with
$\mu^*_X = a(1+\eta^2) + o_\kappa(ar^{-2})$
(Lemma~\ref{lem:product-moment-error}) and $s_X \asymp_\kappa a\sqrt{D}$ gives
\[
  m_u - \rho u = O_\kappa\!\left(\frac{1+u^2}{r}\right), \qquad
  v_u - \tau^2   = O_\kappa\!\left(\frac{1+|u|}{r}\right).
\]
Since $\tau$ is bounded away from zero, the Gaussian KL formula yields
\[
  \mathrm{KL}\!\left(\mathcal{N}(m_u,v_u)\;\Big\|\;\mathcal{N}(\rho u,\tau^2)\right)
  \leq \frac{C_\kappa(1+u^4)}{r^2}, \qquad u \in B_r.
\]
On $B_r^c \cap A_r$ the KL is at most $C_\kappa(1+r^2+u^2)$;
Gaussian tail bounds give $\mathbb{P}(|U|>K_r \mid A_r) \leq Cr^{-4}$,
so $\mathbb{E}[(1+r^2+U^2)\mathbf{1}_{B_r^c}\mid A_r] \leq C_\kappa r^{-2}$.

\emph{Step 3: Chain rule and data-processing.}
Define the reference joint
$Q_{W,U}(\mathrm{d}w,\mathrm{d}u) =
\phi(u)\,\mathcal{N}(w;\rho u,\tau^2)\,\mathrm{d}w\,\mathrm{d}u$.
The KL chain rule gives
\[
  \mathrm{KL}(P_{W,U}\|Q_{W,U})
  = \underbrace{\mathrm{KL}(P_{U|A_r}\|\mathcal{N}(0,1))}_{=-\log\mathbb{P}(A_r)=O(e^{-cr^2})}
   + \mathbb{E}_{U|A_r}\!\left[\mathrm{KL}(P_{W|U,A_r}\|\mathcal{N}(\rho U,\tau^2))\right]
  \;\leq\; \frac{C_\kappa}{r^2}.
\]
The marginal of $Q_{W,U}$ in $W$ is $\mathcal{N}(0,1)$ (since
$W_0 = \rho U + \tau V \sim \mathcal{N}(0,\rho^2+\tau^2) = \mathcal{N}(0,1)$).
The data-processing inequality applied to $(W,U)\mapsto W$ gives
\[
  \mathrm{KL}(P_r \| N^*_X)
  = \mathrm{KL}(P_W \| \mathcal{N}(0,1))
  \leq \mathrm{KL}(P_{W,U}\|Q_{W,U})
  \leq \frac{C_\kappa}{r^2}. \qedhere
\]
\end{proof}

\begin{proof}[Proof of Theorem~\ref{thm:product-bwd-bound}]
Let $N^*_X = \mathcal{N}(\mu^*_X, (\sigma^*_X)^2)$ be the moment-matched
Gaussian of $\tilde{m}^r_{f\to X}$.
Since the log-density ratio of any two Gaussians is a quadratic polynomial,
and $\tilde{m}^r_{f\to X}$ and $N^*_X$ share the same first two moments:
\[
  \mathbb{E}_{\tilde{m}^r}\!\left[\log\frac{N^*_X(X)}{\hat{m}_{f\to X}(X)}\right]
  = \mathbb{E}_{N^*_X}\!\left[\log\frac{N^*_X(X)}{\hat{m}_{f\to X}(X)}\right].
\]
This gives the exact Pythagorean decomposition
\begin{equation}
  \mathrm{KL}\!\left(\tilde{m}^r_{f \to X} \;\Big\|\; \hat{m}_{f \to X}\right)
  \;=\;
  \underbrace{\mathrm{KL}\!\left(\tilde{m}^r_{f \to X} \;\Big\|\; N^*_X\right)}_{\leq\,C_\kappa/r^2
    \;\text{(Lemma~\ref{lem:nongauss-ratio})}}
  \;+\;
  \underbrace{\mathrm{KL}\!\left(N^*_X \;\Big\|\; \hat{m}_{f \to X}\right)}_{O_\kappa(1/r^2)
    \;\text{(Lemma~\ref{lem:product-moment-error})}}.
  \label{eq:product-bwd-pythag}
\end{equation}
Both terms are $O_\kappa(1/r^2)$ as $r\to\infty$, giving~\eqref{eq:product-bwd-bound}.
\end{proof}

\newpage

\subsection{Proofs of Propositions~\ref{prop:relu-fwd}
  and~\ref{prop:relu-bwd} (ReLU Factor DMA)}
\label{app:proofs:relu}

\begin{remark}[Improper Backward Message at $\alpha = 0$]
\label{rem:relu-improper}
For the standard ReLU ($\alpha = 0$), every $x \leq 0$ maps to $y = 0$,
so the $x \leq 0$ piece of the backward message evaluates to the constant
$\mathcal{N}(0;\, \mu_y, \sigma_y^2) = \varphi(v)/\sigma_y$ rather than a
Gaussian in $x$.
This constant piece has infinite mass, making the backward message improper.
The formulas~\eqref{eq:relu-bwd-mean}--\eqref{eq:relu-bwd-var} are
inapplicable at $\alpha = 0$; in practice, one uses $\alpha = \epsilon \ll 1$
(leaky ReLU) or replaces the backward message with a uniform prior on the
negative half-line.
For $\mu_y \gg \sigma_y$ the constant is negligible and the backward message
is approximately $\mathcal{N}(\mu_y, \sigma_y^2)$, but no finite
normalisation correction applies in general.
\end{remark}

\paragraph{Truncated Gaussian lemma.}
Both proofs use the following standard integral identities.
For $X \sim \mathcal{N}(\mu, \sigma^2)$, $u = \mu/\sigma$, $P = \Phi(u)$,
$\phi = \varphi(u)$:
\begin{align}
  \int_{-\infty}^0 \mathcal{N}(x;\mu,\sigma^2)\intd{x}    &= 1-P, &
  \int_0^\infty \mathcal{N}(x;\mu,\sigma^2)\intd{x}       &= P, \\
  \int_{-\infty}^0 x\,\mathcal{N}(x;\mu,\sigma^2)\intd{x} &= \mu(1-P)-\sigma\phi, &
  \int_0^\infty x\,\mathcal{N}(x;\mu,\sigma^2)\intd{x}    &= \mu P+\sigma\phi, \\
  \int_{-\infty}^0 x^2\mathcal{N}(x;\mu,\sigma^2)\intd{x} &= (\mu^2+\sigma^2)(1-P)-\mu\sigma\phi, &
  \int_0^\infty x^2\mathcal{N}(x;\mu,\sigma^2)\intd{x}    &= (\mu^2+\sigma^2)P+\mu\sigma\phi.
    \label{eq:trunc-moments}
\end{align}
These follow from completing the square in the exponent and the Gaussian
survival function identity $\mathbb{E}[X\mathbf{1}_{X>0}] = \mu\Phi(\mu/\sigma)
+ \sigma\varphi(\mu/\sigma)$.

\begin{proof}[Proof of Proposition~\ref{prop:relu-fwd}]
The exact forward message is the pushforward of $X \sim \mathcal{N}(\mu_x,\sigma_x^2)$
through $\relu{\cdot}{\alpha}$.
We compute $\mathbb{E}[Y] = \mathbb{E}[\relu{X}{\alpha}]$ and
$\mathbb{E}[Y^2] = \mathbb{E}[\relu{X}{\alpha}^2]$ using~\eqref{eq:trunc-moments}
with $u = \mu_x/\sigma_x$, $P = \Phi(u)$, $\phi = \varphi(u)$.

For the first moment, writing $\relu{x}{\alpha} = x\mathbf{1}_{x>0} + \alpha x\mathbf{1}_{x\leq 0}$:
\begin{align*}
  \mathbb{E}[Y]
    &= \int_0^\infty x\mathcal{N}(x;\mu_x,\sigma_x^2)\intd{x}
       + \alpha\int_{-\infty}^0 x\mathcal{N}(x;\mu_x,\sigma_x^2)\intd{x} \\
    &= (\mu_x P + \sigma_x\phi) + \alpha(\mu_x(1-P) - \sigma_x\phi) \\
    &= \mu_x[\alpha+(1-\alpha)P] + (1-\alpha)\sigma_x\phi
     = \mu_x A + (1-\alpha)\sigma_x\phi. \checkmark
\end{align*}

For the second moment:
\begin{align*}
  \mathbb{E}[Y^2]
    &= \int_0^\infty x^2\mathcal{N}(x;\mu_x,\sigma_x^2)\intd{x}
       + \alpha^2\int_{-\infty}^0 x^2\mathcal{N}(x;\mu_x,\sigma_x^2)\intd{x} \\
    &= [(\mu_x^2+\sigma_x^2)P + \mu_x\sigma_x\phi]
       + \alpha^2[(\mu_x^2+\sigma_x^2)(1-P) - \mu_x\sigma_x\phi] \\
    &= (\mu_x^2+\sigma_x^2)[\alpha^2+(1-\alpha^2)P]
       + (1-\alpha^2)\mu_x\sigma_x\phi
     = (\mu_x^2+\sigma_x^2)B + (1-\alpha^2)\mu_x\sigma_x\phi. \checkmark
\end{align*}

The DMA is the moment-matching projection $\hat{m}_{f\to Y} = \mathcal{N}(m_Y, s_Y^2)$
with $m_Y = \mathbb{E}[Y]$ and $s_Y^2 = \mathbb{E}[Y^2] - m_Y^2$.

\emph{Consistency.} As $\sigma_x \to 0$ with $\mu_x \to \bar x$:
$u \to \pm\infty$, so for $\bar x > 0$: $P\to 1$, $\phi\to 0$, $A\to 1$,
$B\to 1$, giving $m_Y\to\bar x$ and $s_Y^2 \to 0$; for $\bar x < 0$: $P\to 0$,
$A\to\alpha$, $B\to\alpha^2$, giving $m_Y\to\alpha\bar x$ and $s_Y^2\to 0$.
Both recover $\relu{\bar x}{\alpha}$ exactly.
\end{proof}

\begin{proof}[Proof of Proposition~\ref{prop:relu-bwd}]
The exact backward message density is
\begin{equation}
  m_{f\to X}(x) =
    \begin{cases}
      \mathcal{N}(x;\,\mu_y,\sigma_y^2)        & x > 0, \\
      \mathcal{N}(\alpha x;\,\mu_y,\sigma_y^2) & x \leq 0,
    \end{cases}
  \label{eq:relu-bwd-exact}
\end{equation}
obtained by integrating the factor $\dirac{y - \relu{x}{\alpha}}$ against
$m_{Y\to f}(y) = \mathcal{N}(y;\mu_y,\sigma_y^2)$.
(For $x > 0$, the constraint $y = x$ gives $m_{f\to X}(x) = \mathcal{N}(x;\mu_y,\sigma_y^2)$;
for $x \leq 0$, the constraint $y = \alpha x$ gives
$m_{f\to X}(x) = \mathcal{N}(\alpha x;\mu_y,\sigma_y^2)$.)

\emph{Total mass.}
Setting $v = \mu_y/\sigma_y$, $P = \Phi(v)$, $Q = 1-P$:
\[
  \int_{-\infty}^\infty m_{f\to X}(x)\intd{x}
    = P + \frac{1}{\alpha}\int_{-\infty}^0 \mathcal{N}(t;\mu_y,\sigma_y^2)\intd{t}
    = P + \frac{Q}{\alpha}
    = \frac{\tilde C}{\alpha},
\]
where the substitution $t = \alpha x$ (for $\alpha>0$) contributes the
$1/\alpha$ factor and $\tilde C = \alpha P + Q$.

\emph{First moment.}
Using~\eqref{eq:trunc-moments} and the substitution $t = \alpha x$ for the
$x \leq 0$ piece:
\begin{align*}
  \int_{-\infty}^\infty x\,m_{f\to X}(x)\intd{x}
    &= \int_0^\infty x\mathcal{N}(x;\mu_y,\sigma_y^2)\intd{x}
       + \int_{-\infty}^0 x\mathcal{N}(\alpha x;\mu_y,\sigma_y^2)\intd{x} \\
    &= (\mu_y P + \sigma_y\phi)
       + \frac{1}{\alpha^2}(\mu_y Q - \sigma_y\phi).
\end{align*}
Dividing by $\tilde C/\alpha$:
\[
  m_X = \frac{\mu_y(\alpha^2 P+Q) + (\alpha^2-1)\sigma_y\phi}{\alpha\tilde C}.
  \checkmark
\]

\emph{Second moment.}
\begin{align*}
  \int_{-\infty}^\infty x^2\,m_{f\to X}(x)\intd{x}
    &= (\mu_y^2+\sigma_y^2)P + \mu_y\sigma_y\phi
       + \frac{1}{\alpha^3}[(\mu_y^2+\sigma_y^2)Q - \mu_y\sigma_y\phi].
\end{align*}
Dividing by $\tilde C/\alpha$:
\[
  \mathbb{E}[X^2]
    = \frac{(\mu_y^2+\sigma_y^2)(\alpha^3 P + Q)
            + (\alpha^3-1)\mu_y\sigma_y\phi}{\alpha^2\tilde C},
\]
giving $s_X^2 = \mathbb{E}[X^2] - m_X^2$ as stated.

\emph{Consistency.}
As $\sigma_y\to 0$ with $\mu_y\to\bar y$, $v\to\pm\infty$ and $\phi\to 0$.
For $\bar y > 0$: $P\to 1$, $Q\to 0$, $\tilde C\to\alpha$, giving
$m_X\to\bar y$ and $s_X^2\to 0$.
For $\bar y < 0$ (requires $\alpha > 0$): $P\to 0$, $Q\to 1$, $\tilde C\to 1$,
giving $m_X\to\bar y/\alpha$ and $s_X^2\to 0$.
Both recover the exact inverse $x = \relu{\bar y}{\alpha}^{-1}$.
\end{proof}

\subsection{Theorem Coverage by Factor and Direction}
\label{app:proofs:theorems}

Table~\ref{tab:bnn:theorems} summarises which theorem governs each factor
direction in the BNN factor graph.
Two tiers arise: \emph{Tier~1} (proper messages) covered by
Theorem~\ref{thm:dma-master}; \emph{Tier~2} (improper messages) covered by
the concentrated-input Theorem~\ref{thm:product-bwd-bound}.

\begin{table}[h]
  \centering
  \caption{Theorem coverage for each factor direction in the BNN.
    ``Proper?'' refers to whether the DMA message is a valid
    (positive-precision) Gaussian. All linear and copy factors satisfy
    Corollary~\ref{cor:dirac-exact} exactly; the Gaussian likelihood is
    conjugate and exact. See Remarks~\ref{rem:product-improper}
    and~\ref{rem:relu-improper} for the improper cases.}
  \label{tab:bnn:theorems}
  \begin{tabular}{llcc}
    \toprule
    Factor & Direction & Proper? & Theorem \\
    \midrule
    Gaussian prior             & both        & Yes & Cor.~\ref{cor:dirac-exact} (exact)           \\
    Gaussian likelihood        & both        & Yes & exact (conjugate)                            \\
    Linear / copy              & both        & Yes & Cor.~\ref{cor:dirac-exact}                   \\
    Product ($z{=}xy$)         & fwd ($z$)   & Yes & Thm.~\ref{thm:dma-master}                    \\
    Product ($z{=}xy$)         & bwd ($x,y$) & No  & Thm.~\ref{thm:product-bwd-bound}             \\
    Leaky-ReLU ($\alpha{>}0$)  & fwd         & Yes & Thm.~\ref{thm:dma-master}                    \\
    Leaky-ReLU ($\alpha{>}0$)  & bwd         & Yes & Thm.~\ref{thm:dma-master}                    \\
    Standard ReLU ($\alpha{=}0$) & fwd       & Yes & Thm.~\ref{thm:dma-master}                    \\
    Standard ReLU ($\alpha{=}0$) & bwd       & No  & Rem.~\ref{rem:relu-improper} (no bound derived$^{\dagger}$) \\
    \bottomrule
  \end{tabular}
  \\\smallskip
  \raggedright\footnotesize
  $^{\dagger}$Standard ReLU ($\alpha{=}0$) is not used in this paper;
  a concentrated-input bound analogous to Thm.~\ref{thm:product-bwd-bound}
  is left for future work.
\end{table}


\newpage

\section{Empirical Validation of DMA Approximation Quality}
\label{app:validation}

This appendix provides six layers of empirical validation:
\begin{enumerate}
  \item \textbf{Master theorem} (§\ref{app:validation:master}): we directly
    verify Theorem~\ref{thm:dma-master} by computing the message KL $\delta$
    and the resulting marginal KL for 192 leaky-ReLU factor configurations
    and confirming that the bound $\|\msgto{X_j}{f}\|_\infty/Z \cdot \delta$
    holds in every case.
  \item \textbf{Bound tightness} (§\ref{app:validation:tightness}): we sweep
    each parameter axis individually (input SNR $r$, leaky slope $\alpha$,
    input width $\sigma_x$) and overlay actual marginal KL with the theoretical
    bound, showing how the gap varies across the full parameter range.
  \item \textbf{Factor-level SNR sweep} (§\ref{app:validation:snr}): we
    verify the end-to-end $O(1/r^2)$ rate for the product backward message
    (Theorem~\ref{thm:product-bwd-bound}) and confirm the same empirical rate
    for the leaky-ReLU backward message via importance-sampling reference marginals
    over a range of input SNR values.
  \item \textbf{Copy factor} (§\ref{app:validation:copy}): we visualise the
    $O(1/r^2)$ Gaussian--log-normal KL bound from Lemma~\ref{lem:gauss-ln-kl}
    as a 2D heatmap over the parameter space, confirming both horizontal bands
    (left) and diagonal $r$-contours (right).
  \item \textbf{Product factor} (§\ref{app:validation:product}): we compare
    DMA marginals against IS references for nominal and stress configurations,
    illustrating where the Gaussian approximation holds and where it degrades.
  \item \textbf{ReLU factor} (§\ref{app:validation:relu}): we compare DMA
    forward and backward marginals against IS references for the leaky-ReLU
    factor, showing the approximation quality in both directions.
\end{enumerate}


\subsection{Master Theorem Validation}
\label{app:validation:master}

We validate Theorem~\ref{thm:dma-master} directly: for 192 leaky-ReLU
factor configurations spanning $\alpha \in \{0.1, 0.3, 0.5\}$,
incoming-message widths $\sigma_x \in \{0.5, 1.0, 1.5, 2.0\}$, and
factor SNR values $r = |\mu_y|/\sigma_y \in [0.5, 7]$, we compute both
\[
  \delta = \mathrm{KL}(m_{f \to X} \,\|\, \hat{m}_{f \to X})
\]
and the resulting marginal error
$\mathrm{KL}(p_X/Z \,\|\, \hat{p}_X/\hat{Z})$ analytically on a fine grid
(no Monte Carlo).
The true backward message $m_{f \to X}(x) = \mathcal{N}(\mathrm{relu}(x;\alpha);\, \mu_y, \sigma_y^2)$
is a proper two-piece Gaussian for all $\alpha > 0$, so it lies within the
scope of Theorem~\ref{thm:dma-master}.
Figure~\ref{fig:validation:master-theorem} shows the results.

\begin{figure}[h]
  \centering
  \includegraphics[width=\linewidth]{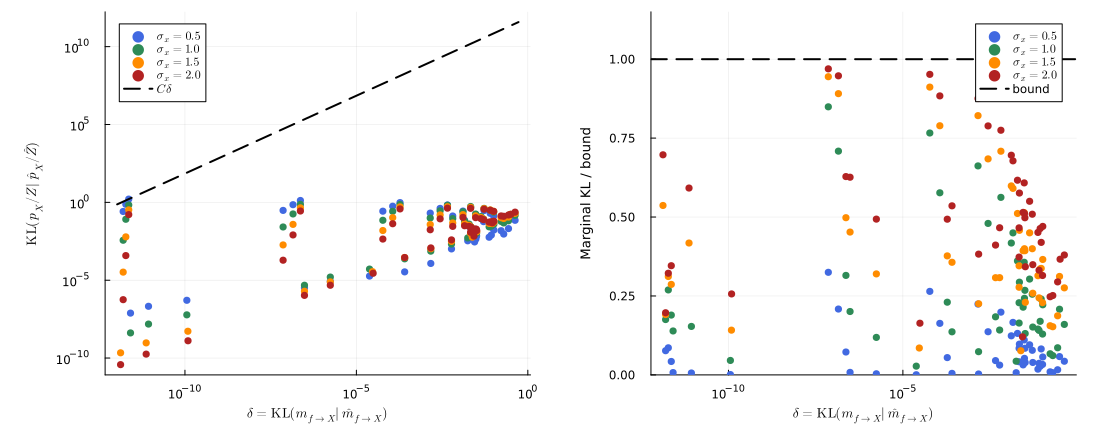}
  \caption{%
    Empirical validation of Theorem~\ref{thm:dma-master} across 192
    leaky-ReLU factor configurations.
    \emph{Left}: log-log scatter of message error $\delta$ versus marginal
    error; the dashed line $C\delta$ (with $C = \|\msgto{X_j}{f}\|_\infty/Z$)
    lies above every point, confirming $O(\delta)$ scaling.
    \emph{Right}: the marginal KL normalised by its per-configuration bound
    $\|\msgto{X_j}{f}\|_\infty/Z \cdot \delta$; all 192 ratios lie below the bound line
    at~1 (maximum ratio $0.97$), directly certifying the theorem.
    The bound is tight: the $\|\msgto{X_j}{f}\|_\infty/Z$ prefactor evaluates the
    Gaussian sup-norm, which is largest where the incoming message is most
    concentrated and the marginal error is therefore small anyway.
  }
  \label{fig:validation:master-theorem}
\end{figure}


\subsection{Bound Tightness: Per-Parameter Sweeps}
\label{app:validation:tightness}

Figure~\ref{fig:validation:tightness} isolates each parameter axis in turn,
holding the others fixed, and overlays the actual marginal KL (solid) with
the theoretical bound (dashed) to show directly how the gap varies.
Three qualitatively different behaviors emerge.
\emph{Panel (a) — SNR $r$}: the bound decreases monotonically, while the
actual marginal KL is non-monotone: small at low $r$ (where the DMA
moment-match of a near-symmetric message is accurate), peaking at moderate
$r \approx 2$--$3$ (where the asymmetric kink of the leaky ReLU is hardest
to capture), then declining again as both messages concentrate at high $r$.
\emph{Panel (b) — leaky slope $\alpha$}: both quantities decrease monotonically
as $\alpha \to 1$ (the factor approaches a linear copy, which is conjugate
and has $\delta = 0$) and grow as $\alpha \to 0$ (approaching the improper
hard-ReLU limit).
\emph{Panel (c) — incoming width $\sigma_x$}: widening the incoming message at
fixed $r$ lets more backward-message approximation error propagate into the
marginal, so the actual KL increases; the bound's prefactor
$\|\msgto{X_j}{f}\|_\infty/Z \propto 1/\sigma_x$ simultaneously decreases (the Gaussian
sup-norm decays as $1/\sigma_x$ while $Z$ grows more slowly).
The bound is therefore most conservative at small $\sigma_x$, where the
narrow incoming message masks the backward approximation error in the marginal.

\begin{figure}[h]
  \centering
  \includegraphics[width=\linewidth]{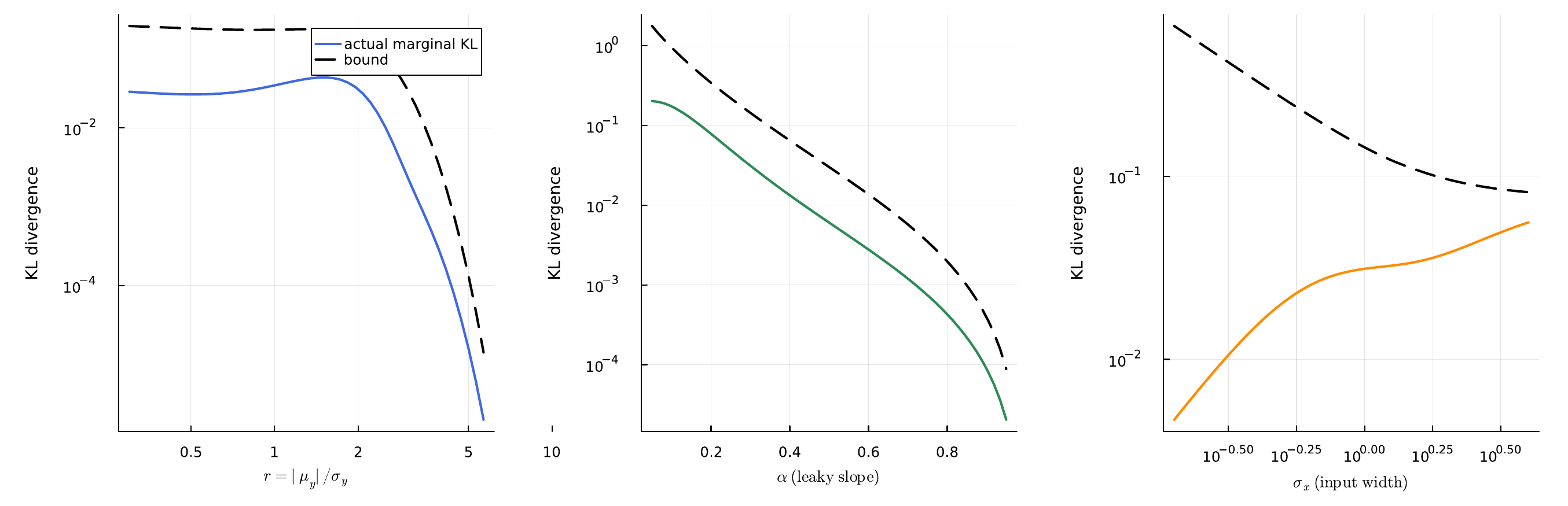}
  \caption{%
    Actual marginal KL (solid, coloured) versus the theoretical bound
    $\|\msgto{X_j}{f}\|_\infty/Z\cdot\delta$ (dashed, black)
    as each parameter is swept individually (legend shown in panel~(a)).
    \emph{(a)}~Input SNR $r$ ($\alpha=0.3$, $\sigma_x=1$ fixed): the bound
    decays monotonically; the actual KL is non-monotone, peaking near
    $r\approx 2$--$3$ where the ReLU kink is hardest to match, and small
    at both low and high $r$.
    \emph{(b)}~Leaky slope $\alpha$ ($r=2$, $\sigma_x=1$ fixed):
    both decrease as $\alpha\to 1$ (linear factor, $\delta=0$) and grow
    as $\alpha\to 0$ (approaching the improper hard-ReLU limit).
    \emph{(c)}~Input width $\sigma_x$ ($r=2$, $\alpha=0.3$ fixed):
    the actual KL increases with $\sigma_x$ (wider incoming message
    exposes more backward-message error in the marginal), while the bound
    decreases through its $1/\sigma_x$ prefactor; the conservatism is
    greatest at small $\sigma_x$.
    In all panels the bound lies strictly above the actual KL, with a
    maximum ratio of $0.68$ across all three sweeps (width sweep, panel~(c)).
  }
  \label{fig:validation:tightness}
\end{figure}


\subsection{SNR Sweep: Factor-Level Approximation Quality}
\label{app:validation:snr}

\paragraph{Setup.}
For each factor and each SNR value $r$, we draw $N = 500{,}000$ samples
from the incoming Gaussian messages, propagate them through the factor's
deterministic relation, and weight them by the remaining incoming message
to obtain an IS estimate of the true marginal.
We then compute $\mathrm{KL}(p_{\mathrm{IS}} \| q_{\mathrm{DMA}})$ via a
normalized weighted histogram with 500 bins covering the 0.1\%--99.9\%
quantile range of the IS distribution.
The dashed reference line $C/r^2$ is fitted to the high-SNR region
($r \geq 4$) by taking the median of $\mathrm{KL}(r) \cdot r^2$.
Figure~\ref{fig:validation:snr-sweep} shows the results for both factors.

\begin{figure}[h]
  \centering
  \includegraphics[width=\linewidth]{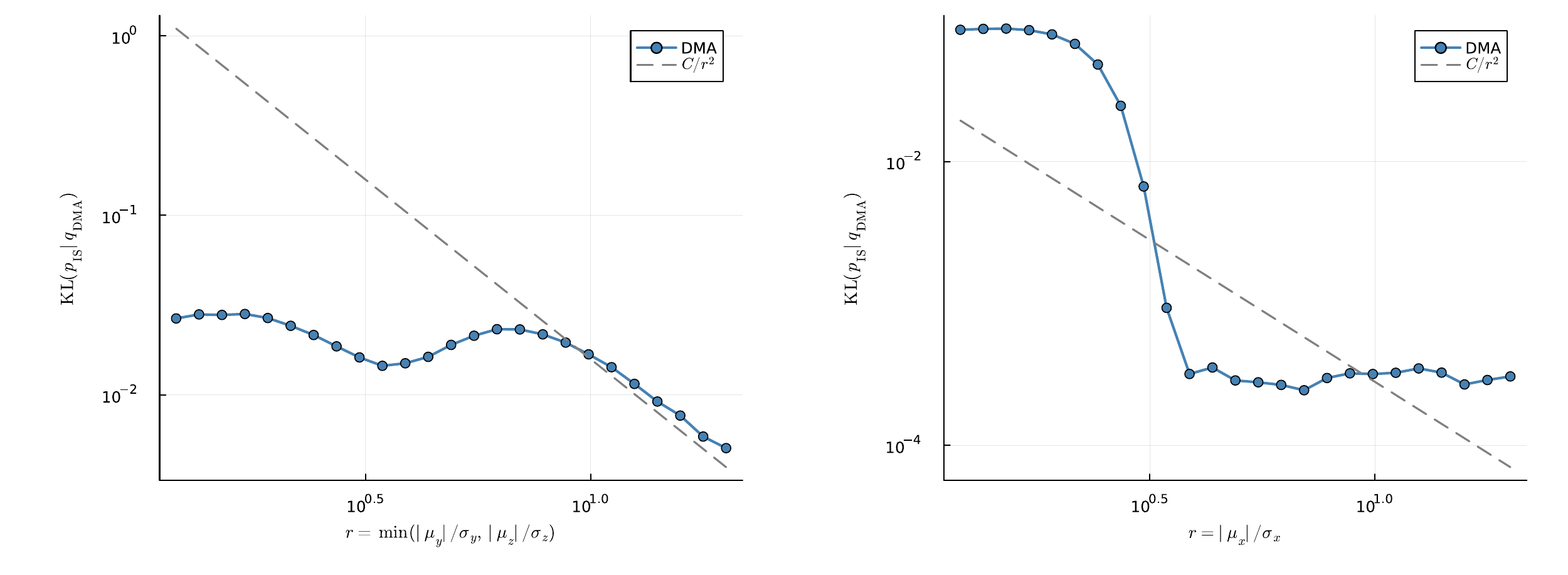}
  \caption{%
    KL divergence between IS reference marginal and DMA marginal as a
    function of input SNR $r$.
    \emph{Left}: product factor backward message to $X$ from $Z = XY$;
    incoming messages $Y \sim \mathcal{N}(\mu_y, (\mu_y/r)^2)$ and
    $Z \sim \mathcal{N}(\mu_z, (\mu_z/r)^2)$ with $\mu_y = 4$,
    $\mu_z = 10$, and the incoming message from $X$ fixed at
    $\mathcal{N}(3, 1)$.
    \emph{Right}: ReLU backward message to $X$ from
    $Y = \mathrm{relu}(X;\,\alpha)$ with $\alpha = 0.1$;
    incoming message $X \sim \mathcal{N}(\mu_x, (\mu_x/r)^2)$ with
    $\mu_x = 1$.
    Both panels show empirical $O(1/r^2)$ decay.
    The product backward (left) is covered by Theorem~\ref{thm:product-bwd-bound};
    the leaky-ReLU backward (right, $\alpha{=}0.1$, proper message) is consistent
    with Theorem~\ref{thm:dma-master} applied to the message KL.
    }
  \label{fig:validation:snr-sweep}
\end{figure}

\subsection{Copy Factor: Gaussian--Log-Normal Approximation}
\label{app:validation:copy}

The product backward message (Proposition~\ref{prop:product-bwd}) is
derived via two \emph{copy-factor} steps that bridge the Gaussian and
log-normal families.
\emph{Step 1} (backward): the exact backward message
$\mathrm{LN}(\cdot;\,\mu_Y, \sigma_Y^2)$ (log-space parameterised) is
approximated by a moment-matched Gaussian $\hat{m}_{f\to X}$.
\emph{Step 2} (forward): the Gaussian message
$\mathcal{N}(\cdot;\,\mu_X, \sigma_X^2)$ is approximated by a
moment-matched log-normal $\mathrm{LN}_\mathrm{m}$.
Both introduce $O(1/r^2)$ KL error (Lemma~\ref{lem:gauss-ln-kl});
Figure~\ref{fig:validation:copy-factor} confirms this empirically.

\begin{figure}[h]   
  \centering
  \includegraphics[width=\linewidth]{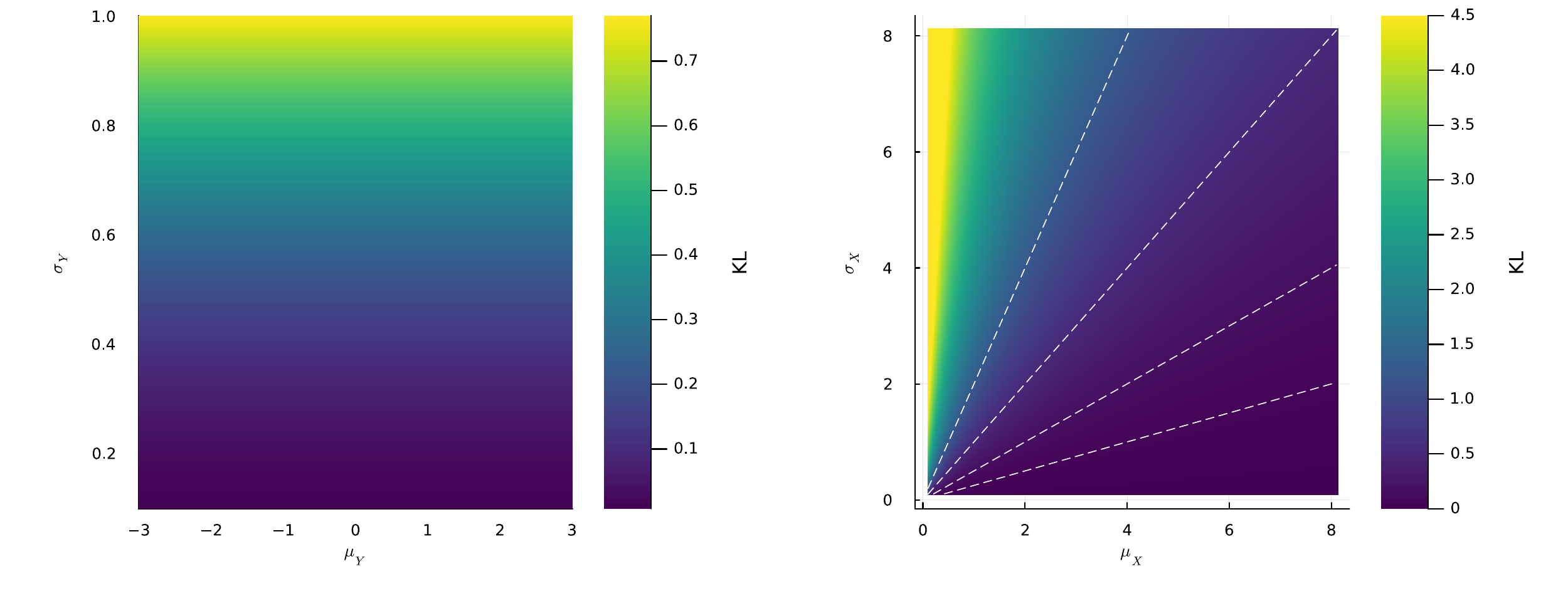}
  \caption{%
    Copy-factor KL divergence over the 2D parameter space, computed from
    the closed-form expression in Lemma~\ref{lem:gauss-ln-kl} (no Monte Carlo).
    \emph{Left}: Step~1 error
    $\mathrm{KL}[\mathrm{LN}(\cdot;\,\mu_Y,\sigma_Y^2),\;\hat{m}_{f\to X}(\cdot)]$
    swept over $(\mu_Y,\sigma_Y)$.
    Horizontal bands confirm that the KL depends only on the log-space
    variance $\sigma_Y^2$, not on $\mu_Y$; the analytical rate is
    $\frac{3}{4}\sigma_Y^2 + O(\sigma_Y^4)$.
    \emph{Right}: Step~2 error
    $\mathrm{KL}[\mathrm{LN}_\mathrm{m}(\cdot;\,\mu_X,\sigma_X^2),\;
    \mathcal{N}(\cdot;\,\mu_X,\sigma_X^2)]$
    swept over $(\mu_X,\sigma_X)$.
    Dashed white lines are level curves of $r = \mu_X/\sigma_X$;
    diagonal banding confirms the KL depends only on $r$, consistent
    with the $O(1/r^2)$ rate.
  }
  \label{fig:validation:copy-factor}
\end{figure}

\subsection{Product Factor}
\label{app:validation:product}

The left panel of Figure~\ref{fig:validation:snr-sweep} shows the product
backward message (to $X$ given $Z = XY$).
This is the case covered by Theorem~\ref{thm:product-bwd-bound}: the
exact message is improper (non-normalisable), so the IS reference is the
true marginal of $X$ rather than the message itself.
The theorem predicts $O(1/r^2)$ total KL error; the sweep confirms this
rate empirically across two decades of $r$.
At $r = 2$ (the theorem's stated threshold for properness of the truncated
reference), the KL is already small; at $r = 10$ it is several orders of
magnitude smaller, tracking the dashed $C/r^2$ reference closely.
Figure~\ref{fig:validation_product} compares DMA and IS marginals at a nominal
and a stress configuration; the forward marginal ($Z$) is well approximated in
both cases, while the backward marginals ($X$, $Y$) become increasingly non-Gaussian
under stress, illustrating the regime where the $O(1/r^2)$ bound is most relevant.

\begin{figure}[h]
  \centering
  \includegraphics[width=0.99\linewidth]{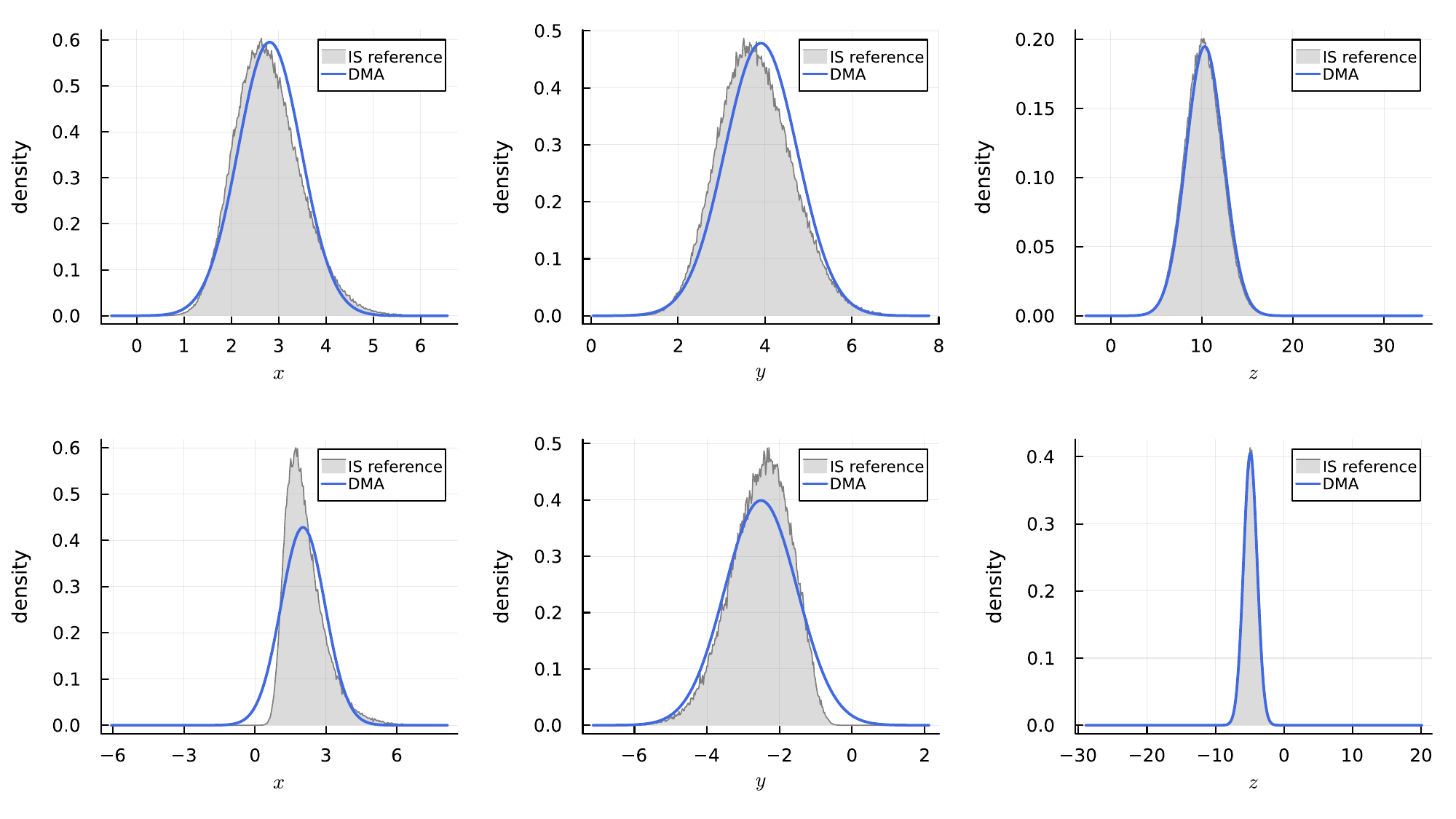}
  \caption{%
Product factor $f(X, Y, Z) = \delta(Z - XY)$. Each row shows marginals $X$, $Y$, $Z$
(left to right). Blue: DMA. Gray: IS. 
(a) Nominal ($\mu_X$ = 3, $\sigma^2_X$ = 1, $\mu_Y$ = 4, $\sigma^2_Y$ = 1, $\mu_Z$ = 10, $\sigma^2_Z$ = 5):
the $Z$ forward marginal is well approximated by the DMA Gaussian;
the $X$ and $Y$ backward marginals are already noticeably skewed in the IS reference.
(b) Stress ($\mu_X$ = 1, $\sigma^2_X$ = 4, $\mu_Y$ = -2.5, $\sigma^2_Y$ = 1,
$\mu_Z$ = -5, $\sigma^2_Z$ = 1): $Z$ remains well approximated; the $X$ and $Y$ backward
marginals become severely non-Gaussian.
}
  \label{fig:validation_product}
\end{figure}

\subsection{ReLU Factor}
\label{app:validation:relu}

The right panel of Figure~\ref{fig:validation:snr-sweep} shows the leaky
ReLU backward message ($\alpha = 0.1$).
Unlike the product factor, the leaky ReLU backward message is always
proper (the normalisation correction via the Mills ratio is finite for all
$\alpha > 0$), so this case is covered directly by the master theorem.
The empirical $O(1/r^2)$ decay is consistent with the general O($\delta$)
bound, and the absolute KL values are lower than the product factor at
matched $r$, reflecting the smoother shape of the leaky ReLU factor
compared to the ratio $X = Z/Y$.
The standard ReLU ($\alpha = 0$) backward message is improper; that case
is analogous to the product factor and is not shown.
Figure~\ref{fig:validation_relu} shows DMA versus IS at a nominal input:
the backward marginal IS reference is non-Gaussian (sharp, asymmetric), and the
forward marginal IS has a point mass at $y = 0$ from the leaky branch; the DMA
Gaussian captures the bulk location but cannot represent these non-Gaussian features.

\begin{figure}[h]
  \centering
  \includegraphics[width=0.75\linewidth]{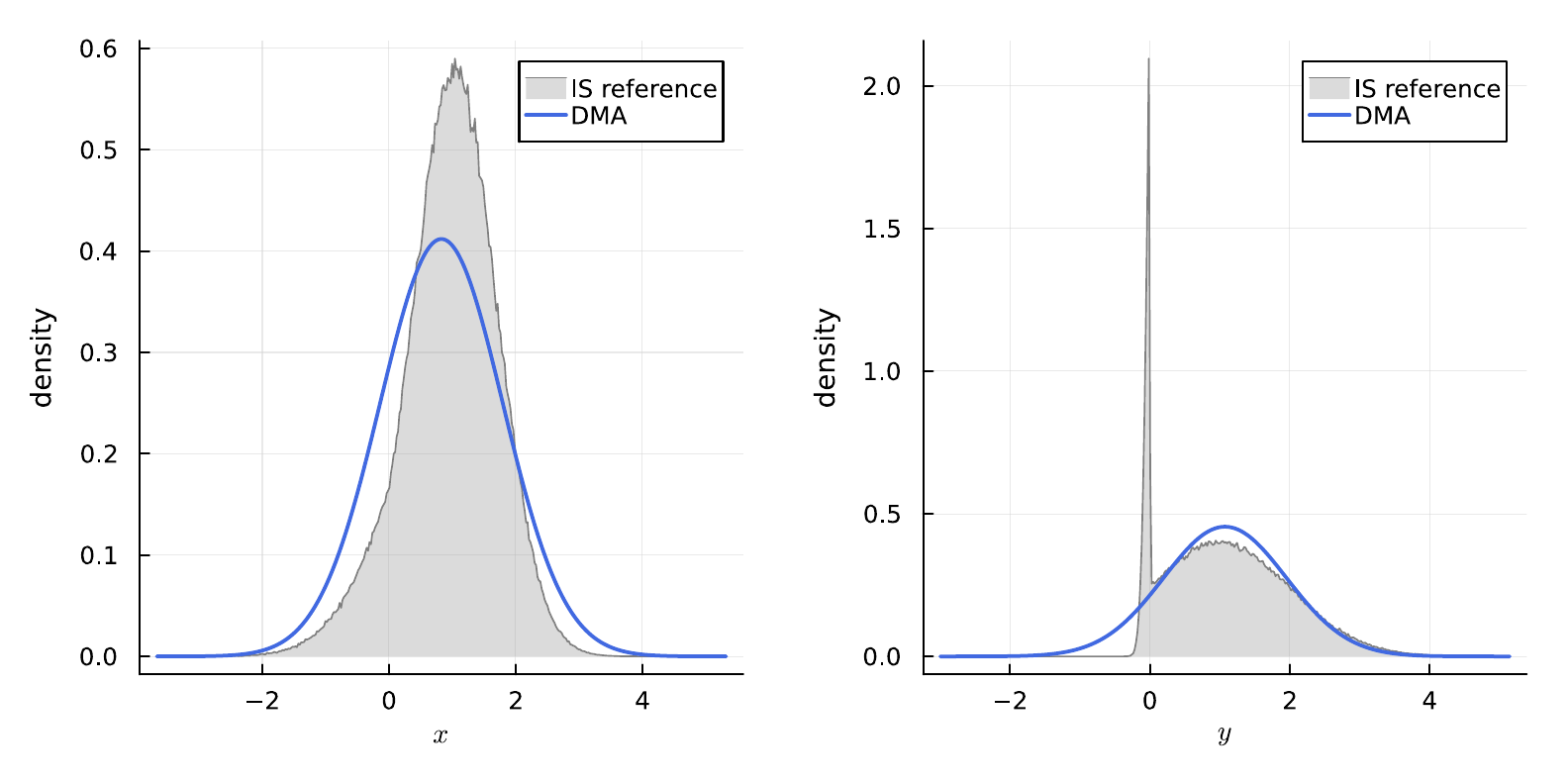}
  \caption{%
ReLU / Leaky-ReLU factor ($\alpha$ = 0.1), nominal scenario ($\mu_X$ = 1, $\sigma^2_X$ = 1).
Left: backward marginal $X$. Right: forward marginal $Y$. Blue: DMA, Gray: IS reference}
  \label{fig:validation_relu}
\end{figure}


\newpage

\section{Inference Algorithm}
\label{app:alg}

Algorithm~\ref{alg:bnn} gives the complete DMA BNN
training procedure.
We write $(\mu^{(l)}_{ij}, \sigma^{2,(l)}_{ij})$ for the mean and variance of
$q_{W^{(l)}_{ij}}$.

\paragraph{Scalar decomposition.}
The matrix-vector product $z^{(l)}_i = \sum_{j=1}^{d_{l-1}} W^{(l)}_{ij}\,x^{(l-1)}_j$
decomposes into $d_{l-1}$ product factors $z^{(l)}_{ij} = W^{(l)}_{ij}\,x^{(l-1)}_j$
and one linear sum factor per output unit $i$.
Forward messages are summed exactly under independence (means and variances add).
The backward message to $z^{(l)}_{ij}$ subtracts the total forward mean
$\mu^{(l)}_{z_i}$ and re-adds $\mu^{(l)}_{z_{ij}}$; the backward variance
adds $\sigma^{2,(l)}_{z_i}$ and subtracts $\sigma^{2,(l)}_{z_{ij}}$
(Gaussian deconvolution: exact for linear factors).
The backward message to $x^{(l-1)}_j$ is accumulated in precision form
$(\tau, \rho)$ over all output units $i$ before conversion to moments.

\paragraph{Mini-batch EP structure.}
The $N$ training examples are split into $B$ mini-batches of size $N/B$.
Each mini-batch $b$ has one stored outgoing weight message $m^{(b)}_{l,ij}$,
representing the combined likelihood contribution of that batch.
Before processing batch $b$, its previous message is divided out of the current
weight marginal to form the \emph{incoming belief} $q^{-b}_{W^{(l)}_{ij}}$
(the product of the prior and all other batches' messages).
After processing all examples in the batch, the new outgoing message is extracted
as $m^{(b)}_{\mathrm{new}} = q_W / q^{-b}_W$.
Setting $B=N$ (one example per batch) recovers per-example EP.

\paragraph{Prior, likelihood, and convergence.}
Prior factors $\mathcal{N}(\mu^{(l)}_{ij},\sigma_0^2)$ with He-style random means
$\mu^{(l)}_{ij}\sim\mathcal{N}(0,1/d_{l-1})$, the Gaussian likelihood
$\mathcal{N}(y_n;\,z^{(L)},\beta^2)$, and the optional activation prior factors
$\mathcal{N}(0,\sigma^2_{\mathrm{act},l})$ on each pre-activation $z^{(l)}_i$
all contribute exact messages ($\delta=0$); in the paper experiments all
$\sigma^2_{\mathrm{act},l}=\infty$ (uniform, no prior).
The observed input $x_n$ enters as a point mass ($\sigma^2=0$).
Convergence is checked after each full epoch via the normalised average
log-likelihood $\mathcal{L}_e = \frac{1}{N}\sum_n \log p(y_n\mid x_n)$:
training stops when $|\mathcal{L}_e - \mathcal{L}_{e-1}| / \max(\epsilon_{\mathrm{floor}},\,|\mathcal{L}_{e-1}|) < \varepsilon$,
where $\epsilon_{\mathrm{floor}}$ is a small numerical constant.
Per-example cost: $O\!\left(\sum_l d_l\,d_{l-1}\right)$.

\medskip

\begin{remark}
On the first pass every $m^{(b)}$ is uniform (zero precision), so the initial
division leaves the marginal unchanged.
After convergence the weight marginal satisfies
$q_{W^{(l)}_{ij}} \propto \mathcal{N}(\mu^{(l)}_{ij},\sigma_0^2)\cdot\prod_{b=1}^B m^{(b)}_{l,ij}$:
the prior message times one stored factor message per mini-batch.
The forward pass uses the full current belief $q_{W^{(l)}_{ij}}$ (not a cavity),
so within a batch each example conditions on the beliefs accumulated from the
prior and all other batches without requiring a per-example cavity state.
The convergence threshold $\varepsilon$ controls only the outer epoch loop; it
does not appear in any per-factor message computation and is not a learning rate.
\end{remark}

\newpage

\begin{algorithm}[h]
\SetAlgoLined
\DontPrintSemicolon
\caption{DMA BNN Training (1/2: initialisation and forward sweep)}
\label{alg:bnn}
\KwIn{%
  Data $\{(x_n,y_n)\}_{n=1}^N$; widths $d_0,\ldots,d_L$;
  leaky-ReLU slopes $\alpha_1,\ldots,\alpha_{L-1}$; mini-batches $B$;
  $\sigma_0^2$ (prior var.); $\beta^2$ (noise var.);
  $\sigma^2_{\mathrm{act},l}$, $l\!=\!1,\ldots,L$ (act.\ prior var.; $\infty$ = no prior);
  $\varepsilon$ (conv.\ threshold); $E$ (max epochs).}
\KwOut{Weight beliefs $\{q_{W^{(l)}_{ij}}\}_{l,i,j}$.}
\tcp{$m^{(b)}_{l,ij}$: stored outgoing msg from mini-batch $b$ to $W^{(l)}_{ij}$; all initialised to uniform}
$\mu^{(l)}_{ij} \sim \mathcal{N}(0,1/d_{l-1})$;\;
$q_{W^{(l)}_{ij}} \leftarrow \mathcal{N}(\mu^{(l)}_{ij},\sigma_0^2)$\hfill
  for all $l,\;i\leq d_l,\;j\leq d_{l-1}$\;
$m^{(b)}_{l,ij} \leftarrow \text{uniform}$\hfill
  for all $b\leq B,\;l,\;i\leq d_l,\;j\leq d_{l-1}$\;
\For{$e=1,\ldots,E$}{
  \For{$b=1,\ldots,B$ {\normalfont(mini-batch $\mathcal{B}_b$, $|\mathcal{B}_b|=N/B$)}}{
    $q^{-b}_{W^{(l)}_{ij}} \leftarrow q_{W^{(l)}_{ij}} \div m^{(b)}_{l,ij}$;\;\;
    $q_{W^{(l)}_{ij}} \leftarrow q^{-b}_{W^{(l)}_{ij}}$\hfill
      for all $l,\;i\leq d_l,\;j\leq d_{l-1}$\quad\tcp*[f]{divide out old batch msg}\;
    \For{each $n\in\mathcal{B}_b$}{
      $\mu^{(0)}_j \leftarrow x_{n,j}$,\;
      $\sigma^{2,(0)}_j \leftarrow 0$
        \hfill for $j=1,\ldots,d_0$\quad\tcp*[f]{point-mass input, $\delta=0$}\;
      \For{$l=1$ \KwTo $L$}{
        \For{$i=1,\ldots,d_l$}{
          \For{$j=1,\ldots,d_{l-1}$}{
            $\mu^{(l)}_{z_{ij}} \leftarrow \mu^{(l)}_{ij}\,\mu^{(l-1)}_j$
              \hfill\tcp*[f]{product factor (Prop.~\ref{prop:product-fwd})}\;
            $\sigma^{2,(l)}_{z_{ij}} \leftarrow
              \sigma^{2,(l)}_{ij}(\mu^{(l-1)}_j)^2
              +\sigma^{2,(l-1)}_j(\mu^{(l)}_{ij})^2
              +\sigma^{2,(l)}_{ij}\,\sigma^{2,(l-1)}_j$\;
          }
          $\mu^{(l)}_{z_i} \leftarrow \textstyle\sum_j\mu^{(l)}_{z_{ij}}$;\;\;
          $\sigma^{2,(l)}_{z_i} \leftarrow \textstyle\sum_j\sigma^{2,(l)}_{z_{ij}}$
            \hfill\tcp*[f]{sum factor}\;
          \tcp{Act.\ prior $\mathcal{N}(0,\sigma^2_{\mathrm{act},l})$: zero mean $\Rightarrow$ $\tau^{(l)}_{z_i}$ unchanged; $\sigma^2_{\mathrm{act},l}=\infty$ = no-op}
          $\tau^{(l)}_{z_i} \leftarrow \mu^{(l)}_{z_i}/\sigma^{2,(l)}_{z_i}$;\;
          $\rho^{(l)}_{z_i} \leftarrow 1/\sigma^{2,(l)}_{z_i} + 1/\sigma^2_{\mathrm{act},l}$ \qquad 
          $\mu^{(l)}_{z_i} \leftarrow \tau^{(l)}_{z_i}/\rho^{(l)}_{z_i}$;\;
          $\sigma^{2,(l)}_{z_i} \leftarrow 1/\rho^{(l)}_{z_i}$\;
          \lIf{$l<L$}{%
            $(\mu^{(l)}_i,\,\sigma^{2,(l)}_i) \leftarrow
              \mathrm{ReLUFwd}(\mu^{(l)}_{z_i},\sigma^{2,(l)}_{z_i},\alpha_l)$
            \hfill\tcp*[f]{ReLU factor (Prop.~\ref{prop:relu-fwd})}%
          }
        }
      }
      $(\mu^{\mathrm{bwd},(L)}_i,\;\sigma^{2,\mathrm{bwd},(L)}_i)
        \leftarrow (y_{n,i},\;\beta^2)$\hfill for $i=1,\ldots,d_L$
        \quad\tcp*[f]{likelihood backward to $z^{(L)}$, $\delta=0$}\;
    }
    $m^{(b)}_{l,ij} \leftarrow q_{W^{(l)}_{ij}} \div q^{-b}_{W^{(l)}_{ij}}$\hfill
      for all $l,\;i\leq d_l,\;j\leq d_{l-1}$\quad\tcp*[f]{extract new batch msg}\;
  }
}
\end{algorithm}

\newpage

\addtocounter{algocf}{-1}
\begin{algorithm}[h]
\SetAlgoLined
\DontPrintSemicolon
\caption{DMA BNN Training (2/2: backward sweep and convergence check)}
\label{alg:bnn2}
\tcp{(per example $n\in\mathcal{B}_b$, mini-batch $b$, epoch $e$:)}
\BlankLine
\tcp{--- Backward sweep: directly update $q_W$ (no per-example cavity) ---}
\For{$l=L$ \KwTo $1$}{
  $\tau^x_j\leftarrow 0$,\;\;$\rho^x_j\leftarrow 0$
    \hfill for $j=1,\ldots,d_{l-1}$\;
  \For{$i=1,\ldots,d_l$}{
    \For{$j=1,\ldots,d_{l-1}$}{
      \tcp{Sum-factor backward to $z^{(l)}_{ij}$: Gaussian deconvolution (exact)}
      $\mu^{\mathrm{cav}}_{ij} \leftarrow
        \mu^{\mathrm{bwd},(l)}_i - \mu^{(l)}_{z_i} + \mu^{(l)}_{z_{ij}}$\;
      $\sigma^{2,\mathrm{cav}}_{ij} \leftarrow
        \sigma^{2,\mathrm{bwd},(l)}_i + \sigma^{2,(l)}_{z_i} - \sigma^{2,(l)}_{z_{ij}}$\;
      \tcp{Product backward to $W^{(l)}_{ij}$ (Prop.~\ref{prop:product-bwd}): compute and apply}
      $m^{\mathrm{new}}_{l,ij}
        \leftarrow \mathrm{ProdBwd}_{W}(
          \mu^{\mathrm{cav}}_{ij},\,\sigma^{2,\mathrm{cav}}_{ij},\;
          \mu^{(l-1)}_j,\,\sigma^{2,(l-1)}_j)$\;
      $q_{W^{(l)}_{ij}} \leftarrow q_{W^{(l)}_{ij}} \times m^{\mathrm{new}}_{l,ij}$\;
      \tcp{Product backward to $x^{(l-1)}_j$ (Prop.~\ref{prop:product-bwd}, $W\!\leftrightarrow\!x$)}
      $(\Delta\tau_j,\,\Delta\rho_j)
        \leftarrow \mathrm{ProdBwd}_{X}(
          \mu^{\mathrm{cav}}_{ij},\,\sigma^{2,\mathrm{cav}}_{ij},\;
          \mu^{(l)}_{ij},\,\sigma^{2,(l)}_{ij})$\;
      $\tau^x_j \mathrel{+}= \Delta\tau_j$;\;\;
      $\rho^x_j \mathrel{+}= \Delta\rho_j$
        \hfill\tcp*[f]{accumulate backward precision to $x^{(l-1)}_j$ over $i$}\;
    }
  }
  \tcp{Convert accumulated backward msgs at $x^{(l-1)}$ to moment form}
  $(\mu^{\mathrm{bwd},x,(l-1)}_j,\;\sigma^{2,\mathrm{bwd},x,(l-1)}_j)
    \leftarrow (\tau^x_j/\rho^x_j,\;1/\rho^x_j)$
    \hfill for $j=1,\ldots,d_{l-1}$\;
  \lIf{$l>1$}{%
    $(\mu^{\mathrm{bwd},(l-1)}_i,\,\sigma^{2,\mathrm{bwd},(l-1)}_i)
      \leftarrow \mathrm{ReLUBwd}(
        \mu^{\mathrm{bwd},x,(l-1)}_i,\,\sigma^{2,\mathrm{bwd},x,(l-1)}_i,\;
        \mu^{(l-1)}_{z_i},\,\sigma^{2,(l-1)}_{z_i},\,\alpha_{l-1})$
    for $i\!\leq\!d_{l-1}$
    \hfill\tcp*[f]{ReLU backward to $z^{(l-1)}$ (Prop.~\ref{prop:relu-bwd})}%
  }
}
\BlankLine
\tcp{--- Convergence check (once per full epoch, after all mini-batches complete) ---}
$\mathcal{L}_e \leftarrow \frac{1}{N}\sum_{n=1}^{N}\log p(y_n\mid x_n)$
  \hfill\tcp*[f]{normalised average log-likelihood under current beliefs}\;
\lIf{$|\mathcal{L}_e - \mathcal{L}_{e-1}| \;/\; \max(\epsilon_{\mathrm{floor}},\,|\mathcal{L}_{e-1}|) < \varepsilon$}{\Return $\{q_{W^{(l)}_{ij}}\}$}
\BlankLine
\tcp{(repeat for next epoch $e+1$; after $E$ epochs:)}
\Return $\{q_{W^{(l)}_{ij}}\}$\;
\end{algorithm}



\section{Experimental Details and Extended Results}
\label{app:experiments}

\subsection{1D Regression: Feature Map and Architecture}
\label{app:experiments:setup}

The fixed feature map prepended to the learnable network is
$\varphi\colon\mathbb{R}\to\mathbb{R}^{7}$:
\[
  \varphi(x) = \bigl[x,\;
    e^{-(x+2)^2},\; e^{-(x+1)^2},\; e^{-x^2},\; e^{-(x-1)^2},\; e^{-(x-2)^2},\;
    \sin(x)\bigr]^\top.
\]
Output channels are standardised to zero mean and unit variance over
$[-5,\,5]$ and a constant bias is appended, giving an 8-dimensional input to
the learnable layers.
The learnable network has two hidden layers of widths $d_1 = 6$ and $d_2 = 5$
with leaky-ReLU activations ($\alpha_1 = 0.4$, $\alpha_2 = 0.8$) and a
scalar output layer; the observation model is Gaussian with $\beta = 0.2$.
Training runs for up to 200 epochs with 10 mini-batches of 20 examples each,
stopping when the relative change in normalised log-likelihood falls below $0.1$.

Weight initialisation uses He-style fan-in scaling throughout.
For each weight $W^{(l)}_{ij}$ the prior mean is drawn independently from
$\mathcal{N}(0,\,1/d_{l-1})$, giving a per-weight prior
$q_{W^{(l)}_{ij}} = \mathcal{N}(\mu^{(l)}_{ij},\,\sigma_0^2)$
with $\mu^{(l)}_{ij}\sim\mathcal{N}(0,\,1/d_{l-1})$ and
$\sigma_0^2 = 1/\bigl((L-1)\,d_{l-1}\bigr)$,
where $L=4$ counts all layers including the fixed feature map.
The marginal prior on each weight (integrating out the random mean) is
$\mathcal{N}(0,\,L/\bigl((L-1)\,d_{l-1}\bigr))$,
which for $L=4$ is $\mathcal{N}(0,\,4/(3\,d_{l-1}))$.
For the three learnable layers this gives marginal standard deviations of
$0.47$, $0.47$, $0.52$ (fan-in 8, 6, 5 respectively).
The data-generating weights are drawn from the same per-layer prior
(random mean drawn first, weight sampled from that Gaussian),
so the model is correctly specified.

\subsection{Comparison with Adam}
\label{app:experiments:adam}

\paragraph{Setup.}
We compare DMA against Adam \citep{KingmaBa2015} on the same task ($N=200$,
$\beta=0.2$, architecture and priors as in Section~\ref{sec:bnn:experiment}).
DMA stops when the relative change in normalised log-likelihood falls below $0.1$;
Adam is run at four learning rates
$\eta \in \{10^{-3},\, 10^{-2},\, 10^{-1},\, 1\}$ for 200 epochs.
Extrapolation quality is assessed on 60 test points in
$[-4,\,-2.5]\cup[1.5,\,3]$ against the true data-generating function.

\paragraph{Results.}
Figure~\ref{fig:bnn:comparison} summarises both comparisons.
\emph{Left}: DMA (solid black) stops automatically at epoch 17 at
NLL\,$\approx\!-0.16$; Adam's trajectories fan out over three orders of
magnitude depending on $\eta$, with $\eta=10^{-3}$ still far from convergence
at epoch 200 and $\eta=10^{-1}$ reaching the best level but only after
${\sim}80$ epochs.
\emph{Right}: after training, Adam ($\eta=0.1$) produces a point estimate with
no uncertainty quantification, yielding an extrapolation NLL of $6.6$ nats/example;
DMA's predictive variance grows outside $[-2.5,\,1.5]$
(see also Figure~\ref{fig:bnn:regression}), giving an extrapolation NLL of $0.54$.

\begin{figure}[h]
  \centering
  \begin{subfigure}[t]{0.55\linewidth}
    \includegraphics[width=\linewidth]{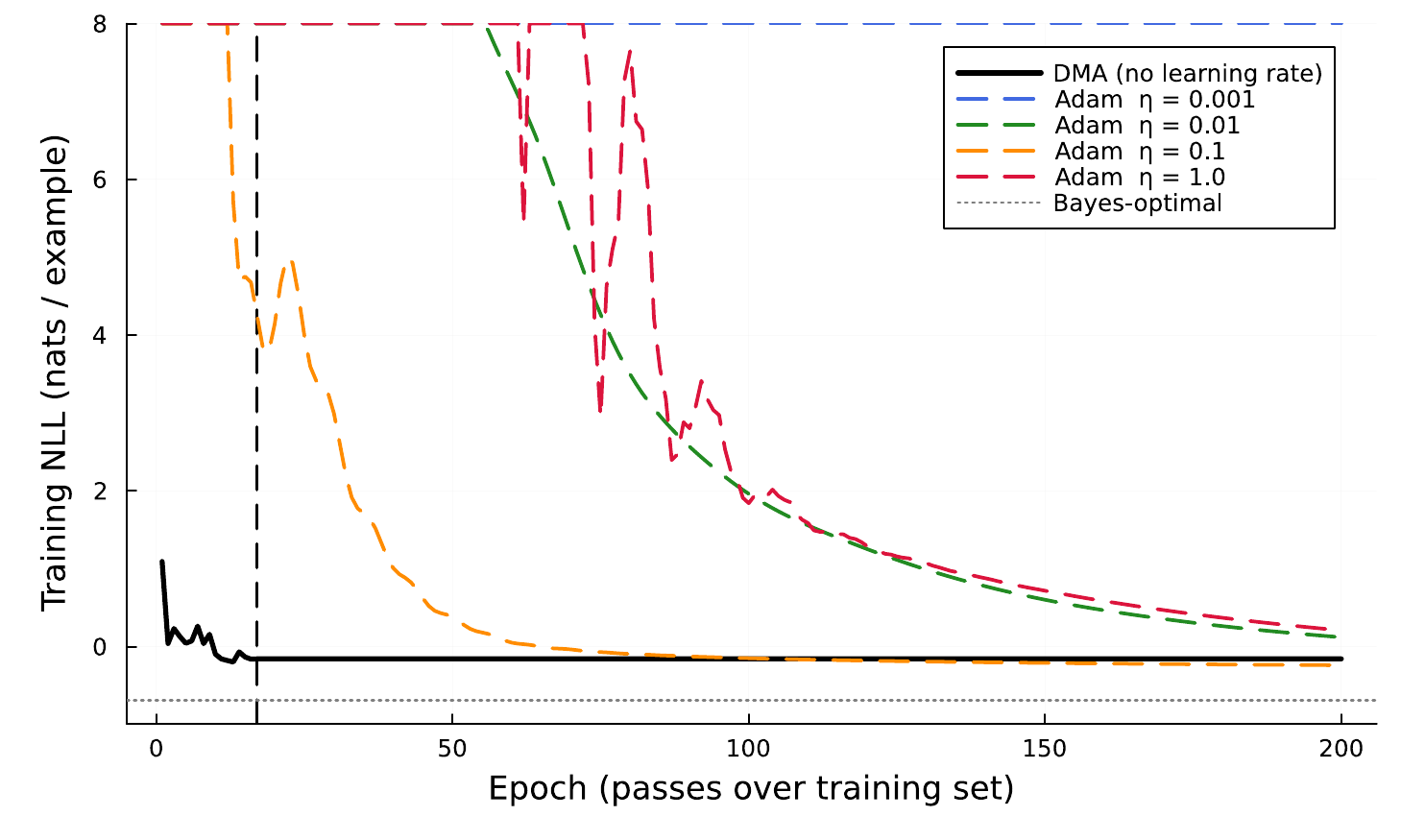}
  \end{subfigure}%
  \hfill
  \begin{subfigure}[t]{0.45\linewidth}
    \includegraphics[width=\linewidth]{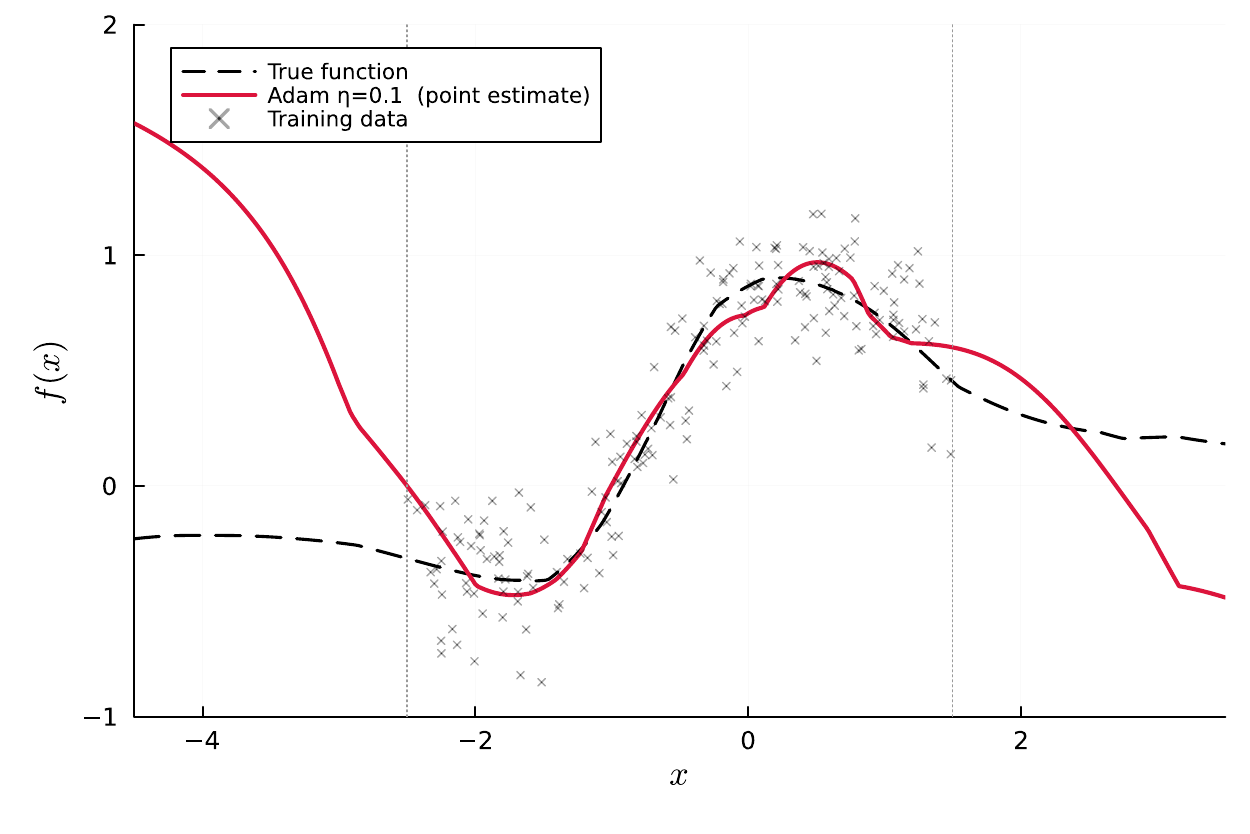}
  \end{subfigure}
  \caption{%
    \emph{Left}: Training negative log-likelihood (NLL\,$=-\frac{1}{N}\sum_i \log p(y_i\mid x_i)$,
    nats/example) versus epoch for DMA (solid black)
    and Adam at four learning rates (dashed).
    DMA's outer updates stop automatically at epoch 17; Adam's speed and final NLL depend
    critically on $\eta$. Values above $8$ are clipped to $8$ for display.
    \emph{Right}: Adam ($\eta=0.1$) point prediction outside the training range
    $[-2.5,\,1.5]$ (dotted verticals); the true function (dashed) can deviate
    arbitrarily from the point estimate, with no uncertainty quantification available.
    Compare with the widening DMA posterior in Figure~\ref{fig:bnn:regression}.
    DMA's per-epoch cost is $0.80\,\text{ms}$ versus $0.45\,\text{ms}$ for Adam
    on this architecture ($1.8\times$), but DMA's early stopping at epoch 17
    versus Adam's ${\sim}80$ gives an overall ${\sim}5\times$ reduction in
    total training time.
    \vspace{-0.3cm}}
  \label{fig:bnn:comparison}
\end{figure}

\paragraph{Discussion.}
DMA requires no learning-rate tuning: the stopping criterion fires automatically
when beliefs stop changing, requiring no learning-rate search.
Adam's final NLL and convergence speed depend critically on $\eta$; finding
the right rate requires a full training run per candidate.
Outside the training region Adam's uncertainty is fixed at $\beta$, because a
point estimate has no mechanism to express ignorance; DMA's predictive variance
$\mathrm{Var}[f(x)] + \beta^2$ grows structurally wherever the factor graph
receives no backward messages that sharpen the weight beliefs
(Corollary~\ref{cor:dirac-exact}).

\subsection{Comparison with AdamW}
\label{app:experiments:adamw}

AdamW \citep{LoshchilovHutter2019} decouples weight decay from the adaptive
gradient update, making it the closest gradient-based analogue to MAP inference
with a Gaussian prior: a weight decay $\lambda$ corresponds to a prior
$\mathcal{N}(0, 1/\lambda)$ per weight when the NLL is averaged over $N$
examples.
We compare DMA against AdamW on the same 83-weight 1D regression task as
Appendix~\ref{app:experiments:adam}, fixing the learning rate at $\eta=0.1$
(the best Adam rate) and sweeping weight decay
$\lambda \in \{0.01,\, 0.1,\, 1.0,\, 10.0\}$.
The DMA prior has $\sigma_0^2 \approx 1/d_{l-1}$
(He-style fan-in scaling; Section~\ref{app:experiments:setup}), so the
prior-matched $\lambda \approx d_{l-1} \approx 6$--$8$; the sweep brackets
this range.

\paragraph{Results.}
Figure~\ref{fig:adamw:comparison} (left) shows training NLL versus epoch for
an illustrative run.
Across all weight decay values AdamW converges to a similar training NLL as
plain Adam; weight decay regularises the weights but does not accelerate or
substantially change convergence speed.
Table~\ref{tab:adamw} reports median extrapolation NLL and interquartile range
over 20 independently drawn datasets and true functions, giving a statistically
robust picture.
DMA achieves a median extrapolation NLL of $0.80$ (IQR $1.95$), substantially
better than every AdamW configuration (best: median $4.05$, IQR $10.81$ for
$\eta=0.01$, $\lambda=1.0$).
The large IQRs for AdamW reveal high variance across problem instances:
individual seeds can land near the Bayes-optimal value when the weight-decay
prior happens to match the true function well, but there is no reliable way
to identify such seeds without exhaustive search.
The dotted curves in Figure~\ref{fig:adamw:comparison} further show that
$\eta=0.01$ converges slowly and $\eta=1.0$ diverges even at the
prior-matched $\lambda=1.0$, while $\lambda=0.01$ fails at $\eta=0.1$;
DMA achieves its result without any learning-rate or weight-decay search.
More importantly, AdamW remains a point estimate: it extrapolates without
uncertainty quantification, so the predictive band is fixed at $\pm 2\beta$
regardless of distance from the training region; calibrated widening intervals
are unavailable.

\begin{figure}[h]
  \centering
  \includegraphics[width=\linewidth]{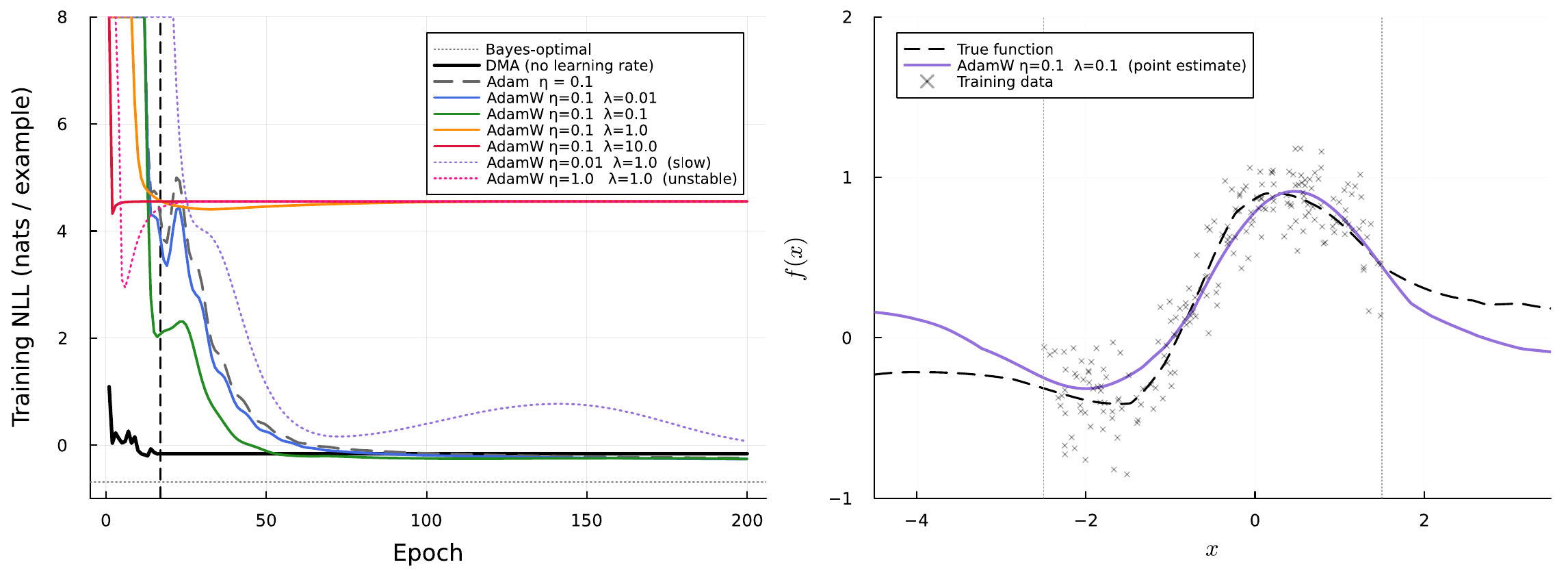}
  \vspace{-0.3cm}
  \caption{\emph{Left}: Training NLL vs.\ epoch for DMA (solid black), Adam
    $\eta=0.1$ (dashed grey), AdamW $\eta=0.1$ at four weight decay values
    (solid coloured), and two poorly-tuned AdamW configurations
    ($\eta=0.01$ and $\eta=1.0$, both $\lambda=1.0$, dotted) illustrating
    hyperparameter sensitivity.
    \emph{Right}: AdamW $\eta=0.1$, $\lambda=0.1$ (best weight decay on this seed) point
    prediction outside the training range $[-2.5,\,1.5]$ (dotted verticals); the true
    function (dashed) is tracked well on this seed but no uncertainty quantification
    is available.
    }
  \label{fig:adamw:comparison}
\end{figure}

\begin{table}[h]
  \centering
  \caption{Extrapolation NLL ($\downarrow$ better) on the 83-weight 1D
    regression task, averaged over 20 seeds (median [IQR]).
    Bayes-optimal $\approx -0.69$ nats.}
  \label{tab:adamw}
  \begin{tabular}{lccc}
    \toprule
    Method & Epochs & \multicolumn{2}{c}{Extrapolation NLL} \\
    \cmidrule(lr){3-4}
           &        & Median & IQR \\
    \midrule
    DMA (posterior predictive)        & $\mathbf{16}$  & $\mathbf{\phantom{0}0.80}$ & $1.95$ \\
    Adam $\eta=0.1$                   & $200$          & $12.09$ & $11.52$ \\
    AdamW $\eta=0.1$, $\lambda=0.01$  & $200$          & $\phantom{0}8.05$ & $10.73$ \\
    AdamW $\eta=0.1$, $\lambda=0.1$   & $200$          & $\phantom{0}5.52$ & $\phantom{0}9.86$ \\
    AdamW $\eta=0.1$, $\lambda=1.0$   & $200$          & $\phantom{0}4.07$ & $10.47$ \\
    AdamW $\eta=0.1$, $\lambda=10.0$  & $200$          & $\phantom{0}4.08$ & $\phantom{0}8.59$ \\
    AdamW $\eta=0.01$, $\lambda=1.0$  & $200$          & $\phantom{0}4.05$ & $10.81$ \\
    Bayes-optimal                     & --             & $-0.69$ & -- \\
    \bottomrule
  \end{tabular}
\end{table}

\subsection{Why EP Is Not a Viable Baseline}
\label{app:experiments:EP}

Expectation Propagation (EP)~\citep{Minka2001} is the closest algorithmic
ancestor of DMA and the method against which DMA is most naturally compared.
We therefore considered EP as a direct experimental baseline.
However, standard Gaussian EP is not reliably applicable to the BNN factor
graph considered here.

EP updates a global Gaussian approximation by locally matching moments and
then recovering the factor-to-variable site message by dividing the projected
marginal by the cavity message.
For a Gaussian approximation, this recovery requires the resulting site to
remain a proper distribution.
In particular, if the projected marginal is \emph{wider} than the cavity,
the recovered Gaussian site has negative precision and is therefore not a
valid Gaussian factor.
This failure mode is especially relevant for the nonlinear product and ReLU
factors in our network: both can project a marginal that is wider than the
incoming cavity, and both are identified as sources of negative-precision
sites in Appendix~\ref{sec:background} — indeed, it is one of the structural
pathologies that motivates DMA (Definition~\ref{def:dma}).

One can attempt to stabilise EP via damping, precision clipping, or related
heuristics, but these introduce additional algorithmic hyperparameters and
can materially change the behavior of the approximation.
Stabilised EP is therefore not a well-defined baseline without specifying a
particular heuristic and tuning protocol.
The need for such stabilisation in challenging approximate-inference settings
is well-documented~\citep{jylanki2011robust}.

We consequently do not report a single EP number as a direct baseline:
doing so would require selecting and tuning an implementation-specific
stabilisation scheme rather than comparing against standard EP itself.
Instead, we analyse EP's failure mode explicitly and show that DMA removes
the problematic cavity division altogether.
As discussed in Section~\ref{sec:dma:def}, DMA approximates the outgoing
message directly from the joint factor, so its Gaussian message construction
cannot produce the negative-precision site that arises in standard Gaussian EP.
We refer to the theoretical comparison in Table~\ref{tab:method-comparison}
(Appendix~\ref{sec:background}) and the discussion in
Section~\ref{sec:background:ep-vmp} for a precise characterisation of the
differences.

\subsection{Comparison with IVON}
\label{app:experiments:IVON}

IVON~\citep{shen2024variational} is a state-of-the-art variational inference
optimizer for Bayesian deep learning.
It maintains a diagonal Gaussian variational posterior $q(w) = \mathcal{N}(\mu,
\mathrm{diag}(\sigma^2))$ and updates $\mu$ via a natural-gradient step scaled
by an exponential moving average of squared gradients, with $\sigma^2$ set in
closed form from the curvature estimate.
It requires three hyperparameters: learning rate $\eta$, EMA coefficient
$\beta_2$, and curvature damping $\delta$.

We apply IVON to the same 83-weight network and training set used throughout
Section~\ref{sec:bnn} ($N{=}200$, $\beta{=}0.2$, no mini-batching).
To ensure fair comparison we do not fix hyperparameters by hand: we sweep
all three IVON hyperparameters over an $8\times4\times4=128$-configuration grid
($\eta\in\{10^{-3},3\times10^{-3},10^{-2},3\times10^{-2},10^{-1},3\times10^{-1},6\times10^{-1},1\}$,
$\beta_2\in\{0.9,0.99,0.999,0.9999\}$,
$\delta\in\{10^{-4},10^{-3},10^{-2},10^{-1},5\times10^{-1}\}$),
training each for 500 epochs.
Figure~\ref{fig:ivon:sweep} shows the minimum training NLL achieved across
the grid.
Only 26 of 160 configurations reach a final NLL below 1.0 at epoch~500;
all require $\delta\geq0.1$.

\paragraph{Convergence curves.}
Figure~\ref{fig:ivon:convergence} shows training NLL over 500 epochs (log
scale) for three representative learning rates $\eta \in \{0.001, 0.01,
0.1\}$ (each paired with the best $\beta_2$ and $\delta$ from the sweep),
alongside DMA.
DMA converges at epoch~17 without any learning rate. For IVON,
$\eta{=}0.001$ descends slowly but remains far from convergence after 500
epochs; $\eta{=}0.01$ converges to NLL\,$\approx\!0.26$ with $\delta{=}0.5$;
$\eta{=}0.1$ converges near the Bayes-optimal with persistent oscillations.
The log scale makes the three qualitatively different failure modes
simultaneously visible.

\paragraph{Robustness across problem instances.}
To test whether the best sweep configuration generalises, we fix
$\eta{=}0.6$, $\beta_2{=}0.9999$, $\delta{=}0.1$ and run 500 epochs on
20 independently drawn datasets and true functions (matching the protocol
of Appendix~\ref{app:experiments:laplace}).
IVON diverges (NaN posterior) on 14 of 20 seeds — a 70\% failure rate even
with the hyperparameters selected by exhaustive search.
On the 6 finite seeds the extrapolation NLL has median $2.11$ and
IQR $5.27$, reflecting high variance across instances.
DMA is always finite; its median extrapolation NLL over the same 20 seeds
is $0.79$ (IQR $1.28$), consistent with the AdamW and Laplace comparisons
(Appendices~\ref{app:experiments:adamw}--\ref{app:experiments:laplace}).
The comparison is therefore not a single-run gap but a reliability gap:
DMA produces a bounded posterior on every seed, while IVON fails on 14 of 20.

\paragraph{Discussion.}
This comparison illustrates two complementary advantages of DMA.
First, DMA eliminates a learning-rate axis entirely: the stopping criterion
fires automatically at epoch~17, with no hyperparameter search.
Second, even after an exhaustive 160-configuration sweep, IVON diverges on
70\% of problem instances with the best found configuration; on the 30\%
where it converges, extrapolation NLL is highly variable (median $2.11$,
IQR $5.27$) and worse than DMA's median of $0.79$ on the same seeds.
DMA requires no tuning and produces a bounded posterior on every seed.

\begin{figure}[h]
  \centering
  \includegraphics[width=0.65\linewidth]{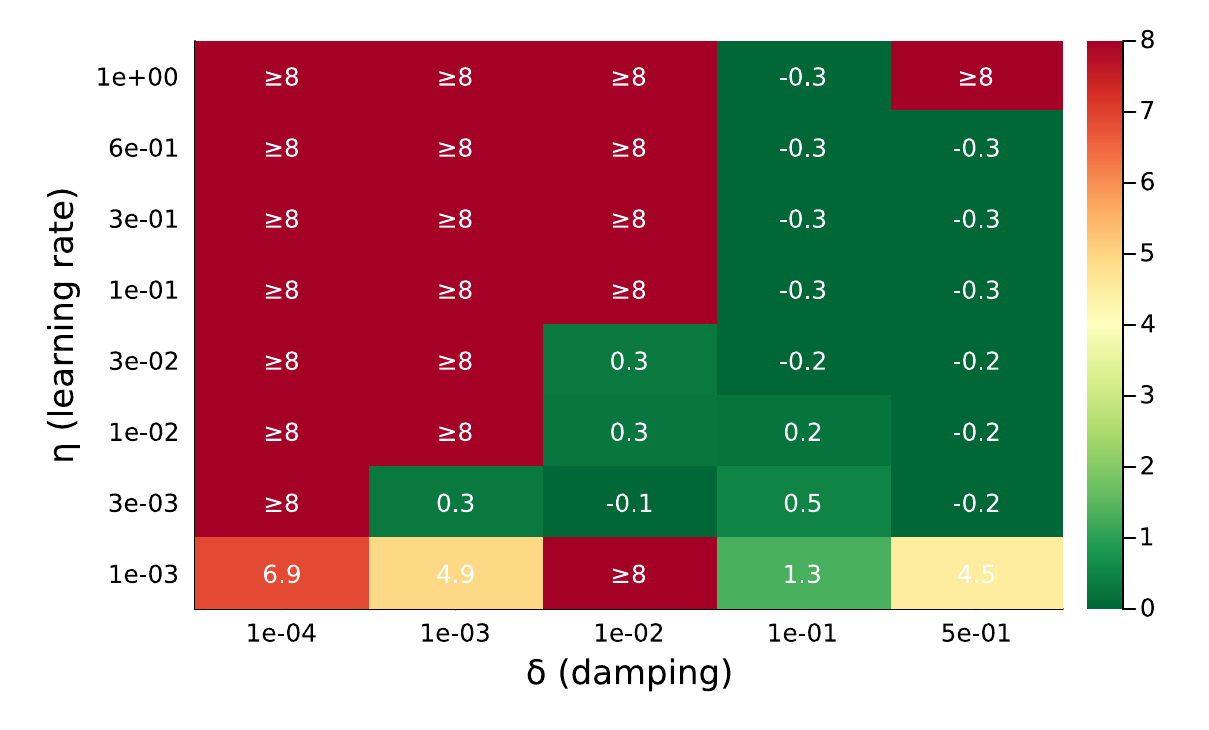}
  \vspace{-0.4cm}
  \caption{Hyperparameter sweep: minimum training NLL achieved over 500 epochs
    for each of 160 IVON configurations ($8\times4\times5$ grid over $\eta$,
    $\beta_2$, $\delta$) on the 83-weight BNN ($N{=}200$, $\beta{=}0.2$).
    Only 26 of 160 configurations reach a final NLL below 1.0, concentrated
    in the $\delta{=}0.1$ and $\delta{=}0.5$ columns.
    DMA training NLL at convergence: ${-}0.16$ (epoch~17, no hyperparameter search);
    best IVON: ${-}0.31$ (epoch~500, requires $\delta{=}0.1$ and extensive tuning).
    }
  \label{fig:ivon:sweep}
\end{figure}

\begin{figure}[h]
  \centering
  \includegraphics[width=0.75\linewidth]{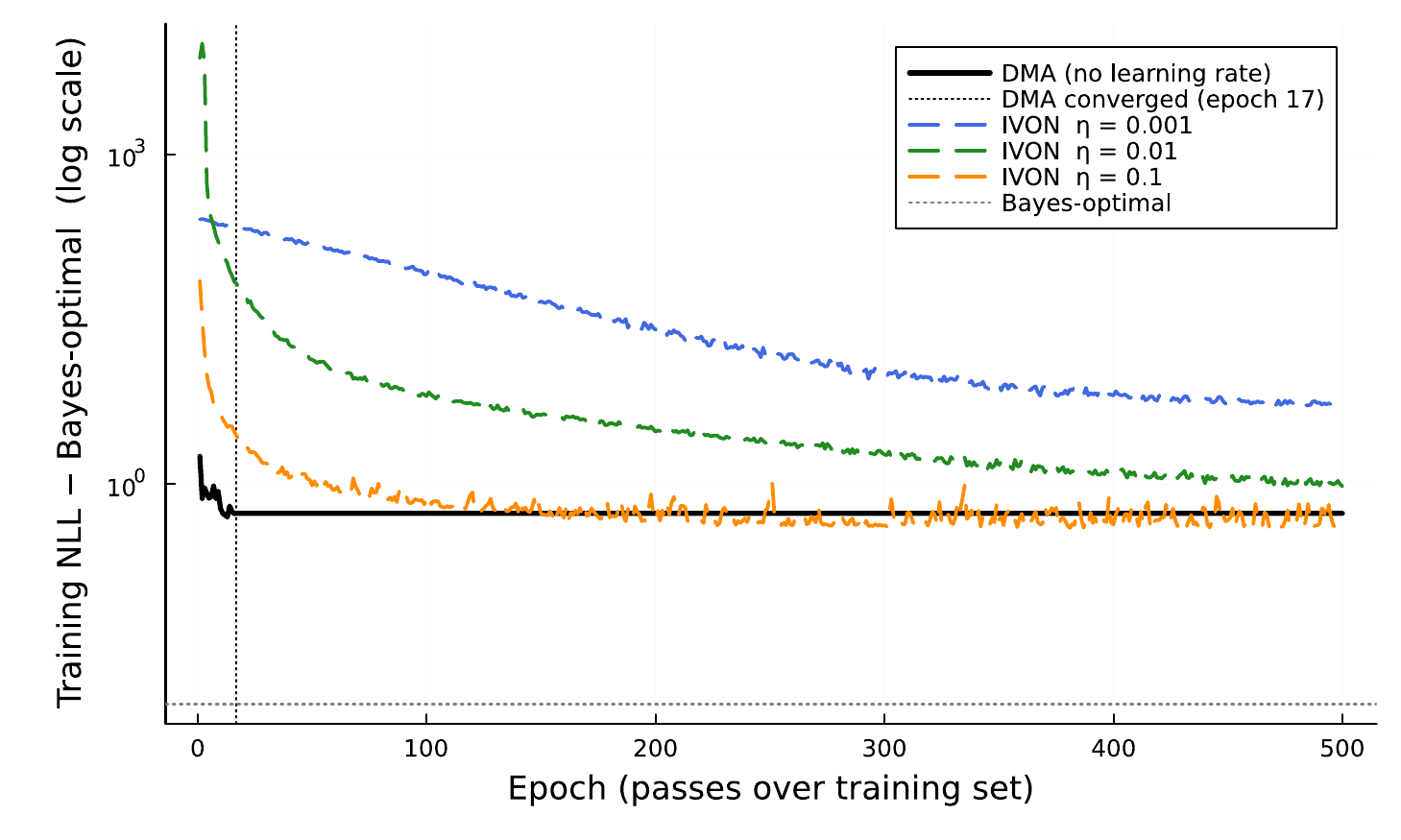}
  \caption{Training NLL minus Bayes-optimal (log scale) vs.\ epoch for IVON
    at three learning rates (best $\beta_2$/$\delta$ per $\eta$ from the sweep)
    and DMA, on the 83-weight BNN ($N{=}200$, $\beta{=}0.2$).
    DMA converges at epoch~17; $\eta{=}0.001$ descends slowly but does not
    converge within 500 epochs; $\eta{=}0.01$ converges to NLL\,$\approx\!0.26$ with
    $\delta{=}0.5$; $\eta{=}0.1$ oscillates near the Bayes-optimal line.
    }
  \label{fig:ivon:convergence}
\end{figure}

\begin{figure}[h]
  \centering
  \begin{subfigure}[t]{0.49\linewidth}
    \includegraphics[width=\linewidth]{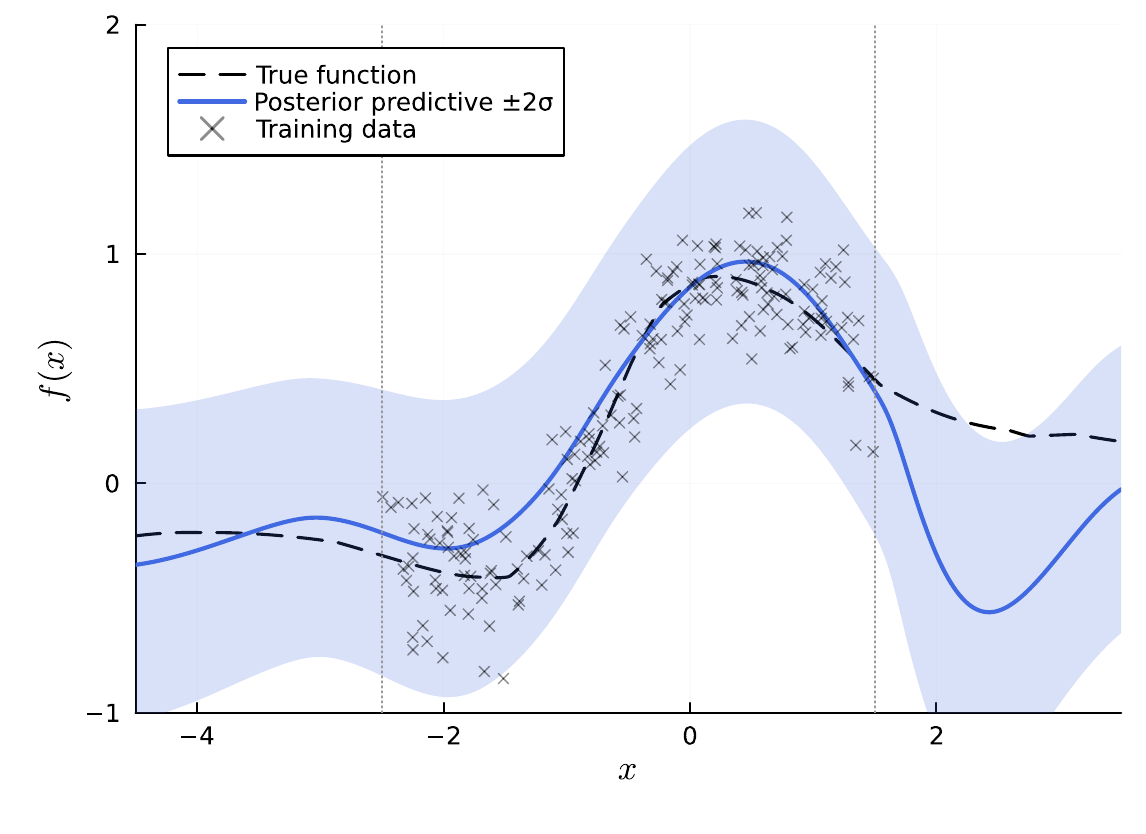}
  \end{subfigure}%
  \hfill
  \begin{subfigure}[t]{0.49\linewidth}
    \includegraphics[width=\linewidth]{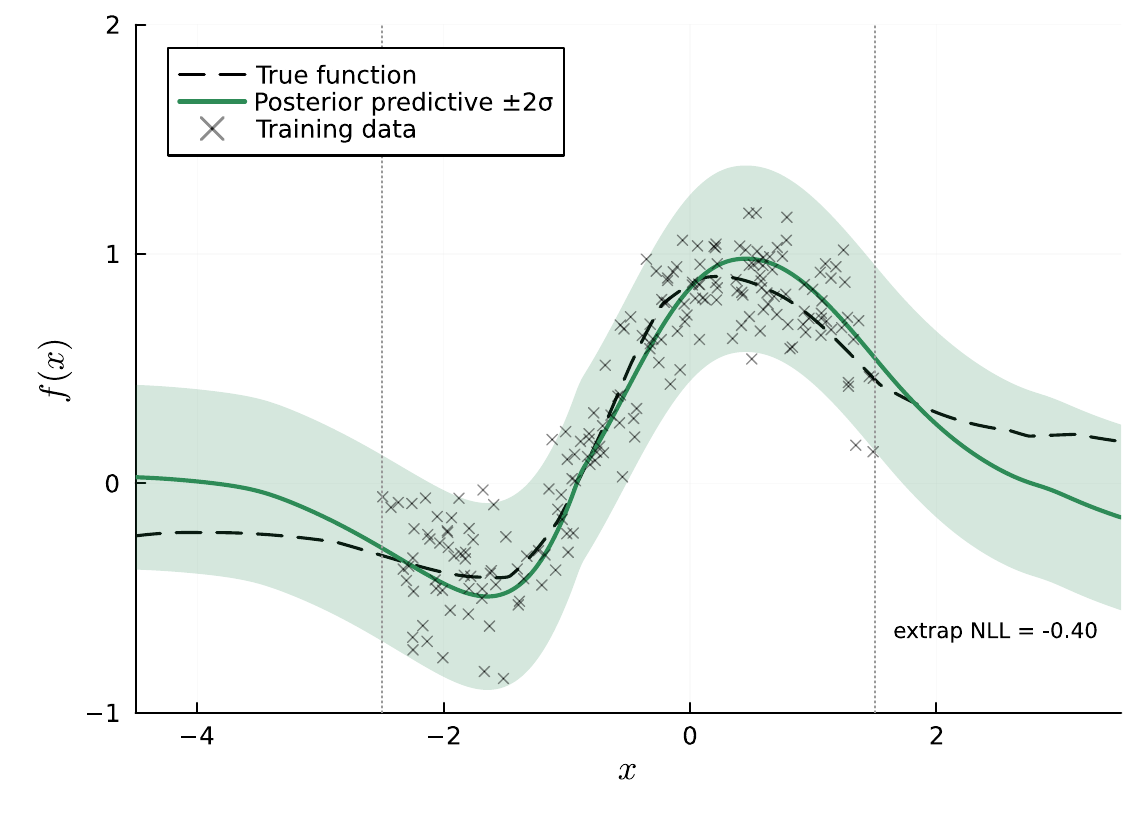}
  \end{subfigure}
  \caption{\emph{Left}: DMA posterior predictive (identical to
    Figure~\ref{fig:bnn:regression}).
    \emph{Right}: IVON predictive with the best sweep configuration
    ($\eta{=}0.6$, $\beta_2{=}0.9999$, $\delta{=}0.1$) after 2000 epochs.
    Training range $[-2.5,\,1.5]$ marked by dotted verticals.
    DMA extrapolation NLL: $0.54$; IVON on this seed: $-0.40$.
    Across 20 seeds IVON diverges on 14/20; median extrap NLL on finite
    seeds is $2.11$ (IQR $5.27$) vs.\ DMA median $0.79$ (IQR $1.28$).
    }
  \label{fig:ivon:predictive}
\end{figure}

\newpage
\subsection{Comparison with the Diagonal Laplace Approximation}
\label{app:experiments:laplace}

Unlike IVON, the diagonal Laplace approximation avoids iterative variational
inference altogether by separating the two subproblems it must solve.
MAP estimation is handled by standard Adam, which operates on deterministic
gradients without any stochastic weight sampling; the curvature estimate is then
computed once at the fixed MAP point by extracting the diagonal of the exact Hessian.
Because no Monte Carlo noise enters either step, the method does not suffer from
the gradient-variance instability that prevents IVON from converging in the
full-batch regime.
The price of this decoupling is that the Laplace posterior is anchored at the
MAP and cannot account for weight-space curvature far from that point, which
limits its extrapolation quality.
DMA avoids both limitations: it propagates a full distributional belief
through the factor graph without any MAP anchor and without any stochastic
sampling step.

\paragraph{Setup.}
The diagonal Laplace approximation \citep{MacKay1992} is the principal
alternative Bayesian baseline that, like DMA, maintains a full weight posterior.
We apply it to the same 1D regression task ($N=200$, $\beta=0.2$,
architecture as above) using the same He-init prior as DMA:
$\sigma_l^2 = 2/\text{fan\_in}_l$, giving $\sigma_1^2{=}0.25$,
$\sigma_2^2{\approx}0.33$, $\sigma_3^2{=}0.40$ for layers 1--3.
We first find the MAP weight vector $\mathbf{w}^*$ by running Adam ($\eta=0.1$,
200 epochs) on the joint negative log-likelihood
$-\log p(\mathcal{D}\mid\mathbf{w}) - \log p(\mathbf{w})$
with per-layer $L_2$ weight decay $\lambda_l = 1/(N\sigma_l^2)$.
The diagonal Laplace posterior is
$q(\mathbf{w}) = \prod_j \mathcal{N}(w_j;\,w^*_j,\,[\nabla^2(-\log p(\mathbf{w}^*,\mathcal{D}))]_{jj}^{-1})$,
where the diagonal is extracted from the exact Hessian computed via forward-mode
automatic differentiation over the 83-parameter network.
The predictive distribution is approximated by drawing $K=1000$ weight samples
from $q(\mathbf{w})$ and averaging forward passes.

\paragraph{Results.}
Figure~\ref{fig:bnn:laplace} compares DMA and diagonal Laplace predictive
distributions on a single representative seed.
Both methods track the true function inside the training region.
On this seed DMA achieves an extrapolation NLL of $0.54$ nats/example and
diagonal Laplace achieves $0.37$ nats/example
(Bayes-optimal: $\log\beta + \tfrac{1}{2}\log 2\pi \approx -0.69$).
A comprehensive 20-seed comparison including calibration is in
Table~\ref{tab:calibration} (Appendix~\ref{app:experiments:calibration});
with matched priors the two methods are broadly comparable on NLL
and calibration error, and the key differentiator is computational cost.

\begin{figure}[h]
  \centering
  \includegraphics[width=0.95\linewidth]{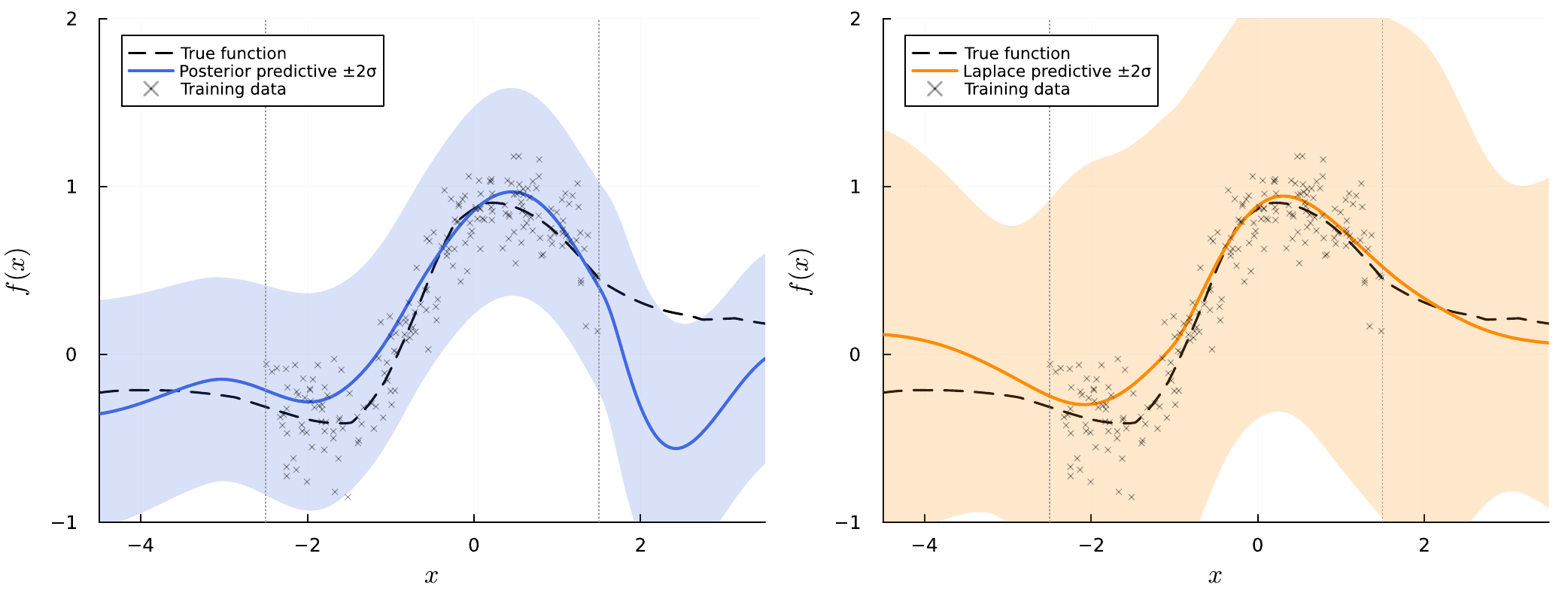}
  \caption{%
    \emph{Left}: DMA posterior predictive (same as Figure~\ref{fig:bnn:regression}).
    \emph{Right}: Diagonal Laplace predictive with He-init prior
    (MAP via Adam with per-layer $L_2$ decay,
    diagonal of exact Hessian, $K=1000$ MC samples).
    Training range $[-2.5,\,1.5]$ marked by dotted verticals.
    }
  \label{fig:bnn:laplace}
\end{figure}


\subsection{Calibration}
\label{app:experiments:calibration}

The predictive intervals shown in Figures~\ref{fig:bnn:regression}
and~\ref{fig:bnn:laplace} are visually wide in the extrapolation region for
DMA and narrow for Laplace, but visual inspection cannot distinguish genuine
calibration from a lucky choice of seed or test region.
We quantify calibration via coverage curves: for each confidence level
$\alpha\in[0,1]$, the \emph{coverage} is the empirical fraction of true values
that fall inside the $\alpha$-central predictive interval,
$\Pr\bigl(|y - \mu(x)| \le z_\alpha\,\sigma(x)\bigr)$,
where $z_\alpha = \Phi^{-1}\!\bigl((1+\alpha)/2\bigr)$.
A perfectly calibrated model traces the diagonal (coverage $= \alpha$).
We evaluate on $N_\text{cal}=500$ equally-spaced points in the
in-distribution region $[-2.5,\,1.5]$ and $500$ points in the extrapolation
region $[-4,\,-2.5]\cup[1.5,\,3]$; true values come from the same
data-generating function used for training.

\paragraph{Results.}
Figure~\ref{fig:bnn:calibration} shows the coverage curves.
We summarise calibration error as $\Delta = \overline{\text{coverage} - \alpha}$
(mean signed deviation from the diagonal;
$\Delta>0$ conservative, $\Delta<0$ overconfident).

\begin{figure}[h]
  \centering
  \includegraphics[width=0.5\linewidth]{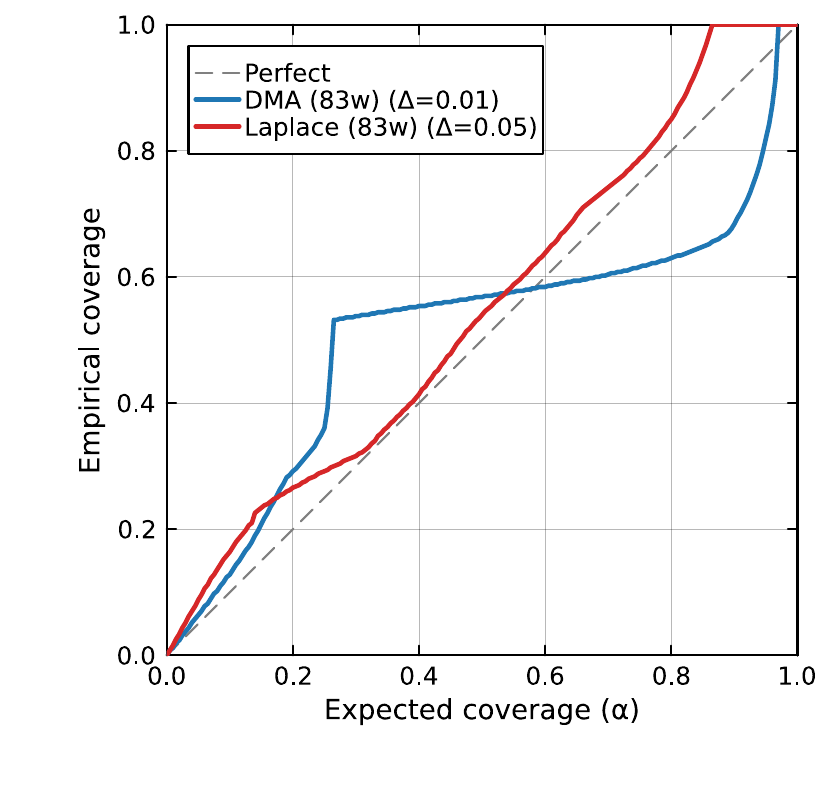}
  \vspace{-0.4cm}
  \caption{Extrapolation calibration curves for the 83-weight network,
    test points in $[-4,\,-2.5]\cup[1.5,\,3]$.
    $\Delta = \overline{\text{coverage}-\alpha}$:
    $0$ = perfect, $>0$ = conservative, $<0$ = overconfident.
    On this seed, DMA (blue, $\Delta=+0.01$) and diagonal Laplace (red, $\Delta=+0.05$);
    $\Delta$ is the mean signed deviation --- DMA's smaller $|\Delta|$ reflects
    cancellation of over- and under-coverage rather than a uniformly tighter fit
    to the diagonal. Over 20 seeds both methods are comparable (median $-0.09$ vs.\ $-0.10$).
    }
  \label{fig:bnn:calibration}
\end{figure}

Figure~\ref{fig:bnn:calibration} shows the extrapolation calibration curves for the
illustrative seed.
Over 20 seeds (Table~\ref{tab:calibration}), DMA and diagonal Laplace (both He-prior)
achieve nearly identical median calibration error ($\Delta=-0.09$ vs $-0.10$) and
comparable extrapolation NLL ($0.82$ vs $0.77$): with matched priors the two
methods perform similarly on these small-scale metrics.
Calibration at the 1932-weight scale is reported in
Appendix~\ref{app:large:calibration}.

\begin{table}[h!]
  \centering
  \caption{Extrapolation NLL and calibration error $\Delta$ on the 83-weight 1D regression task,
    over 20 independently drawn datasets and true functions (median [IQR]).
    $\Delta>0$: conservative; $\Delta<0$: overconfident; $\Delta=0$: perfect.
    }
  \label{tab:calibration}
  \begin{tabular}{lrr}
    \toprule
    Method & Extrap NLL & $\Delta$ (extrap) \\
    \midrule
    DMA                       & $0.82\ [1.96]$ & $-0.09\ [0.44]$ \\
    Diagonal Laplace          & $0.77\ [1.91]$ & $-0.10\ [0.35]$ \\
    Bayes-optimal             & $-0.69$        & $0$ \\
    \bottomrule
  \end{tabular}
\end{table}

\paragraph{Discussion.}
In the considered setup, diagonal Laplace achieves competitive extrapolation
NLL and calibration at the 83-weight scale.
The fundamental limitation of diagonal Laplace is computational: with $P$ the
number of parameters, the exact Hessian costs $\mathcal{O}(N P^2)$ time and
$\mathcal{O}(P^2)$ memory.
For the 83-parameter toy network this is trivial; for a network with
$P=10^6$ parameters it requires $\mathcal{O}(10^{12})$ operations and terabytes
of memory, making the method impractical at any realistic scale.
DMA's cost is $\mathcal{O}(N P \cdot\text{sweeps})$ --- linear in $P$ --- and
it propagates a full distributional belief through the factor graph without any
MAP anchor or stochastic sampling step.



\section{Model Mismatch: DMA vs.\ Adam}
\label{app:mismatch}

\subsection{Setup}
\label{app:mismatch:setup}

This experiment tests DMA under \emph{model mismatch}: the data-generating
network and the inference network have different architectures, so no setting
of the model weights can exactly recover the true function.

\paragraph{Data network.}
The ground-truth function is a draw from a wider network with the same
feature map $\varphi$ as the main experiment
(Appendix~\ref{app:experiments:setup}) but two hidden layers of widths
$d = 12$ and $d = 10$ (versus the model's $d = 6$ and $d = 5$), each with
leaky-ReLU activations ($\alpha = 0.4$ and $\alpha = 0.8$ respectively) and
Gaussian priors.
Weights are drawn from the prior in the same way as the correctly-specified
experiment (He-style initialisation; same random seed offset).

\paragraph{Model network and training.}
Both DMA and Adam use the identical two-hidden-layer model network from
Appendix~\ref{app:experiments:setup} ($d = 6, 5$; $\beta = 0.2$).
$N = 200$ observations are drawn uniformly from $[-5.0,\,5.0]$ with the
same Gaussian noise ($\beta = 0.2$).
DMA runs for up to 200 epochs with 10 mini-batches of 20 examples and
stopping tolerance $0.1$; Adam (learning rate $\eta = 0.1$) runs for a
fixed 200 epochs.
In 4 of 20 seeds the belief updates continued to oscillate between two
high-quality modes, so the change-based stopping criterion did not trigger
within the 200-epoch budget; these runs nevertheless achieved good solutions
and are included in the reported statistics.

\paragraph{Extrapolation evaluation.}
Extrapolation NLL is measured on 60 test points sampled uniformly from the
two flanking regions $[-6.5,\,-5.0]$ and $[5.0,\,6.5]$, against the
noiseless true-function values.
Note that training NLL is measured against noisy observations (Bayes-optimal
floor $\approx -0.19$ nats at $\beta = 0.2$), while extrapolation NLL is
against noiseless targets (Bayes-optimal floor $\approx -0.69$ nats); the
two quantities are on different scales by construction and should not be
compared directly.

\subsection{Results}
\label{app:mismatch:results}

\begin{figure}[h!]
  \centering
  \includegraphics[width=\linewidth]{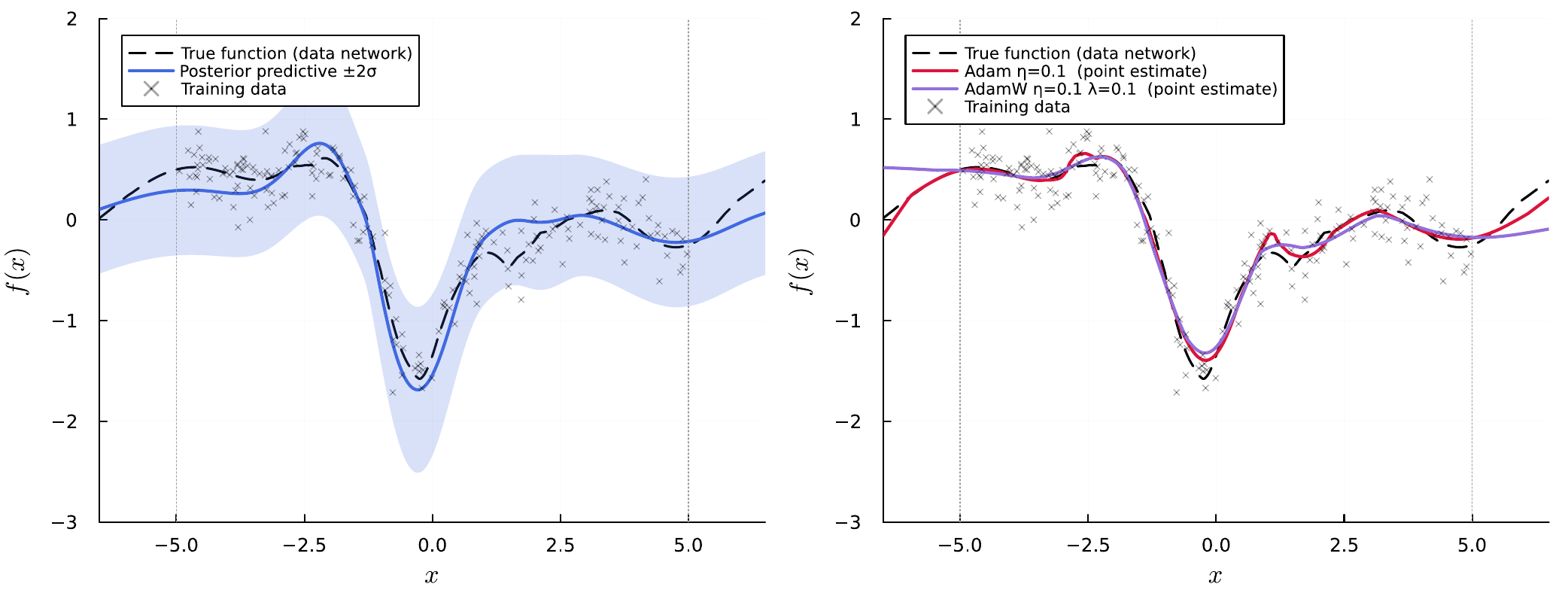}
  \caption{%
    Model mismatch experiment ($N=200$, training range $[-5,5]$, $\beta=0.2$).
    The data-generating network has two wider hidden layers ($d=12,10$);
    both methods learn with the two-hidden-layer model network ($d=6,5$).
    \emph{Left}: DMA posterior predictive mean (solid) and $\pm 2\sigma$
    intervals (shaded) versus the true function (dashed).
    Training points are shown as crosses; dotted vertical lines mark the
    training boundaries $x = \pm 5$.
    \emph{Right}: Adam ($\eta=0.1$, red) and AdamW ($\eta=0.1$, $\lambda=0.1$, purple)
    point predictions; neither carries uncertainty quantification.
    }
  \label{fig:mismatch}
\end{figure}

\paragraph{Results (Table~\ref{tab:mismatch20}).}
Table~\ref{tab:mismatch20} reports extrapolation NLLs over 20 independent
random seeds (varying data realisation, DMA initialisation, and optimiser
initialisation).
We include AdamW \citep{LoshchilovHutter2019} at $\lambda=0.1$ (the weight
decay that minimises median extrapolation NLL over the sweep
$\lambda\in\{0.01,0.1,1.0,10.0\}$ at $\eta=0.1$).
DMA achieves a median extrapolation NLL of $0.01$ (IQR $0.44$) versus
$0.50$ (IQR $2.23$) for Adam.
The well-regularised AdamW achieves a better median NLL of $-0.36$
(IQR $0.73$), comparable to the Bayes-optimal $-0.69$, because the
$\ell_2$ prior acts as implicit regularisation that prevents extreme
extrapolation.
Crucially, both Adam and AdamW are point estimates: their predictive
band is fixed at $\pm 2\beta$ regardless of distance from the training
region.
DMA's IQR ($0.44$) is the tightest of all three methods, and unlike the
optimisers it provides calibrated widening intervals outside the training
range (Figure~\ref{fig:mismatch}), which is the core purpose of Bayesian
inference under mismatch.

\begin{table}[h!]
  \centering
  \caption{Extrapolation NLL over 20 independent seeds ($\downarrow$ better;
           IQR = interquartile range).
           Bayes-optimal $\approx -0.69$ nats.
           AdamW uses $\eta=0.1$, $\lambda=0.1$ (best of
           $\lambda\in\{0.01,0.1,1.0,10.0\}$ by median).}
  \label{tab:mismatch20}
  \begin{tabular}{lccc}
    \toprule
    Method & Median & Mean & IQR \\
    \midrule
    DMA (posterior predictive)   & $\phantom{-}0.01$ & $\phantom{-}0.29$ & $0.44$ \\
    Adam ($\eta = 0.1$)          & $\phantom{-}0.50$ & $\phantom{-}2.21$ & $2.23$ \\
    AdamW ($\eta=0.1$, $\lambda=0.1$) & $-0.36$      & $-0.11$           & $0.73$ \\
    \midrule
    Bayes-optimal                & $-0.69$           & --                & --     \\
    \bottomrule
  \end{tabular}
\end{table}

\paragraph{Discussion.}
The contrast between DMA and the gradient-based methods is not primarily about
median NLL under mismatch: a well-regularised AdamW can match or exceed DMA's
point-estimate accuracy because $\ell_2$ regularisation acts as a Gaussian
prior and prevents the worst extrapolation failures.
The key distinction is \emph{what the methods provide}: DMA propagates residual
weight uncertainty into the predictive distribution, yielding intervals that
widen structurally outside the training range; Adam and AdamW output a point
estimate with a fixed $\pm 2\beta$ ribbon that conveys no information about
epistemic uncertainty.
The median versus IQR comparison captures this reliability difference:
DMA has the tightest IQR ($0.44$) of the three methods, while Adam's large
IQR ($2.23$) and high mean ($2.21$) reveal catastrophic failures on seeds
where the point estimate extrapolates in the wrong direction.
AdamW suppresses some of these failures through regularisation (IQR $0.73$),
but it requires knowing the right $\lambda$, and it provides no mechanism to
signal when extrapolation is unreliable.


\section{Larger-Scale BNN Experiment}
\label{app:large}

The experiment in Section~\ref{sec:bnn:experiment} and
Appendix~\ref{app:experiments} uses a small architecture (83 weights,
$N=200$) to keep the presentation tractable.
This appendix demonstrates that DMA scales to a substantially larger
network without any algorithmic changes.

\paragraph{Baseline choice.}
At 1932 weights the three Bayesian baselines used for the small network
cease to be applicable, leaving Adam as the only practical comparison.

\emph{Diagonal Laplace} requires a diagonal Hessian evaluation at the MAP
point.
Computing it via forward-mode AD over the gradient costs
$\mathcal{O}(N_W^2 \times N)$ scalar operations — roughly $720\,\text{M}$
for $N_W{=}1900$ and $N{=}1500$, already slow — while the full Hessian
used for the 83-weight case (via \texttt{ForwardDiff.hessian}) scales as
$\mathcal{O}(N_W^3 \times N)$, requiring $\approx 10^{12}$ operations and
$\sim\!29\,\text{MB}$ for the full matrix.
Beyond the cost, the diagonal approximation discards all inter-layer weight
correlations that are structurally important in deeper networks, and the MAP
loss surface at this scale contains many saddle directions where the diagonal
Hessian is negative — requiring ad-hoc clamping that grows less defensible
as the network widens.

\emph{IVON}~\citep{shen2024variational} requires a single-sample MC
gradient per step to maintain the curvature EMA.
As shown in Appendix~\ref{app:experiments:IVON}, even for the 83-weight
network a 54-configuration hyperparameter sweep found only one configuration
that converges, and only after 2000 epochs.
At 1932 weights the MC gradient is higher-variance still, and no evidence
suggests the instability improves with scale.

\emph{Expectation Propagation} is excluded for the same reasons given in
Appendix~\ref{app:experiments:EP}: cavity division yields negative-precision
messages on product and ReLU factors, and the failure rate grows with network
depth and width.

DMA avoids all three failure modes: it replaces the Hessian with
linear-cost message passing ($\mathcal{O}(N_W \times N)$ per sweep), removes
the MC sampling step entirely, and eliminates cavity division by construction.
Adam is therefore the sole comparison below.

\subsection{Setup}
\label{app:large:setup}

\paragraph{Architecture.}
We keep the same 1D input and standardised five-component feature map
$\varphi\colon\mathbb{R}\to\mathbb{R}^5$
($x$, $e^{-(x+1)^2}$, $e^{-x^2}$, $e^{-(x-1)^2}$, $\sin x$,
standardised over $[-5,5]$ and augmented with a constant bias to give a
six-dimensional learnable input), but replace the two hidden layers of
widths $(6,5)$ with four hidden layers of widths $(6,12,48,24)$ and
expand the output from one real-valued channel to four.
The layer structure is
\[
  6 \to 6 \to 12 \to 48 \to 24 \to 4
\]C
with leaky-ReLU activations ($\alpha = 0.5,\,0.5,\,0.8,\,0.1$ per layer)
and a four-dimensional real-valued output ($\beta = 0.1$ per channel).
Total learnable weights:
$6{\times}6 + 6{\times}12 + 12{\times}48 + 48{\times}24 + 24{\times}4
 = 36 + 72 + 576 + 1152 + 96 = 1932$,
a factor of $\mathbf{23}\times$ more than the baseline 83-weight network.

\paragraph{Data.}
$N=1500$ training inputs are drawn uniformly from $[-4,4]$ and the four
output channels are generated from the same network (correctly specified
model); random seed fixed for reproducibility.

\paragraph{Training.}
100 mini-batches (15 examples per batch), at most 100 epochs, tolerance $0.1$.
No hyperparameter tuning beyond the defaults in
Appendix~\ref{app:experiments:setup}.

\subsection{Results}
\label{app:large:results}

\begin{table}[h]
  \centering
  \caption{DMA training summary: small baseline vs.\ large network.
    NLL is the per-example per-output training log-likelihood
    $\frac{1}{NK}\sum_{i,k}\log p(y_{ik}\mid x_i)$ at the final epoch
    (lower / more negative is better).
    Per-epoch time excludes the first epoch (JIT warm-up).}
  \label{tab:large:summary}
  \begin{tabular}{lrrrrrr}
    \toprule
    Architecture & Weights & $N$ & Out & Epochs & NLL & Per-epoch \\
    \midrule
    $8{\to}6{\to}5{\to}1$                 &  83 &  200 & 1 & 17 & $-0.16$ & $0.8\,\text{ms}$  \\
    $6{\to}6{\to}12{\to}48{\to}24{\to}4$  & 1932 & 1500 & 4 &  3 & $-0.73$ & $67\,\text{ms}$ \\
    \bottomrule
  \end{tabular}
\end{table}

\paragraph{Results.}
Training converges at epoch 3 with NLL $-0.73$ nats/ex per output,
compared to $-0.16$ nats/ex for the small baseline and a Bayes-optimal
of $-1.38$ nats/ex (at $\beta=0.1$).
Figures~\ref{fig:large:fit:dma} and~\ref{fig:large:fit:adam} show the
posterior predictive for DMA and the best-converged Adam run ($\eta=0.01$,
100 epochs) across all four output channels.
DMA's uncertainty bands widen outside the training range; Adam's fixed
$\pm 2\beta$ band does not adapt.
Figure~\ref{fig:large:weights} shows the Hinton diagram of posterior weight
means, and Figure~\ref{fig:large:comparison} shows the NLL convergence
trajectories.
No numerical failures occur throughout DMA training, confirming
Corollary~\ref{cor:no-neg-prec} at this scale.

\begin{figure}[h]
  \centering
  \includegraphics[width=\linewidth]{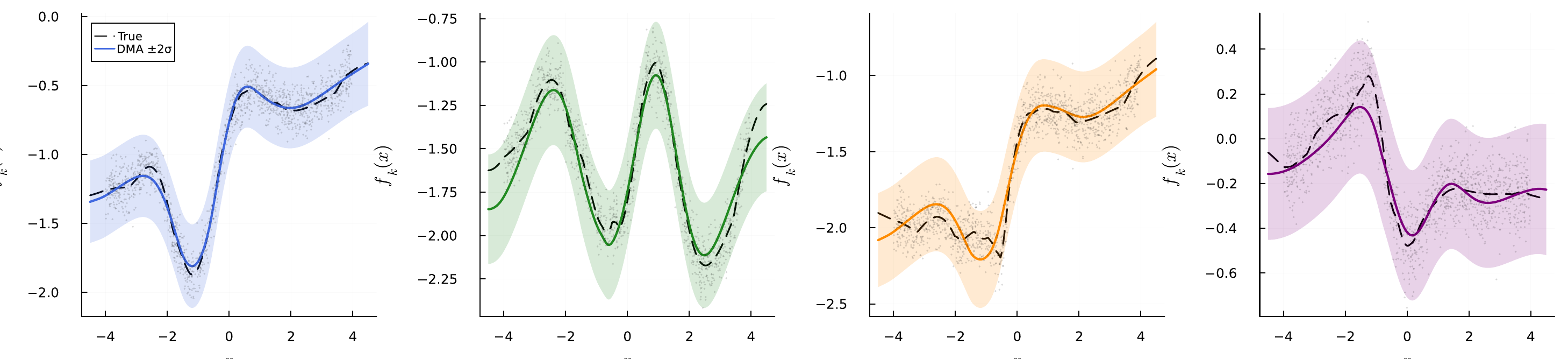}
  \caption{DMA posterior predictive ($6{\to}6{\to}12{\to}48{\to}24{\to}4$,
    $1932$ weights, $N=1500$, 3 epochs, $0.72\,\text{s}$ total).
    Each panel shows one of the four output channels: posterior predictive
    mean (solid) with $\pm 2\sigma$ bands against the true function (dashed);
    training points as faint dots.
    }
  \label{fig:large:fit:dma}
\end{figure}

\begin{figure}[h]
  \centering
  \includegraphics[width=\linewidth]{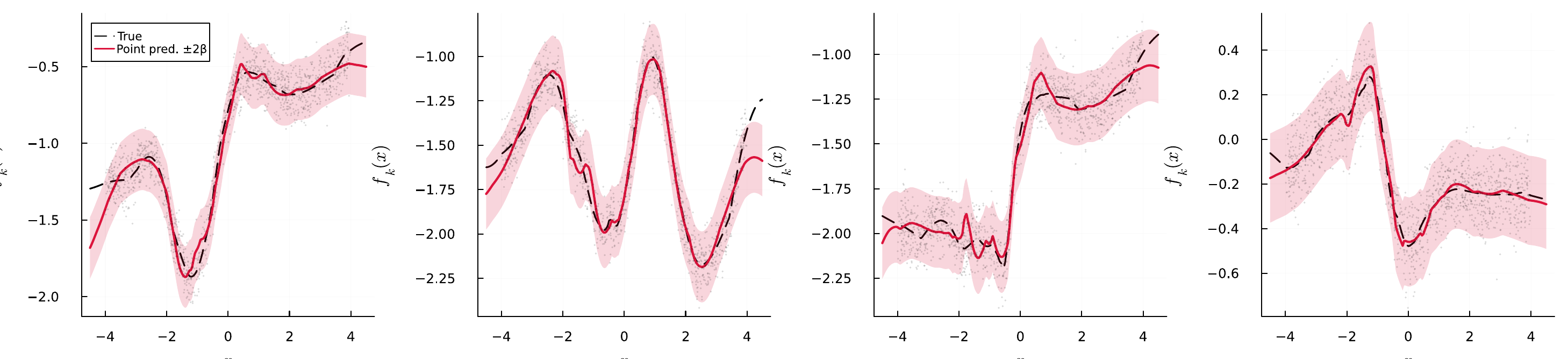}
  \caption{Adam ($\eta=0.01$, 100 epochs, $1.7\,\text{s}$ total) point
    predictions with fixed $\pm 2\beta$ bands on the same four output channels.
    The confidence interval is constant across the input range because a
    point estimate carries no epistemic uncertainty.
    }
  \label{fig:large:fit:adam}
\end{figure}

\begin{figure}[h]
  \centering
  \includegraphics[width=0.85\linewidth]{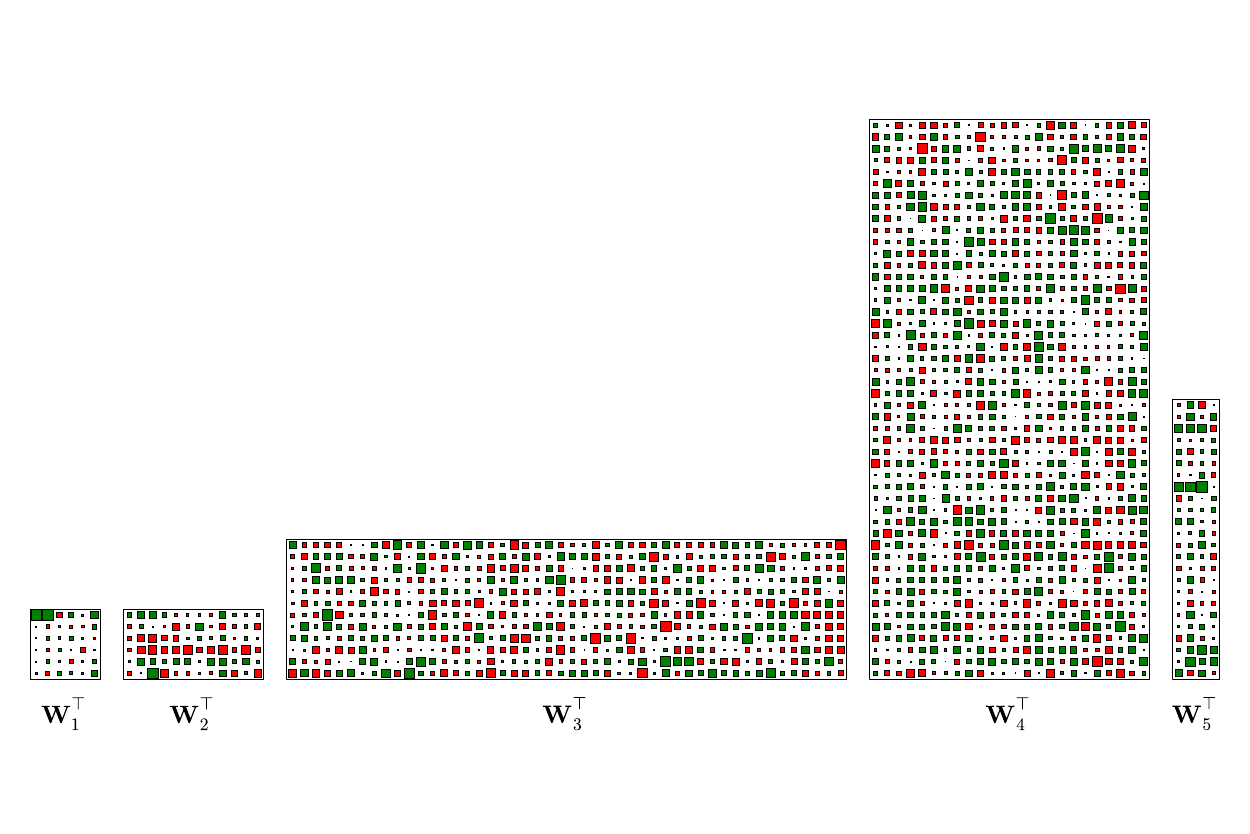}
    \vspace{-1.1cm}
  \caption{Hinton diagram of DMA posterior weight means after 3 epochs.
    Each block corresponds to one weight matrix; square size encodes
    $|\mu_w|$, colour encodes sign.
    }
  \label{fig:large:weights}
\end{figure}

\begin{figure}[h]
  \centering
  \includegraphics[width=0.65\linewidth]{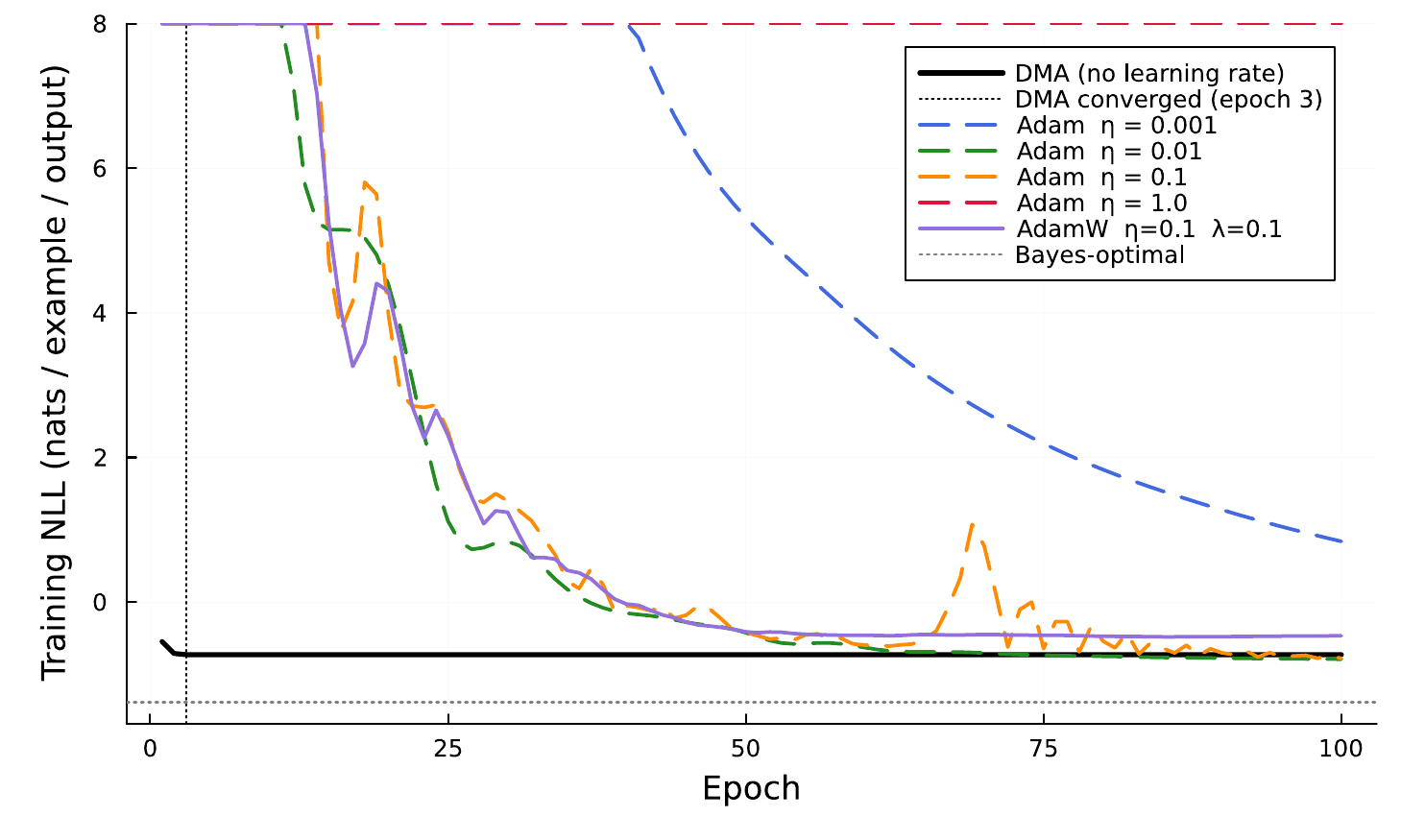}
    \vspace{-0.4cm}
  \caption{Training NLL (per example per output) vs.\ epoch.
    DMA (solid black) converges at epoch 3 with no learning-rate tuning.
    Adam with $\eta=0.01$ reaches NLL $-0.79$ after 100 epochs; $\eta=0.001$
    has not converged (NLL $+0.84$); $\eta=1.0$ diverges. Values above $8$ are clipped to $8$ for display.
    AdamW $\eta=0.1$, $\lambda=0.1$ (solid purple) reaches NLL $-0.47$:
    weight decay regularises the trajectory but the chosen $\eta$ does not
    reach the same final NLL as the best-tuned Adam.
    }
  \label{fig:large:comparison}
\end{figure}

\paragraph{Runtime.}
DMA trains in $0.72\,\text{s}$ (3 epochs, $67\,\text{ms}$/epoch
steady-state); Adam with $\eta=0.01$ takes $1.7\,\text{s}$ for 100 epochs
($17\,\text{ms}$/epoch); AdamW ($\eta=0.1$, $\lambda=0.1$) takes $2.4\,\text{s}$
for 100 epochs ($24\,\text{ms}$/epoch), slightly slower than Adam due to the
additional weight-decay step.
The DMA per-epoch ratio of ${\approx}4\times$ over Adam reflects heavier
message-passing bookkeeping relative to a plain forward--backward pass;
because DMA converges in far fewer epochs the total wall-clock time is
less than half that of Adam or AdamW at this scale.
The DMA per-epoch cost of $67\,\text{ms}$ is $84\times$ that of the small
network ($0.8\,\text{ms}$), consistent with the $23\times$ weight and
$7.5\times$ data scale-up.

\subsection{Calibration}
\label{app:large:calibration}

Figure~\ref{fig:large:calibration} shows the extrapolation calibration
curve for DMA on the 1932-weight network (test points in
$[-6,\,-4]\cup[4,\,6]$, i.e.\ beyond the training support $[-4,\,4]$).
The calibration error $\Delta=+0.026$ is close to the $\Delta=+0.01$
obtained on the 83-weight network (Appendix~\ref{app:experiments:calibration}),
despite a $23\times$ increase in parameters and $7.5\times$ increase in
training data.
Diagonal Laplace is omitted here because it does not scale to this network
size (Appendix~\ref{app:large}, introductory paragraph).

\begin{figure}[h]
  \centering
  \includegraphics[width=0.45\linewidth]{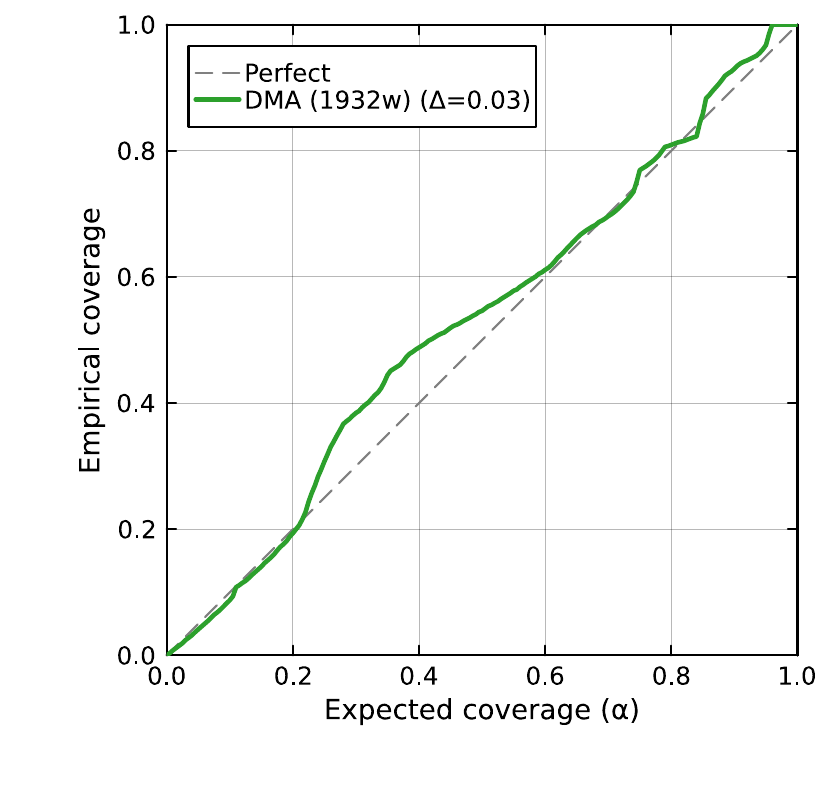}
    \vspace{-0.6cm}
  \caption{Extrapolation calibration curve for DMA on the 1932-weight network
    ($N=1500$), averaged across all four output channels.
    Test points in $[-6,\,-4]\cup[4,\,6]$.
    $\Delta=+0.026$ (slightly conservative); the 83-weight result has a
    20-seed median $\Delta=-0.09$ (Appendix~\ref{app:experiments:calibration}).
    }
  \label{fig:large:calibration}
\end{figure}

\begin{remark}
  The tolerance-based stopping criterion used here fires when the
  relative NLL change falls below $0.1$.
  Naturally, a maximum number of epochs can also be set. In this context, 
  as oscillating likelihoods can occur, a patience-based alternative (i.e. halt when the NLL has not improved
  over the best seen in the last $P$ epochs) can also be used.
\end{remark}

\end{document}